\documentclass{article}

\usepackage{microtype}
\usepackage{graphicx}
\usepackage{subfigure}
\usepackage{booktabs} 
\usepackage{hyperref}
\PassOptionsToPackage{sort}{natbib}

\usepackage{algorithm}
\usepackage{algorithmic}
\usepackage[final]{neurips_2022}

\usepackage{amsmath,amsfonts,bm}
\usepackage{amsthm}

\def\eqref#1{equation~\ref{#1}}

\def\1{\bm{1}}

\def\vg{{\bm{g}}}

\def\vm{{\bm{m}}}

\def\vp{{\bm{p}}}
\def\vq{{\bm{q}}}

\def\vs{{\bm{s}}}

\def\vu{{\bm{u}}}
\def\vv{{\bm{v}}}
\def\vw{{\bm{w}}}
\def\vx{{\bm{x}}}
\def\vy{{\bm{y}}}

\def\mA{{\bm{A}}}
\def\mB{{\bm{B}}}
\def\mC{{\bm{C}}}
\def\mD{{\bm{D}}}

\def\mF{{\bm{F}}}
\def\mG{{\bm{G}}}
\def\mH{{\bm{H}}}
\def\mI{{\bm{I}}}

\def\mL{{\bm{L}}}
\def\mM{{\bm{M}}}

\def\mO{{\bm{O}}}
\def\mP{{\bm{P}}}
\def\mQ{{\bm{Q}}}
\def\mR{{\bm{R}}}
\def\mS{{\bm{S}}}

\def\mU{{\bm{U}}}
\def\mV{{\bm{V}}}
\def\mW{{\bm{W}}}

\def\mY{{\bm{Y}}}
\def\mZ{{\bm{Z}}}

\DeclareMathAlphabet{\mathsfit}{\encodingdefault}{\sfdefault}{m}{sl}
\SetMathAlphabet{\mathsfit}{bold}{\encodingdefault}{\sfdefault}{bx}{n}

\newcommand{\E}{\mathbb{E}}

\newcommand{\R}{\mathbb{R}}

\DeclareMathOperator*{\argmin}{arg\,min}

\newcommand{\Rmn}{\mathbb{R}^{m\times n}}
\newcommand{\Rmm}{\mathbb{R}^{m\times m}}
\newcommand{\Rnn}{\mathbb{R}^{n\times n}}

\newcommand{\Rmr}{\mathbb{R}^{m\times r}}
\newcommand{\Rrn}{\mathbb{R}^{r\times n}}

\DeclareMathOperator{\svd}{SVD}
\DeclareMathOperator{\eig}{EVD}

\DeclareMathOperator{\qr}{QR}
\DeclareMathOperator{\tr}{Tr}
\DeclareMathOperator{\vect}{Vec}
\DeclareMathOperator{\devect}{Mat}
\DeclareMathOperator{\normalize}{Norm}
\DeclareMathOperator{\whiten}{Whitening}
\newcommand{\Ft}{\tilde{\bm{F}}}
\newcommand{\family}{\mathcal{H}}
\newcommand{\Fnorm}[1]{\Vert#1\Vert_F^2}

\DeclareMathOperator{\diag}{Diag}
\newcommand{\vecg}{\vec{\bm{g}}}
\newcommand{\Rnr}{\mR_n^{\frac{1}{2}}}
\newcommand{\Lmr}{\mL_m^{\frac{1}{2}}}
\newcommand{\vecgt}{\overrightarrow{[\mG^T]}}
\newcommand{\mUf}{\mU_{f}}

\newcommand{\elesquare}{^{\odot 2}}

\DeclareMathOperator{\diagb}{Diag_B}
\DeclareMathOperator{\diagv}{Diag_v}
\DeclareMathOperator{\diagm}{Diag_M}

\usepackage{ulem}
\usepackage{pifont}

\theoremstyle{plain}
\newtheorem{theorem}{Theorem}[section]
\newtheorem{proposition}{Proposition}
\newtheorem{lemma}{Lemma}

\theoremstyle{definition}
\newtheorem{definition}[theorem]{Definition}

\theoremstyle{remark}

\usepackage{amsmath}

\usepackage[textsize=tiny]{todonotes}

\usepackage[acronym, style=list]{glossaries}
\newacronym{fim}{FIM}{Fisher information matrix}
\newacronym{f-norm}{F-norm}{Frobenius norm}
\newacronym{ngd}{NGD}{natural gradient descent}
\newacronym{spd}{SPD}{symmetric positive definite}
\newacronym{ema}{EMA}{exponential moving average}
\newacronym{sve}{SVE}{singular value editing}
\newacronym{evd}{EVD}{eigen-value decomposition}
\newacronym{alicec}{Eigen-Adam}{Eigenspace Adam}
\newacronym{asham}{ASH}{\underline{A}damized-\underline{Sh}ampoo}
\newacronym{noema}{SWAN}{stateless optimizer}
\newacronym{alice}{Alice}{\underline{A}daptive \underline{l}ow-d\underline{i}mensional subspa\underline{c}e \underline{e}stimation}
\newacronym{alicez}{Alice-0}{Alice without tracking}
\newacronym{tracking}{tracking}{tracking}
\newacronym{ssgd}{RACS}{Row and Column Scaled SGD}
\newacronym{racs}{RaCS}{Row and Column Scaled SGD}
\newacronym{svd}{SVD}{singular value decomposition}
\newacronym{sgd}{SGD}{stochastic gradient descent}
\makeglossaries

\usepackage[capitalize,noabbrev]{cleveref}

\renewcommand{\algorithmiccomment}[1]{\hfill\(\triangleright\){\textcolor{blue}{\textit{#1}}}}

\title{Towards Efficient Optimizer Design for LLM via Structured Fisher Approximation with a Low-Rank Extension}

\author{%
Wenbo Gong$^*$ \\
Microsoft Research \\
\texttt{wenbogong@microsoft.com} \\
\And
Meyer Scetbon$^*$ \\
Microsoft Research \\
\texttt{t-mscetbon@microsoft.com} \\
\And
Chao Ma$^*$ \\
Microsoft Research \\
\texttt{chao.ma@microsoft.com} \\
\And
Edward Meeds \\
Microsoft Research \\
\texttt{edward.meeds@microsoft.com} \\
}

\crefname{equation}{Eq.}{Eqs.}
\Crefname{equation}{Eq.}{Eqs.}
\crefname{appendix}{App.}{Apps.}   
\Crefname{appendix}{App.}{Apps.}   

\crefname{figure}{Fig.}{Figs.}     
\Crefname{figure}{Fig.}{Figs.}     
\crefname{section}{Sec.}{Secs.}
\Crefname{section}{Sec.}{Secs.}

\begin{document}

\maketitle

\def\thefootnote{*}\footnotetext{These authors contributed equally to this work.}



\begin{abstract}
Designing efficient optimizers for large language models (LLMs) with low-memory requirements and fast convergence is an important and challenging problem. 
This paper makes a step towards the systematic design of such optimizers through the lens of structured \gls{fim} approximation. We show that many state-of-the-art efficient optimizers can be viewed as solutions to \gls{fim} approximation (under the Frobenius norm) with specific structural assumptions. Building on these insights, we propose two design recommendations of practical efficient optimizers for LLMs, involving the careful selection of structural assumptions to balance generality and efficiency, and enhancing memory efficiency of optimizers with general structures through a novel low-rank extension framework. We demonstrate how to use each design approach by deriving new memory-efficient optimizers: \gls{ssgd} and \gls{alice}. Experiments on LLaMA pre-training (up to 1B parameters) validate the effectiveness, showing faster and better convergence than existing memory-efficient baselines and Adam with little memory overhead. Notably, \gls{alice} achieves better than $2\times$ faster convergence over Adam, while \gls{ssgd} delivers strong performance on the 1B model with SGD-like memory.
\end{abstract}

\section{Introduction}
\label{sec: introduction}
Adaptive optimizers are critical in training large language models (LLMs). Yet, as models and datasets continue to grow, one important issue associated with scaling is the memory overhead of many optimizers, especially in a distributed training setup \citep{llama3, korthikanti2023reducing}. This has several implications for training, including increased GPU requirements or reduced per-device batch size, which lowers overall training throughput. Adam, for instance, triples memory requirements due to the storage of two internal \gls{ema} states, while other optimizers \citep{gupta2018shampoo, vyas2024soap} with faster convergence (in terms of training steps) can further inflate the total memory.  Meanwhile, some memory efficient optimizers, like \gls{sgd}, fail to train the LLMs effectively. Thus, designing efficient optimizers  has become increasingly important.\footnote{We define efficiency as the amount of memory and wall-clock time used to achieve a target evaluation loss}

There are several promising lines of research that have advanced one or more of these aspects of efficiency. One aims to design new optimizers that reduce, or remove entirely, internal optimizer states without expensive per-step computation \citep{ma2024swan, jordan2024muon, zhang2024adam, xu2024adamlearningratescaling, zhu2024apollo}. Alternatively, low-rank approximations of gradients have also been used to reduce the memory of states with a slight degradation in performance \citep{hu2021lora, lialin2023relora, zhao2024galore, chen2024fira, si2024flora}. 
Despite these advances, developing new optimizers remains challenging. To address this, this paper explores structured \gls{fim} approximation as a practical framework for optimizer design.

To demonstrate the effectiveness of this framework, we begin by showing that many existing optimizers and gradient operators, including Adam, Shampoo, gradient normalization and whitening \citep{zhang2024adam, gupta2018shampoo, vyas2024soap, you2019lamb, ma2024swan, jordan2024muon}, can be recast under it, with different structural assumptions. We then go on to show how to derive two generalizations of Adam\footnote{Generality of structural assumption defines how relaxed this assumption is. We say structure $A$ is more general than $B$ iff.~$B$ can be recovered by applying further constraints on $A$. The general structures tend to give better approximations to \gls{fim} than the less general one. } with structures based on block diagonal matrices and Kronecker products, named \gls{alicec} and SOAP/AdaDiag++ \citep{vyas2024soap, anonymous2024improving}. 
Although this framework provides a clear link between structures and optimizers, working with more general structures that improve the \gls{fim} approximation can come at the cost of efficiency. For example, SOAP can require $7$ times more memory than SGD.  

Building on these insights, our first design recommendation proposes to choose structural assumptions that balance generality with practical efficiency. We demonstrate this by choosing a structure that generalizes gradient normalization, leading to a new efficient optimizer, \glsreset{ssgd}\gls{ssgd}, with \gls{sgd}-like memory requirements.

For optimizers with more general structures it may not be always possible to achieve such a balance. Instead of rejecting the potential of these generalizations, our second design recommendation proposes to apply a novel low-rank extension framework to improve their efficiency. This framework consists of three steps that can convert full-rank optimizers with more general structures into their low-rank approximations with reduced memory and computational costs. 
We demonstrate this by deriving a low-rank extension of \gls{alicec}, called \glsreset{alice}\gls{alice}. 

On experiments of pre-training LLaMA \citep{touvron2023llama} with C4 dataset \citep{raffel2020C4}, we demonstrate that \gls{ssgd} and \gls{alice} consistently outperform Adam and several memory-efficient baselines. \gls{alice} achieves better than $2\times$ speed-ups compared to Adam, and \gls{ssgd} performs strongly on pre-training the 1B LLaMA. Additionally, our preliminary results indicate that the 1B model trained by \gls{alice} reaches evaluation perplexity on-par or better than the 7B model trained with other memory-efficient baselines. 

To summarize, our contributions are:
\begin{itemize}
    \item We propose structured \gls{fim} approximation as a practical framework for optimizer design and show that several existing optimizers can be viewed as special cases within this framework.
    \item We propose to design optimizers by choosing structures that balance generality and efficiency and demonstrate it with a new optimizer, \gls{ssgd}.
    \item We present a low-rank extension framework to convert full-rank optimizers to corresponding low-rank approximations and demonstrate it with a new optimizer \gls{alice}.
    \item We demonstrate the effectiveness of \gls{ssgd} and \gls{alice} on LLaMa pre-training tasks when compared to Adam and other baselines. 
\end{itemize}
\section{Preliminaries}
\label{sec: preliminary}
\subsection{Basic notations and setup}
\label{subsec: notations and setups}
Throughout the paper we consider  $2$D matrix parameters $\mW$ (i.e. layer weights) of size $m\times n$ and $m\leq n$; the operation $\vect(\cdot)$ vectorizes the input matrix by stacking its columns; $\devect(\cdot)$ is the inverse of $\vect(\cdot)$ and reshapes the vector back into a matrix. We use $\mathcal{L}_\theta$ as the loss function where $\theta = \vect(\mW)$. 
$\mG=\nabla_{\mW}\mathcal{L}$ is the matrix gradient and $\vecg$ is the vectorized gradient, i.e.~$\vecg = \vect(\mG)$. $\vg_i$ denotes the $i^{\text{th}}$ column of $\mG$. In the paper, we assume $\mW$ are parameters of one layer. $\otimes$ indicates the Kronecker product. By default, we use $\mM\elesquare$ and $\sqrt{\mM}$ to denote the element-wise square and square-root of matrix, and $\mM^2$ and $\mM^{\frac{1}{2}}$ indicate the matrix square and square-root by default. For vectors $\vv$, both $\vv^{2}$ and $\vv^{\frac{1}{2}}$ represent element-wise operations. 
$\eig(\mM,r)$ performs the \gls{evd} and keeps the top $r$ eigenvectors ordered by the descending eigenvalues. If $r$ is omitted, we keep all eigenvectors. $\qr(\mM)$ with orthonormal $\mM\in\Rmr$ obtains the remaining $m-r$ orthogonal basis via QR decomposition.

We also introduce the following diagonal operations:
$\diag(\mM)$ will extract the diagonals of $\mM$ to a vector. $\diagb(\mM_1,\ldots,\mM_n)$ will stack the input sequence of matrices to form a larger block diagonal matrix. $\diagv(\vv)$ will expand the input vector $\vv$ to a pure diagonal matrix. $\diagm(\mM)$ will expand the input $\mM$ to a larger pure diagonal matrix by stacking the element of $\mM$ in a column-wise order. We demonstrate the above operations with examples in \cref{app: example diagonal operations}.
Note that results in this paper may involve the expectation $\E[\cdot]$ of some quantities, we assume the use of an \gls{ema} as its practical estimation. 
\subsection{Fisher information and natural gradient descent}
\label{subsec: fisher info}
Fisher information is a fundamental concept in statistics that measures the amount of information a random variable carries about a parameter of interest. In this paper, we ignore dependence between layers and treat each layer independently. 
Under the context of LLMs, with the vectorized mini-batched gradient of one layer $\vecg$, we define the empirical \gls{fim} for that layer as $\mF = \E[\vecg\vecg^T]$.
\Gls{ngd} leverages the inverse of \gls{fim} to smooth the local geometry to improve convergence and stabilize training \citep{martens2020new}. 
In practice, the square-root inverse of \gls{fim} may be preferred due to its properties and improved performance on certain optimization problems \citep{yang2008principal, lin2024can, loshchilov2016sgdr, bergstra2012random, choi2019empirical}.  
The corresponding update of $\mW$ is:
\begin{align}
    \mW\leftarrow \mW-\lambda \devect(\mF^{-\frac{1}{2}}\nabla_{\theta}\mathcal{L}).
    \label{eq: square root ngd}
\end{align}
Due to the large size of $\mF\in\R^{mn\times mn}$, computing its inverse is a significant impediment to applying this update rule in practice. One way to address this is to enforce certain structures to approximate \gls{fim}. In this paper, we consider two main classes of structures: block diagonal and Kronecker product matrices. They possess favorable properties that can significantly reduce computational complexity. \Cref{subapp: kronecker product and block diagonals} provides a brief introduction and summary of their key properties. We also include a comprehensive background in \cref{app: background} to be more self-contained.
\section{Structural Approximation to \gls{fim}}
\label{sec: structure approx fim}
\begin{table*}[]
\centering
\caption{The summary of underlying structure assumptions of existing and newly proposed optimizers, along with practical efficiency. ``Generalizes" refers to the optimizer whose structure is generalized, e.g.~\gls{alicec}'s structure generalizes Adam (see \cref{subapp: connections between structural assumptions} for further details). Optimizers in \textcolor{blue}{blue} color are new. The ``Computation" is the cost of per-step of the optimizer update. We assume $n\geq m\gg r$. }
\label{tab: summary table}
\resizebox{\textwidth}{!}{
\begin{tabular}{l|lllllll}
\hline
            & Adam          & Shampoo                                          & Eigen-Adam\slash AdaDiag                                                                                                        &  SOAP/AdaDiag++                                                  & GaLore          & \textcolor{blue}{RACS}            & \textcolor{blue}{Alice}             \\ \hline
Stucture    & $\diagv(\vv)$ & $\mR_n^{\frac{1}{2}}\otimes \mL_m^{\frac{1}{2}}$ &$\diagb(\{\mUf\mD_i\mUf^T\}_i)$ &  $(\mU_R\otimes\mU_L)\tilde{\mD}(\mU_R\otimes\mU_L)^T$ & Approx. Alice   & $\mS\otimes\mQ$ & Low-rank to Eigen-Adam \\
Generalizes & N/A           & N/A                                              &Adam                                                                                                         &  Eigen-Adam + Shampoo                                       & N/A             & Normalization   & GaLore            \\
Computation & $O(mn)$       &$O(m^3+n^3)$                                     & $O(m^3)$                                          & $O(m^3+n^3)$                                                                                                     &  $O(mnr+m^2n/K)$ & $O(mn)$         & $O(mnr+m^2r/K)$   \\
Memory      & $3mn$         & $mn+m^2+n^2$                                     &$3mn+2m^2$                                                                                                   &  $3mn+2m^2+2n^2$                                       & $mn+2nr+mr$     & $mn+m+n+1$      & $mn+2nr+mr+n+r^2$ \\ 
Full-rank update& \checkmark & \checkmark & \checkmark & \checkmark & \ding{55} & \checkmark & \checkmark\\
Section & \cref{subsec: existing optimizer} &\cref{subsec: shampoo} & \cref{subsec: alicec} &  \cref{subsec: new opt combination} & \cref{subapp: Low rank optimizers} & \cref{sec: ssgd} & \cref{subsec: alice optimizer}\\
\hline
\end{tabular}}
\end{table*}
The structured \gls{fim} approximation framework consists of three steps: first, choose a structure to approximate \gls{fim} $\Ft$; second, find a solution for the approximation by minimizing the following objective:
\begin{equation}
    \min_{\Ft\in\family} \Fnorm{\Ft-\bm{F}},
    \label{eq: UFE equation}
\end{equation}
where $\bm{F}$ is the empirical \gls{fim}, $\family$ is the matrix family corresponding to the structural assumption; third, derive the square-root \gls{ngd} update \Cref{eq: square root ngd} with approximated $\Ft$. In this section, we show that many existing optimizers and gradient operators can be reformulated within this framework by varying structural assumptions, thereby establishing connections between the choice of structure and optimizers.
\subsection{Adam: purely diagonal structure}
\label{subsec: existing optimizer}
There have been many work arguing that Adam's second moment aims to approximate the diagonal of \gls{fim} \citep{kingma2014adam, hwang2024fadam, sun2021connection}. Although it is easy to prove that this is, in fact, optimal approximation under \cref{eq: UFE equation}, we include the following proposition for completeness. 
\begin{proposition}[diagonal approximation]
    Assuming $\family =\{\diagv(\vv); v_i>0\}$, then \cref{eq: UFE equation} has analytic solution 
    \begin{equation}
        \Ft^* = \diagv(\E[\vecg^2])
    \end{equation}
where $\vecg^2$ indicates the element-wise square of $\vecg=\vect(\mG)$. 
\label{prop: adam solution}
\end{proposition}
It is trivial to show the resulting square-root \gls{ngd} recovers Adam's second moment when using the \gls{ema} to estimate $\E$. Together with the first moment, one can recover Adam.

\subsection{Shampoo: Kronecker product structure}
\label{subsec: shampoo}
Although previous work \citep{morwani2024new} has demonstrated the connection of Shampoo \citep{gupta2018shampoo} to the Kronecker product approximation to \gls{fim} through power iteration algorithm, 
we will make its connection to \cref{eq: UFE equation} more explicit and provide an alternative view: Shampoo's pre-conditioner can be derived by minimizing an upper bound of \cref{eq: UFE equation} with this structural assumption: 
\begin{align*}
    \family=\{\Rnr \otimes \Lmr; \mR_n\in\Rnn, \mL_m\in\Rmm\}
\end{align*}
where $\mR_n$ and $\mL_m$ are \gls{spd} matrices.
\begin{theorem}[Shampoo pre-conditioner]
    Assume the above structural assumption, then we have an upper bound of \cref{eq: UFE equation}
    \begin{align}
        \Fnorm{\Ft-\mF} &\leq 3(mn\Vert\mA^2-\mC^2\Vert_F\Vert\mB^2-\mC^2\Vert_F\nonumber\\
        &+ \sqrt{mn}\Fnorm{\mC}(\Vert\mA^2-\mC^2\Vert_F+\Vert\mB^2-\mC^2\Vert_F))
        \label{eq: shampoo upper bound}
    \end{align}
    where $\mA = \Rnr\otimes \mI_m$, $\mB = \mI_n \otimes \Lmr$, $\mC=\E[\vecg\vecg^T]^{\frac{1}{2}}$. Minimizing the upper-bound admits
    \begin{align*}
        \mR_n^* = \frac{1}{m}\E[\mG^T\mG],\;\;\; \mL_m^* = \frac{1}{n}\E[\mG\mG^T]
    \end{align*}
    \label{thm: optimal shampoo}
\end{theorem}
\cref{subapp: Shampoo update formula} shows that the corresponding square-root \gls{ngd} leads to the Shampoo's update formula. Therefore, the structure behind Shampoo is the Kronecker product of two square-root \gls{spd} matrices.

\subsection{Normalization and Whitening: Simple block diagonal structures}
\label{subsec: sve}
For an input $\mG$, the whitening and normalization operator are defined as 
\begin{align*}
    \whiten(\mG) =& ~(\mG\mG^T)^{-\frac{1}{2}}\mG\\
    \normalize(\mG) =& ~\mG\mS^{-\frac{1}{2}}
    \label{eq: whiten and normal operator}
\end{align*}
where $(\cdot)^{-\frac{1}{2}}$ denotes square root inverse, and $\diag(\mS)$ contains the squared column $l_2$ norm of $\mG$. Next we provide an new interpretation of these operators and show that they are the square-root \gls{ngd} updates under the following structural assumptions:
\begin{align}
    \family =& \{\mI_n\otimes \mM\} \;\;\; (\text{Whitening}) \text{\footnotemark}  \\
    \family =& \{\mS\otimes \mI_m; S_{ii}>0,\; \forall i\} \;\;\; (\text{Normalization}) 
\end{align}
\footnotetext{Note that this structure has also been proposed and discussed  \cite{duvvuricombining} under a slightly different setting.}
where $\mM\in\Rmm$ is \gls{spd} and $\mS\in\Rnn$ is a positive diagonal matrix. Given those structural assumptions, one can show:

\begin{proposition}[Normalization and whitening]
    Assuming $\family = \{\mI_n\otimes \mM\}$, minimizing \cref{eq: UFE equation} admits
    \begin{align}
        \mM^* = \frac{1}{n}\E[\mG\mG^T]
        \label{eq: generalization whitening}
    \end{align}
    If one assumes $\family = \{\mS\otimes \mI_m; S_{ii}>0,\; \forall i\}$, then the corresponding solution is 
    \begin{align}
        \mS^* = \frac{1}{m}\E[\diagv(\vg_1^T\vg_1,\ldots,\vg_n^T\vg_n)]
    \label{eq: generalization to normalization}
    \end{align}
    \label{coro: generalization to whitening and normalization}
\end{proposition}
The proof can be found in \cref{subapp: proof normalization}, where we prove a much more general solution (\cref{thm: generalization to normal and whiten}), and \cref{coro: generalization to whitening and normalization} is a special case. The corresponding square-root \gls{ngd} update with \cref{eq: generalization whitening} is $\devect(\Ft^{-\frac{1}{2}}\vecg)= \sqrt{n}\E[\mG\mG^T]^{-\frac{1}{2}}\mG$ (refer to \cref{subapp: update of generalization of whitening}). Therefore, we can view $\whiten(\cdot)$ as a special case with one-sample estimate for $\E$. Similarly, normalization is the square-root \gls{ngd} update ($\devect(\Ft^{-\frac{1}{2}}\vecg)= \sqrt{m}\mG\mS^{*-\frac{1}{2}}$) with \cref{eq: generalization to normalization} and one-sample estimate of $\E$.

Many recently proposed optimizers, such as Muon, SWAN, LARS and LAMB \citep{jordan2024muon, ma2024swan, you2017lars, you2019lamb}, rely on normalization and/or whitening. These gradient operators improve convergence \citep{you2017lars, jordan2024muon} and can replace Adam’s internal states \citep{ma2024swan}. See \cref{subapp: connection to existing optimizers} for a detailed discussion.

\subsection{\gls{alicec}: Generalization to Adam with eigenspace rotation}
\label{subsec: alicec}
\label{subsubsec: alice-c optimizer}
All structures considered above are simple and do not strictly generalize Adam's purely diagonal structure. In this and the following sections, we consider two structures that strictly generalize Adam, normalization, and whitening. Here, we first consider a block diagonal matrix with a shared eigen-space. 

Precisely, we propose the following structure that generalizes Adam\footnote{In \cref{subapp: solution to general block diagonal}, we propose an even more general block diagonal structure.}:
\begin{align}
    \family = \{\diagb(\mM_1, \ldots, \mM_n); \mM_i =\mUf\mD_i\mUf^T\}
    \label{eq: alicec structure}
\end{align} 
where $\mUf$ defines a shared \textbf{full-rank} eigen-space, and $\mD_i$ is a positive eigenvalue matrix. Adam is a special case by constraining $\mUf=\mI$. Additionally, the structures in \cref{subsec: sve} are also special cases. Whitening is obtained by a shared $\mD$ (i.e.~$\mD_i=\mD$); and normalization is by $\mUf=\mI$, $\mD_i=s_i\mI$. 
However, this structure does not directly lead to an analytic solution for \cref{eq: UFE equation}. 
Instead, we propose to approximate the solution by solving 1-iteration alternating optimization: 

\begin{theorem}[1-iteration refinement]
    For the structure in \cref{eq: alicec structure}, we consider the following 1-iteration alternating optimization of objective (\ref{eq: UFE equation}): (i) constrain $\mD_i=\mD$ to be equal, and solve $\mUf^* = \arg \min_{\mUf,\mD} \Fnorm{\diagb(\mUf\mD_1\mUf^T, \ldots, \mUf\mD_n \mUf^T) -\bm{F}} $; (ii) fix the $\mUf^*$ and find $\{\mD_i^*\} = \arg \min_{\{\mD_i\}} \Fnorm{\diagb(\mUf^*\mD_1{\mUf^*}^T, \ldots, \mUf^*\mD_n {\mUf^*}^T) -\bm{F}}$. Then, (i) and (ii) admits the following analytic solution:
    \begin{equation}
        \mUf^* = \eig(\E[\mG\mG^T]).
    \end{equation}
    where $\eig$ is the eigenvalue decomposition; and: 
    \begin{equation}
        \tilde{\mD^*} = \diagm(\E[({\mUf^*}^T\mG)\elesquare])
    \end{equation}
    where $\tilde{\mD}^* = \diagb(\mD_1^*,\ldots,\mD_n^*)$. 
    \label{thm: alicec 1 step refinement}
\end{theorem}
Based on this result, we can derive the corresponding square-root \gls{ngd} with given $\mU$ (refer to \cref{subapp: update of generlized adam}):
\begin{equation}
    \devect(\Ft^{-\frac{1}{2}}\vecg) = \mUf\frac{\mUf^T\mG}{\sqrt{\E[(\mUf^T\mG)\elesquare]}}.
    \label{eq: alicec update}
\end{equation}
This can be viewed as applying Adam's update on a space ``rotated" by eigen-matrix $\mUf$.
Consequently, we propose an optimizer, called \gls{alicec}, with the following update procedures:
\begin{align}
    &\vm_{t} = \beta_1\vm_{t-1}+(1-\beta_1)\mG_t\;\;\; (\text{first moment})\nonumber\\
    &\mQ_{t} = \beta_3\mQ_{t-1}+(1-\beta_3)\mG_t\mG_t^T\;\;\;(\text{\gls{ema} for }\E[\mG_t\mG_t^T])\nonumber\\
    & \mU_{f,t} = \eig(\mQ_{t})\nonumber\\
    & \vv_t = \beta_2\vv_{t-1}+(1-\beta_2)(\mU_{f,t}^T\mG_t)\elesquare\;\;\;(\text{second moment})\nonumber\\
    &\Delta_t = \mU_{f,t}\frac{\mU_{f,t}^T\vm_{t}}{\sqrt{\vv_t}}
\label{eq: practical alicec equations}    
\end{align}
\Cref{alg: alicec optimizer} in \cref{subapp: AdaDiag} summarizes the practical procedure. 
In fact, the above procedures closely relates to two related works: AdaDiag and one-sided SOAP \citep{anonymous2024improving, vyas2024soap}, which are heuristic memory-efficient variants of the full algorithms (i.e.~AdaDiag++ and SOAP). Before we discuss their connections, next we first show that the SOAP optimizer can also be reformulated as FIM approximation under a specific structural assumption.

\subsection{SOAP: Combination of Kronecker product with block diagonal structure}
\label{subsec: new opt combination}
All previous structures, apart from the one behind Shampoo, are under the class of block diagonal structures. However, such block-diagonal structure does not takes into account the off-diagonal part of \gls{fim}. Structure under Kronecker product, like the one behind Shampoo, can go beyond this. Therefore, we can consider combining the structure of \gls{alicec} with Shampoo, to obtain a more general structural assumption. We show this exactly recovers SOAP \citep{vyas2024soap}.

Specifically, we consider the following structural assumption: 
\begin{equation}
    \family = \{(\mU_R\otimes \mU_L)\tilde{\mD}(\mU_R\otimes \mU_L)^T\}
    \label{eq: SOAP structure}
\end{equation}
where $\mU_R\in\Rnn$, $\mU_L\in\Rmm$ are orthonormal matrix, and $\tilde{\mD}\in\mathbb{R}^{mn\times mn}$ is a diagonal matrix with positive values. We can easily show that structure behind \gls{alicec} is a special case by constraining $\mU_R=\mI_n$; and Shampoo is also a special case by constraining $\tilde{\mD}$ to be decomposed by Kronecker product (refer to \cref{subapp: connections between structural assumptions}). 

Similar to \gls{alicec}, it is hard to directly minimizing \cref{eq: UFE equation} with this assumption. We can approximate the solution by a similar 1-iteration alternating optimization procedure as \gls{alicec}. 

\begin{theorem}[SOAP as 1-iteration alternating optimization of \Cref{eq: UFE equation}]
    Assuming the above structural assumptions.
    Consider the following 1-iteration aternating optimization of \Cref{eq: UFE equation}: (i) assuming $\tilde{\mD}$ can be decomposed as Kronecker product of two diagonal matrix, then solve for $\mU_R^*$, $\mU_L^* = \arg \min_{\mU_R, \mU_L, \tilde{\mD}} \Fnorm{(\mU_R\otimes \mU_L)\tilde{\mD}(\mU_R\otimes \mU_L)^T -\bm{F}} $; (ii) fix $\mU_R^*$, $\mU_L^*$, solve for $\tilde{\mD}^* = \arg \min_{\tilde{\mD}} \Fnorm{(\mU_R^*\otimes \mU_L^*)\tilde{\mD}(\mU_R^*\otimes \mU_L^*)^T -\bm{F}}$ without Kronecker product assumption of $\tilde{\mD}$. Then,
    (i) admits analytic solution when minimizing the upper bound of \cref{eq: UFE equation} (i.e.~\cref{eq: shampoo upper bound}):
    \begin{align*}
        \mU_R^* = \eig(\E[\mG^T\mG]),\;\;\; \mU_L^* = \eig(\E[\mG\mG^T]).
    \end{align*} Step (2) admits an analytic solution for \cref{eq: UFE equation}:
    \begin{align*}
        \tilde{\mD^*} = \diagm(\E[({\mU_L^*}^T\mG{\mU_R}^*)\elesquare])
    \end{align*}. 
    \label{thm: optimal asham}
\end{theorem}
The proof is a straightforward combination of 
\cref{thm: optimal shampoo} and \cref{thm: alicec 1 step refinement}, and can be found in \cref{subapp: proof asham}. One can show that the corresponding square-root \gls{ngd} update associated with the above result exactly recovers the update rules in SOAP optimizer (refer to \cref{subapp: update formula for SOAP} for details).

\paragraph{Connections to \gls{alicec}} Compared to \gls{alicec}, SOAP follows a more general structural assumption. However, from the FIM approximation perspective, SOAP does not exactly solve the 1-iteration alternating optimization problem. Instead, when solving for $\mU_R^*, \mU_L^* = \arg \min_{\mU_R, \mU_L, \tilde{\mD}} \Fnorm{(\mU_R\otimes \mU_L)\tilde{\mD}(\mU_R\otimes \mU_L)^T -\bm{F}}$, SOAP minimizes the upper bound instead. On the contrary, \gls{alicec} exactly solves the 1-iteration refinement problem. In addition, the structures behind \gls{alicec} and SOAP are different, and \gls{alicec} should not be viewed as a simple variant of SOAP.

\section{\gls{ssgd}: memory-efficient optimizer from a carefully selected structure}
\label{sec: ssgd}
The structured \gls{fim} approximation reveals two important insights: there exists a correspondence between structural assumption and optimizers, and structural generality often comes at the cost of practical efficiency. For example, while the structures of \gls{alicec} and SOAP offer more accurate \gls{fim} approximations than a simple structure like gradient normalization, they require expensive computation and memory consumption (\cref{tab: summary table}), making them impractical for training LLMs. Building on this, our first design recommendation is to \textbf{select structures that balance generality and practical efficiency.} 

To demonstrate this, we select a structure that generalizes gradient normalization, which scales both the rows and columns simultaneously:
\begin{align}
    \family = \{\mS \otimes \mQ\}
    \label{eq: ssgd structure}
\end{align}
where $\mS\in\Rnn$, $\mQ\in\Rmm$ are positive diagonal matrices. The idea of diagonal approximation has also been leveraged in previous work under different setups \citep{zhao2024adapprox, shazeer2018adafactor, li2018preconditioner}. 
The optimal solution of \cref{eq: UFE equation} can be solved by a fixed-point iterative procedure:

\begin{proposition}[Two-sided scaling]
Assuming the structure of \cref{eq: ssgd structure}, and $\E[\mG\elesquare]$ contains only positive values, solving \cref{eq: UFE equation} admits an iterative fixed point procedure:
\begin{align}
    \vs = \frac{\diag\left(\E[\mG^T\mQ\mG]\right)}{\Vert\mQ\Vert_F^2},\;\;\;
    \vq =\frac{\diag\left(\E[\mG\mS\mG^T]\right)}{\Vert\mS\Vert_F^2}.
    \label{eq: iterative procedure double scaling}
\end{align}
where $\vs=\diag(\mS)$, $\vq=\diag(\mQ)$.
Additionally, the fixed point solution $\vs$, $\vq$ converges to the right and left principal singular vector of $\E[\mG\elesquare]$ up to a scaling factor with unique $\mS^*\otimes \mQ^*$.
\label{prop: two sided scaling}
\end{proposition}
In practice, we find initializing $\vq=\bm{1}$ and use 1-sample estimate of $\E[\cdot]$ gives good performance. 
Interestingly, \citet{morwani2024new} also connects Shampoo to 1-step power iteration. Here, the \cref{eq: iterative procedure double scaling} can also be viewed as a power iteration algorithm. The main difference is that \citet{morwani2024new} assumes full \gls{spd} matrix $\mS$ and $\mQ$, but our structural constraint is positive diagonal matrix. Consequently, our procedure is computationally efficient and allows for multiple iterations.

The corresponding square-root \gls{ngd} update scales both rows and columns through $\mS$ and $\mQ$ (i.e.~$\devect(\Ft^{-\frac{1}{2}}\vecg) = \mQ^{-\frac{1}{2}}\mG\mS^{-\frac{1}{2}}$).
We name this optimizer, \glsreset{ssgd}\gls{ssgd} (\cref{alg: ssgd optimizer}). Although analytic solutions of $\vs$, $\vq$ exist, we perform $5$ steps of \cref{eq: iterative procedure double scaling} as an efficient approximation. To further stabilize training, we also incorporate the norm-growth limiter used in \citet{chen2024fira}. \gls{ssgd} is highly memory efficient since it only needs the storage of two diagonal matrices $\mS$ and $\mQ$ and a scalar for the limiter, consuming $m+n+1$ memory. In \cref{subapp: connection to existing optimizers} we discuss  connections to Adapprox, Apollo and Adafactor \citep{zhao2024adapprox, zhu2024apollo, shazeer2018adafactor}.
\begin{algorithm}
    \caption{\gls{ssgd}}
    \label{alg: ssgd optimizer}
    \begin{algorithmic}[1]
        \STATE {\bfseries Input:} learning rate $\lambda$, $\beta$, scale $\alpha$, limiter threshold $\gamma$, optimization steps $T$.
        \STATE $\vs_0 = 0$; $\vq_0=0$; $\phi_0=0$
        \FOR{$t=1,\ldots,T$}
            \STATE $\mG_t=\nabla_{\mW_t}\mathcal{L}$
            \STATE Obtain $\mS_t$ and $\mQ_t$ by \cref{eq: iterative procedure double scaling}
            \STATE $\vs_t=\beta\vs_{t-1}+(1-\beta)\diag(\mS_t)$; 
            \STATE $\vq_t=\beta\vq_{t-1}+(1-\beta)\diag(\mQ_t)$
            \STATE$\tilde{\mG}_t= \diagv(\vq_t)^{-\frac{1}{2}}\mG\diagv(\vs_t)^{-\frac{1}{2}}$
            \STATE $\eta =\gamma/\max\{\frac{\Vert\tilde{\mG}_t\Vert}{\phi_{t-1}}, \gamma\}$ if $t>1$ else $1$ 
            \STATE $\phi_t = \eta\Vert\tilde{\mG}_t\Vert$
            \STATE $\mW_{t+1}=\mW_t -\lambda\eta \alpha\tilde{\mG}_t$
        \ENDFOR
    \end{algorithmic}
\end{algorithm}

This design recommendation has its limitations. Finding such a structure with a balanced trade-off may not always be easy, and the resulting structure tends to be simple, offering less accurate approximation to \gls{fim} compared to the general ones. Since the main bottleneck of more general optimizers is their practical efficiency, our second design recommendation is to: \textbf{improve their efficiency by converting full-rank optimizers into low-rank counterparts using a novel low-rank extension framework}.

\section{\gls{alice}: memory-efficient optimizer from low-rank extension framework}
\label{sec: memory efficient opt}

In this section, we propose a novel low-rank framework consisting of three steps, low-rank \textbf{tracking}, subspace \textbf{switching}; and \textbf{compensation}. Tracking aims to reduce the memory cost of \gls{ema} states, whereas switching and compensation are designed to correct the potential issues caused by tracking and the limitations of low-rank projections. We demonstrate this procedure by converting \gls{alicec} to its low-rank version, \gls{alice}. While the procedure could be applied to SOAP in a similar manner, we leave this for future work. 

\paragraph{Reduce computational cost}
To improve the computational efficiency, we make two modifications to \gls{alicec}. First, we propose to use 1-step subspace iteration algorithm as an efficient scheme to find leading eigenvectors (\cref{alg: subspace iteration} in \cref{app: background}). Second, we only update projection $\mU$ every $K$ steps, effectively partitioning the training into time blocks with size $K$, and amortizing the cost. Consequently, $\mU$ is fixed within a time block. 

\subsection{Tracking: low-rank projections to reduce memory}
By carefully examining the \gls{alicec}'s procedure (\cref{eq: alicec update}), we notice all internal states are connected through the projection $\mU_{f,t}$. To reduce the memory cost, we can obtain a low-rank $\mU_t$ by keeping only the top $r$ eigenvectors, and denote the remaining $m-r$ basis as $\mU_{c,t}$ (i.e.~$\mU_{f,t}=[\mU_t, \mU_{c,t}]$). For tracking state $\mQ_t$, we can also apply low-rank approximation and only track the projected states. We call it \textbf{low-rank tracking}:
\begin{align}
    \bm{\sigma}_t = \mU_t^T\mG_t;\;\;\;\widetilde{\mQ}_{t+1} = \beta_3\widetilde{\mQ}_{t} + (1-\beta_3)\bm{\sigma}_{t}\bm{\sigma}_t^T
    \label{eq: tracking step}
\end{align}
where $\bm{\sigma}_t$ is the projected gradient, and $\widetilde{\mQ}_t$ is the low-rank tracking state. One can easily reconstruct back $\mQ_t \approx \mU_t\widetilde{\mQ}_t\mU_t^T$ when needed. This reduces the memory from $m^2$ to $r^2$. 

However, this low-rank projection comes with two major consequences: (1) the reconstructed state $\mQ_t$ is no longer accurate; (2) the resulting parameter update $\Delta$ in \cref{eq: alicec update} ignores the information within $\mU_{c,t}$ due to low-rank $\mU_t$. Next, we propose two additional steps, switching and compensation, rooted in theoretical insights to address these two problems, respectively.

\subsection{Switching: mixing leading basis with the complements}
\label{subsec: switching}
We omit the subscript $t$ in $\mU$ and $\mU_c$ in the following since they are fixed within a time block. 
Since the projected gradient $\bm{\sigma}_t$ only maintains the information in the space spanned by $\mU$, the low-rank tracking state $\widetilde{\mQ}_t$ necessarily discards information in $\mU_{c}$. Therefore, even if those directions should become the leading basis at the time we update the $\mU$, $\widetilde{\mQ}_t$ will ignore their contributions, causing the stability of the previous leading eigenvectors and preventing the exploration of other spaces. We prove that this is possible by showing that the true tracking state $\mQ_t$ can be decomposed into low-rank tracking reconstruction and a residual term quantifying the importance of $\mU_{c}$:
\begin{proposition}[Subspace switching]
    Assuming the setup mentioned above and all assumptions of \gls{alicec} are satisfied. We further assume the low-rank $\mU\in\Rmr$ is obtained at the beginning of $i+1$ time block by $\eig(\mQ^*_{ik},r)$ where $\mQ^*_{ik}$ is the true tracking state. Further, we assume the stability of the eigen-basis such that gradient $\mG_t$ during $i+1$ time block shares the same eigen-basis as $\mQ^*_{ik}$. Then, the true tracking state at the end of $i+1$ block, $\mQ^*_{(i+1)k}$, can be decomposed into:
    \begin{align}
        \mQ^*_{(i+1)k}=\sum_{t=ik+1}^{(i+1)k} \mG_t\mG_t^T = \sum_{t=ik+1}^{(i+1)k} \widetilde{\mG}_t\widetilde{\mG}_t^T+\mU_c\bm{\Sigma}_t\mU_c^T
        \label{eq: what is happening with low-rank U}
    \end{align}
    where $\widetilde{\mG}_t=\mU\bm{\sigma}_t$ is the low rank reconstructed gradients, $\bm{\Sigma}_t\in\mathbb{R}^{(m-r)\times (m-r)}$ is a diagonal matrix with positive values, and $\mU_c$ is the remaining eigen-basis such that $[\mU,\mU_c]$ will form the complete eigen-basis of $\mQ^*_{ik}$. 
    \label{prop: subspace switching}
\end{proposition}
When some entries in $\bm{\Sigma}_t$ become dominant, the corresponding basis in $\mU_c$ will form the new leading eigen-basis when updating the projection $\mU$ through $\eig(\mQ_{(i+1)k}^*)$. On the other hand, if $\mU$ is updated with low-rank reconstructed state alone, it is equivalent to setting $\bm{\Sigma}_t=0$, ignoring the contributions of $\mU_c$, and resulting in the stability of the previous leading basis. 
Inspired by this insight, we propose a practical scheme to update $\mU$, instead of relying on \gls{evd} of low-rank reconstructed state. We call it subspace switching (\cref{alg: subspace switching}). Intuitively, we force the new projection $\mU$ to mix the leading eigen-basis with randomly sampled eigenvectors from the remaining basis $\mU_c$, allowing the optimizer to explore other spaces. Although the true $\mU_c$ is hard to obtain in practice, we propose to approximate it by the QR decomposition of $\mU$. 

\begin{algorithm}
    \caption{Subspace switching}
    \label{alg: subspace switching}
    \begin{algorithmic}[1]
        \STATE {\bfseries Input:} Reconstructed state $\mQ$, rank $r$, leading basis number $l$, previous low-rank projection $\mU_{t-1}$
        \STATE $\mU_t'=\text{Subspace iteration}(\mQ, \mU_{t-1})$
        \STATE $\mU'_{t,1}\leftarrow \text{Keep top $l$ eigenvectors of }\mU'_t$ 
        \STATE Uniformly sample $r-l$ basis from the complement $\mU'_{c,t} = \qr(\mU'_{t})$ as $\mU'_{t,2}$
        \STATE Combine basis $\mU = [\mU'_{t,1},\mU'_{t,2}]$
        \STATE {\bfseries Return} $\mU$
    \end{algorithmic}
\end{algorithm}

\subsection{Compensation: convert low-rank update to be full-rank}
\label{subsec: compensation}
Another problem with low-rank projection $\mU$ is the information loss in the resulting parameter update $\Delta$ at each step. The goal of compensation step is to compensate for this information loss with minimal memory overhead. Firstly, we need to know which information has been discarded. We show that the update $\Delta$ with full-rank $\mU_{f}$ in \cref{eq: alicec update} can be decomposed into the low-rank update with $\mU$ and complement update controlled by $\mU_{c}$ (proof in \cref{subapp: discussion low-rank}):
\begin{align}
    \Delta =\mU\frac{\mU^T\mG}{\sqrt{\E[(\mU^T\mG)\elesquare]}}
    + \underbrace{\mU_{c}\frac{\mU_{c}^T\mG}{\sqrt{\E[(\mU_{c}^T\mG)\elesquare]}}}_{\devect(\Ft_{c}^{-\frac{1}{2}}\vecg)}
    \label{eq: alice update decomposition}
\end{align}
where $\Ft_{c}=\diagb(\mU_{c}\mD_{c,1}\mU_{c}^T,\ldots, \mU_{c}\mD_{c,n}\mU_{c}^T)$ is the approximated \gls{fim} corresponding to the complement basis $\mU_c$, $\mD_{c,i}=\diagv(\E[(\mU_c^T\vg_i)^2])$ and $^{-\frac{1}{2}}$ is the square-root pseudo-inverse. 

We notice that the discarded information, $\devect(\Ft_c^{\frac{1}{2}}\vecg)$, has the same format as the square-root \gls{ngd} with \gls{fim} $\Ft_c$. From the proposed \gls{fim} view point, the design of compensation term becomes the selection of a structure to approximate $\Ft_c$, and the application of the corresponding square-root \gls{ngd}. Considering the trade-off between structural generality and practical efficiency, one structural choice is the diagonal structure of normalization operator, which simply scales the columns of gradient matrix and is highly memory efficient. In addition, we only want to focus on the discarded information $\mU_c\mU_c^T\mG$ for compensation, rather than the entire gradient $\mG$. We propose the following compensation at each step $t$:
\begin{align}
    \mC_t= \devect((\mS\otimes\mU_c\mU_c^T)\vecg_t)=\mU_c\mU_c^T\mG_t\mS
    \label{eq: compensation term}
\end{align}
where $\mS$ is a positive diagonal matrix, $\mU_c\mU_c^T\mG=(\mG-\mU\mU^T\mG)$ is the gradient information within the remaining basis. We show that such design choice admits an optimal solution of $\mS$ to the \gls{fim} approximation problem.
\begin{theorem}[Optimal compensation]
    Assume that the conditions of \gls{alicec} are satisfied. With the proposed form of compensation (\cref{eq: compensation term}), minimizing \gls{fim} reconstruction loss
    \begin{align*}
        \Fnorm{(\mS_t^{-2}\otimes \mU_c\mU_c^T)-\Ft_{c}}
    \end{align*}
    admits analytic solution:
    \begin{equation}
        \diag(\mS_t) = \frac{\sqrt{m-r}}{\sqrt{\E[\bm{1}_m^T\mG_t\elesquare-\bm{1}_r^T(\mU^T\mG_t)\elesquare}]}
         \label{eq: optimal D for compensation}
    \end{equation}
    where $\bm{1}_m\in\mathbb{R}^m$, $\bm{1}_r\in\mathbb{R}^r$ are the column vectors with element $1$.
    \label{thm: optimal compensation}
\end{theorem}
\cref{alg: compensation} summarizes the compensation step.

\begin{algorithm}
    \caption{Compensation}
    \label{alg: compensation}
    \begin{algorithmic}[1]
        \STATE {\bfseries Input:} $\mG_t$, projection $\mU$, previous norm $\vp$, limiter norm $\bm{\phi}$, limiter threshold $\gamma$, decay rate $\beta$
        \STATE $\vp \leftarrow \beta\vp+(1-\beta)(\bm{1}_m^T\mG_t\elesquare-\bm{1}_r^T(\mU^T\mG_t)\elesquare)$
        \STATE $\mC_t\leftarrow\sqrt{m-r}(\mG_t-\mU\mU^T\mG_t)\diagv(\vp)^{-\frac{1}{2}}$
        \STATE $\eta=\gamma\slash \max\{\frac{\Vert\mC_t\Vert}{\phi},\gamma\}$ if $\phi>0$ else $1$        
        \STATE $\phi = \Vert\eta\mC_t\Vert$      
        \STATE $\mC_t\leftarrow \eta \mC_t$
        \STATE {\bfseries Return} $\mC_t$, $\vp$, $\phi$
    \end{algorithmic}
\end{algorithm}

\subsection{\Gls{alice} optimizer}
\label{subsec: alice optimizer}
By combining \gls{alicec} with low-rank $\mU$, tracking, switching and compensation, we obtain \gls{alice}, a novel low-rank optimizer. One can also design a simple variant, \gls{alicez}, by disabling the tracking for better memory efficiency.
\paragraph{Connections to GaLore}
Interestingly, GaLore, in fact, is an approximation to \gls{alice} without tracking, switching and compensation. Based on the connection of \gls{alice} to \gls{alicec}, we reveal that GaLore is a simple low-rank extension of \gls{alicec}, a more general optimizer than Adam. This also reflects the advantage of the \gls{fim} view point, which provides a deeper understanding and an explanation on its better performance than Adam under certain scenarios \citep{zhao2024galore}.

\begin{algorithm}
\caption{\gls{alice}/\gls{alicez} optimizer}
\label{alg: alice optimizer}
    \begin{algorithmic}[1]
        \STATE {\bfseries Input:} learning rate $\lambda$, scale $\alpha$, compensation scale $\alpha_c$, update interval $k$, $\beta_1$, $\beta_2$, $\beta_3$ ($\beta_3=0$ for \gls{alicez}),  optimization step $T$, rank $r$, loss function $\mathcal{L}$, limiter threshold $\gamma$, leading basis number $l$ 
        \STATE $\widetilde{\mQ}_0 = 0$, $\mU_0=0$, $\vp_0=0$, $\bm{\phi}=0$, $\vm_0=0$, $\vv_0=0$
        \FOR{$t=1\ldots, T$}
            \STATE $\mG_t=\nabla_{\mW_t}\mathcal{L}$
            \IF{$t==1$ or $(t\mod K)==0$}
                \STATE 
                $\mQ_t=\beta_3\mU_t\widetilde{\mQ}_{t-1}\mU_t^T+(1-\beta_3)\mG_t\mG_t^T$ 
                \STATE $\mU_t=Switch(\mQ_t, r, l, \mU_{t-1})$
            \ELSE
                \STATE $\mU_t=\mU_{t-1}$
            \ENDIF
            \STATE $\bm{\sigma}_t = \mU_t^T\mG_t$ 
            \STATE $\widetilde{\mQ}_t = \beta_3 \widetilde{\mQ}_{t-1} +(1-\beta_3)\bm{\sigma}_t\bm{\sigma}_t^T$ 
            \STATE $\vm_t = \beta_1 \vm_{t-1}+ (1-\beta_1)\bm{\sigma}_t$
            \STATE $\vv_t = \beta_2 \vv_{t-1} + (1-\beta_2) \bm{\sigma}_t\elesquare$
            \STATE $\bm{\omega} = \frac{\vm_t}{\sqrt{\vv_t}}$
            \STATE $\Delta_c, \vp_t, \bm{\phi}_t = Compensation(\mG_t, \mU_t, \vp_{t-1}, \bm{\phi}_{t-1}, \gamma, \beta_1)$ 
            \STATE $\mW_{t+1}=\mW_t-\lambda\alpha (\mU\bm{\omega}+\alpha_c\Delta_c)$
        \ENDFOR
    \end{algorithmic}
\end{algorithm}

\section{Related Work}
\label{sec: related work}
\paragraph{Optimizer based on structural approximation}
Due to the desirable properties and convergence of second-order optimization, various work has been proposed to efficiently approximate Hessian-like matrix, e.g.~\gls{fim}. KFAC \citep{martens2015optimizing} was one of the first work that goes beyond the simple diagonal approximations, and approximate the layer-wise \gls{fim}. Subsequent works extends KFAC beyond MLP layers \citep{grosse2016kronecker, martens2018kronecker}. Further refinements to KFAC are also proposed, including refinement of eigenvalues \citep{george2018fast}, fixing the trace \citep{gao2021trace}, and refinement by Kronecker product singular value decomposition \citep{koroko2022efficient}. Our proposed view point is different from KFAC, where KFAC decompose the \gls{fim} using the back-proped gradients and layer input. In addition, KFAC needs to be re-derived for different types of layers. On the other hand, our proposed view point is closer to another line of work, aiming to approximate the full AdaGrad \citep{duchi2011adaptive}. In particular, Shampoo \citep{gupta2018shampoo, anil2020scalable} is proposed as a Kronecker product approximation to AdaGrad. Later, \citep{morwani2024new} explicitly proved that it is a 1-step power iteration to optimal Kronecker product approximation. 
In here, we propose an alternative view of Shampoo as minimizing a upper bound of the approximation error. SOAP \citep{vyas2024soap} is a recently proposed adaptive optimizer that further improves Shampoo based on the insights from \citet{george2018fast}. 
In this work, we make \textbf{explicit} connection of those approaches to \gls{fim} approximation, and establish the equivalence of structural assumption to optimizers. We also additionally provide connections of gradient operators to \gls{fim} approximation, and design new optimizers from this view point. We provide discussions of our approach to many existing optimizers, including Apollo \citep{zhu2024apollo}, GaLore \citep{zhao2024galore}, Muon \citep{jordan2024muon}, SWAN \citep{ma2024swan}, Adapprox \citep{zhao2024adapprox}, Lars \citep{you2017lars}, Lamb \citep{you2019lamb}, Fira \citep{chen2024fira} and AdaDiag \citep{anonymous2024improving}, in \cref{subapp: connection to existing optimizers}. 
In addition to the above, preconditioning SGD (PSGD) \citep{li2017preconditioned, pooladzandi2024curvature} aims to directly approximate the inverse Hessian or \gls{fim} through different structural assumptions. \citet{li2018preconditioner} also discussed the diagonal Kronecker product structure as in \gls{ssgd}, but they apply this structural assumption under the framework of PSGD to directly approximate the inverse of \gls{fim} through gradient descent.  

\paragraph{Memory-efficient optimizer}
Practical efficiency is a crucial factor when training large models. In particular, there are many works that focus on optimizing memory efficiency, as less memory consumption allows larger batch size, effectively improving throughput.
There are two main lines of research: (1) use low-rank approximation to reduce memory of optimizer internal states; (2) remove the internal states.
GaLore \citep{zhao2024galore}, a well-know low-rank optimizer, proposed to use \gls{svd} for a low-rank projection, followed by applying Adam within it. It can be seen as a special case of \gls{alice} without tracking, switching and compensation.
Fira \citep{chen2024fira}, an extension to GaLore, adds compensation term to turn low-rank update to be full-rank, substantially improves the performance. Flora \citep{si2024flora} used randomly sampled Gaussian matrix as the subspace to save compute and memory. However, it is mainly focused on the fine-tuning tasks. ReLora \citep{lialin2023relora}, an extension to LoRA \citep{hu2021lora}, periodically merges the LoRA weights to enable full-rank learning. 
On the other hand, many optimizers require fewer internal states compared to Adam. Lion \citep{chen2024symbolic} and Signum \citep{bernstein2018signsgd} only require the storage of the first moment, offering a balance between memory efficiency and performance. Apollo \citep{zhu2024apollo}, a recently proposed approach, maintains a low-rank GaLore states (e.g.~Apollo-mini uses rank $1$) for estimating the scaling matrix for the raw gradient. Although it still requires GaLore states, using rank 1 allows it to achieve SGD-like memory. At the same time, \citet{ma2024swan} developed SWAN, which manages to completely removes the internal states through two gradient operators: normalization and whitening, and obtains stronger performance than Adam. In this paper, we also show that normalization and whitening operators are special cases of \gls{fim} approximation. 

\section{Experiments}
\label{sec: experiments}
We include all setup details along with additional experiment results in \cref{app: experiment details}.

\subsection{Pretraining LLaMA with C4 dataset}
\label{sbusec: experiment pretrain C4} 
\paragraph{Setup}
We evaluate the proposed \gls{ssgd}, \gls{alice} and its variant \gls{alicez} on pre-training LLaMA \citep{touvron2023llama} with the C4 dataset \citep{raffel2020C4}. We train the following model sizes: $60$M, $130$M, $350$M and $1.3$B using a similar setup as \citet{zhao2024galore, zhu2024apollo}. For baselines, we consider GaLore, Fira, Apollo-mini, Apollo-svd and Adam.
An important consideration in our experiments is that \textbf{all previous low-rank methods rely on full-rank Adam to train the last layer, which is arguably one of the most important layers} \citep{zhao2024deconstructing}. To thoroughly assess their effectiveness, we report performance for both cases when evaluating low-rank methods—training the last layer with and without Adam—but prioritize the latter as the main evaluation criterion. For full-rank methods (i.e.~\gls{ssgd}, Apollo-mini, Apolli-svd and Adam), we assume the last layer is trained by Adam. 

\begin{table}[]
\caption{LLaMA pretraining performance. Ppl.~is the evaluation perplexity when the last layer is not trained by Adam; and Ppl.* is when the last layer is trained by Adam. For Adam, we report both our reproduced performance and perplexity cited from \citet{zhu2024apollo}. We also cite Ppl.* of other baselines from \citet{zhu2024apollo}. 
The speed-up is measured against Adam in terms of training steps. The TP is the number of training tokens processed per second, and effective TP is total token used by Adam divided by time used by the candidate optimizer to reach the same final eval ppl.~of Adam.}
\label{tab: pretrain performance}
\centering
\resizebox{\textwidth}{!}{
\begin{tabular}{c|cccccccc}
\hline
Methods                         & \multicolumn{2}{c}{60M}          & \multicolumn{2}{c}{130M}         & \multicolumn{2}{c}{350M}         & \multicolumn{2}{c}{1.3B}         \\ \hline
                                & Ppl.            & Ppl.*          & Ppl.             & Ppl.*         & Ppl.             & Ppl.*         & Ppl.             & Ppl.*         \\
Adam                            & NA              & 33.94         & NA               & 25.03         & NA               & 19.24         & NA               & 16.44         \\ 
Adam (cited) & NA & 34.06 & NA & 25.08 & NA&18.80&NA&15.56\\
\hline
GaLore                          & 38.91           & 34.88          & 27.18            & 25.36         & 21.11            & 18.95         & 16.60            & 15.64         \\
Fira                            & 33.77           & 31.06          & 25.21            & 22.73         & 18.93            & 17.03         & 15.14            & 14.31         \\
Apollo-mini                     & NA              & 31.93          & NA               & 23.84         & NA               & 17.18         & NA               & 14.17         \\
Apollo-svd                      & NA              & 31.26          & NA               & 22.84         & NA               & 16.67         & NA               & 14.10         \\ \hline
RACS                            & NA              & 30.25          & NA               & 22.67         & NA               & \textbf{16.51}         & NA               & \textbf{13.52}         \\
Alice-0                         & \textbf{28.83}           & 29.74          & \textbf{21.99}            & 22.43         & \textbf{16.66}            & \textbf{16.43}         & \textbf{13.97}            & \textbf{13.47}         \\
Alice                           & \textbf{28.69}           & 29.33          & \textbf{21.95}            & \textbf{21.79}         & \textbf{16.61}            & \textbf{16.37}         & \textbf{13.85}            & \textbf{13.52}         \\ \hline
Speed-up in steps (Alice)       & \multicolumn{2}{c}{2.22x}        & \multicolumn{2}{c}{2.00x}        & \multicolumn{2}{c}{2.45x}        & \multicolumn{2}{c}{2.82x}        \\
Throughput(TP)/Effect TP (Adam) & \multicolumn{2}{c}{97748/97748}  & \multicolumn{2}{c}{82247/82247}  & \multicolumn{2}{c}{63139/63139}  & \multicolumn{2}{l}{53588/53588}  \\
Throughput(TP)/Effect TP(Alice) & \multicolumn{2}{c}{92589/\textbf{202058}} & \multicolumn{2}{c}{71583/\textbf{141148}} & \multicolumn{2}{l}{58847/\textbf{143088}} & \multicolumn{2}{c}{45523/123048} \\
Throughput(TP)/Effect TP (RACS) & \multicolumn{2}{c}{98238/162423} & \multicolumn{2}{c}{73233/123116} & \multicolumn{2}{c}{55970/131372} & \multicolumn{2}{c}{47488/\textbf{129817}} \\ \hline
\end{tabular}}
\end{table}

\begin{table}[]
\caption{Memory consumption estimation. The reported memory is the combined consumption of: weight parameters; Adam optimizer states for non-matrix parameters; and candidate optimizer states for matrix parameters. Mem.* represent the consumption when the last layer is trained by Adam. }
\label{tab: memory consumption}
\centering
\resizebox{\textwidth}{!}{
\begin{tabular}{c|cccccccc}
\hline
Methods     & \multicolumn{2}{c}{60M} & \multicolumn{2}{c}{130M} & \multicolumn{2}{c}{350M} & \multicolumn{2}{c}{1.3B} \\ \hline
            & Mem.       & Mem.*      & Mem.        & Mem.*      & Mem.        & Mem.*      & Mem.        & Mem.*      \\
Adam        & NA         & 0.32G      & NA          & 0.75G      & NA          & 2.05G      & NA          & 7.48G      \\ \hline
GaLore      & 0.21G      & 0.26G      & 0.51G       & 0.57G      & 1.2G        & 1.29G      & 4.25G       & 4.43G      \\
Fira        & 0.21G      & 0.26G      & 0.51G       & 0.57G      & 1.2G        & 1.29G      & 4.25G       & 4.43G      \\
Apollo-mini & NA         & 0.23G      & NA          & 0.43G      & NA          & 0.93G      & NA          & 2.98G      \\
Apollo-svd  & NA      & 0.26G      & NA       & 0.57G      &NA       & 1.29G      & NA       & 4.43G      \\ \hline
RACS        & NA         & 0.23G      & NA          & 0.43G      & NA          & 0.93G      & NA          & 2.98G      \\
Alice-0     & 0.21G      & 0.26G      & 0.51G       & 0.57G      & 1.2G        & 1.29G      & 4.25G       & 4.44G      \\
Alice       & 0.22G      & 0.26G      & 0.53G       & 0.59G      & 1.24G       & 1.33G      & 4.42G       & 4.6G       \\ \hline
\end{tabular}}
\end{table}

\begin{table}[]
\centering
\caption{Eval ppl.~of 1B v.s. 7B LLaMA at different training steps. For Apollo, Apollo-mini, 8-bit Adam and Galore, we cite the number from \citet{zhu2024apollo}.}
\label{tab: 7B preformance}
\begin{tabular}{l|l|llll}
\hline
Method           & Mem.   & 40K & 80K & 120K & 150K \\ \hline
8-bit Adam (7B)&26G &18.09 & 15.47 & 14.83 & 14.61  \\
8-bit Galore (7B) &18G & 17.94 & 15.39 &14.95 &14.65\\
Apollo (7B)      & 15.03G &  17.55   &  14.39   &  13.23    &   13.02   \\
Apollo-mini (7B) & 13.53G &  18.03   &  14.60   &  13.32    &  13.09    \\\hline
\gls{ssgd} (1B) & 2.98G & 16.43 &14.26& \textbf{13.20}& \textbf{13.01}\\
Alice (1B)       & 4.6G   &  \textbf{15.93}   &  \textbf{14.08}   & \textbf{13.15}     & \textbf{12.97}     \\
\hline
\end{tabular}
\end{table}

\paragraph{Main results} \Cref{tab: pretrain performance} reports the pretraining performance in terms of evaluation perplexity, and \cref{tab: memory consumption} summarizes the corresponding estimated memory consumption. Our proposed \gls{ssgd} and \gls{alice} outperforms the other memory-efficient baselines and full-rank Adam consistently across various model sizes. \gls{alice} and \gls{alicez} perform on par with each other, suggesting \gls{alicez} may be preferred for better memory efficiency. One major advantage of \gls{alice} is its fast convergence compared to Adam, and achieves more than $2\times$ speed-ups across model sizes, while maintaining similar memory consumption as GaLore. \Cref{fig:1b pretrain curve} demonstrate this fast convergence of the evaluation perplexity during training for the 1B model. 
Despite the simplicity of \gls{ssgd}, it performs well for 350M and 1.3B model, and surpasses the other scaling-based baseline, Apollo-mini and Apollo-svd, consistently. 
From the throughput (TP), we can observe \gls{alice} and \gls{ssgd} are not significantly slower than Adam with $15\%$ and $11\%$ drop for 1B model, respectively. 
Considering the speedup effect, we report the effective TP to represent how quickly the optimizers reach a target loss in wall-clock time. \gls{alice} and \gls{ssgd} achieve $123048$ and $129817$ with 1B model, respectively, compared to $53588$ for Adam, resulting in more than 2× faster convergence in wall-clock time to reach Adam’s final evaluation perplexity.

For Adam, since we cannot reproduce the exact number reported in \citet{zhu2024apollo}. For fair comparison, we also cite the perplexities of Adam from \citet{zhu2024apollo}, and compute the corresponding speed-ups in terms of steps. \gls{alice} achieves $2.22$x, $2.11$x, $2.18$x, and $2.15$x for 60M, 130M, 350M and 1B models, respectively.  
For the actual memory footprints and additional training curves, see \cref{subapp: additional pretrain results}.
\begin{figure}
    \centering
    \includegraphics[scale=0.3]{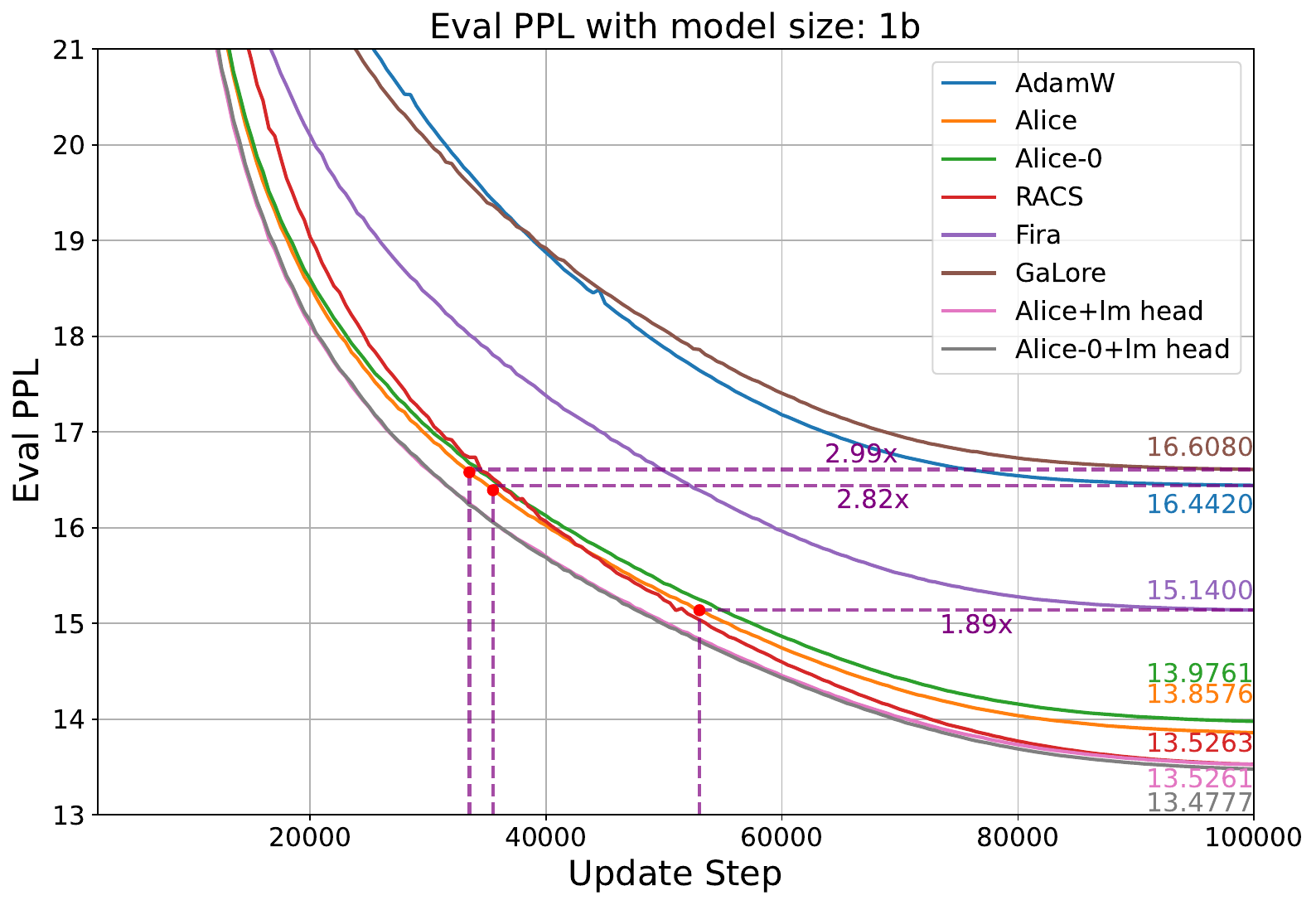}
    \caption{1B LLaMA C4 pretraining evaluation ppl.~curve. "+lm head" represents the last layer is trained by full-rank Adam.}
    \label{fig:1b pretrain curve}
\end{figure}
\paragraph{1B v.s.~7B LLaMA}
To further demonstrate the effectiveness, we follow the setup of \citet{zhu2024apollo}, but train a 1B LLaMA with \gls{alice} and \gls{ssgd} to compare with 7B LLaMA trained by Apollo, Apollo-mini, 8-bit Adam and 8-bit Galore. \Cref{tab: 7B preformance} shows that 1B model trained with our proposed optimizers consistently outperforms 7B model at different training checkpoints with less memory cost\footnote{The total memory cost consists of (1) model parameters; (2) Adam optimizer cost for non-matrix parameters; (3) candidate optimizer cost for matrix parameters. Here, we report the combined cost.}. To complete 150K steps with 8xA100 GPUs, Apollo requires \textbf{15} days while \gls{alice} and \gls{ssgd} require around \textbf{3.8} days.

\subsection{Ablation: Effectiveness of the design choice}
\label{subsec: exp ablation study}

\paragraph{Effect of low-rank tracking}
First, we verify whether low-rank tracking (\cref{eq: tracking step}) is beneficial and whether our conjecture about the stability of the leading basis holds. As shown in \cref{tab: pretrain performance}, \gls{alicez} performs on par with \gls{alice}, suggesting that tracking does not provide a significant boost. However, \cref{fig: effect of tracking} indicates that tracking is helpful when compensation is disabled and must be used alongside switching. Without switching, tracking leads to inferior performance due to the stability of the eigenspace. \Cref{fig: eigenspace cosine similiary} supports this conjecture by showing the cosine similarity before and after updating 
$\mU$ every $200$ steps, confirming that tracking contributes to the stability of the leading basis.

\paragraph{Switching strategy}
We evaluate the effectiveness of our theory-inspired switching strategy compared to other heuristics. The considered alternatives are: (1) \textbf{Gaussian}: $\mU$ is sampled from a Gaussian distribution; (2) \textbf{Gaussian mix}: the leading basis is mixed with vectors sampled from a Gaussian matrix; (3) \textbf{full basis}: instead of sampling the $r-l$ basis in \cref{alg: subspace switching} solely from $m-r$ complement basis, they are sampled jointly from the entire basis excluding the top $l$ leading eigenvectors, i.e.~$[\mU,\mU_c]\backslash {\mU_{:,:l}}$. As shown in \Cref{fig: effect of switch strategy}, our strategy outperforms these alternatives, particularly the Gaussian-based approaches. One possible reason is that the orthogonality of the complement basis ensures a more effective switch, whereas randomly sampled Gaussian vectors may introduce overlaps between switches.

\paragraph{Compensation strategy}
The closest work to our compensation step is Fira \citep{chen2024fira}, which introduces a heuristic-based compensation term. To evaluate the effectiveness of our optimal compensation, we compare it against Fira and a no-compensation baselines by integrating these alternatives into \gls{alice}. Additionally, we introduce Fira+, our proposed modification that further enhances Fira’s performance. 
As shown in \Cref{fig: effect of compensation}, our optimal compensation achieves better performance and convergence than Fira-based compensations, with only a small additional sublinear memory cost (i.e., $n$). Compared to the no-compensation, all strategies yield noticeable performance improvements. \Cref{tab: effectiveness of each component in alice} summarizes the contributions of each component.

\begin{table}[]
\centering
\caption{Effectiveness of each components in \gls{alice} with 130M LLaMA.}
\label{tab: effectiveness of each component in alice}
\begin{tabular}{l|l}
\hline
      Components                       & Evaluation ppl. \\ \hline
No tracking, switch, compen. & 26.96           \\
Tracking                     & 27.35           \\
Tracking+Switch              & 25.11           \\
Tracking+Switch+Compen.      & 21.95           \\ \hline
\end{tabular}
\end{table}

\paragraph{Other ablations}
We also performed additional ablations, examining (1) the effect of the last layer and (2) the effect of \gls{ema} in \gls{ssgd}. For more details, see \cref{subapp: ablation results}.

\section{Conclusion}
\label{sec: conclusion}
In this paper, we take a step toward the systematic design of efficient optimizers for LLMs through structured \gls{fim} approximation. We first establish the connection between structural assumptions and optimizers by solving the structured \gls{fim} approximation. Building on this insight, we propose two design approaches:
(1) Selecting a structure that balances generality and practical efficiency, then solving the \gls{fim} approximation problem accordingly;  
(2) Using a general structure for \gls{fim} approximation, followed by our proposed low-rank framework to improve efficiency.  
Following these principles, we develop two memory-efficient optimizers, \gls{ssgd} and \gls{alice}, each corresponds to one of these design approaches. Experimental validation on LLaMA pre-training demonstrates their effectiveness.  

Our work lays the foundation for a more systematic approach to efficient optimizer design, opening up several promising directions for future research, including: developing low-rank counterparts for SOAP; exploring other possible classes of structures; and investigating approximation problems beyond \gls{fim}.  
By providing a structured perspective on optimizer design, we hope to inspire further advancements in scalable and efficient training methods for LLMs.


\clearpage
\newpage
\bibliography{example_paper}

\begin{thebibliography}{52}
\providecommand{\natexlab}[1]{#1}
\providecommand{\url}[1]{\texttt{#1}}
\expandafter\ifx\csname urlstyle\endcsname\relax
  \providecommand{\doi}[1]{doi: #1}\else
  \providecommand{\doi}{doi: \begingroup \urlstyle{rm}\Url}\fi

\bibitem[Anil et~al.(2020)Anil, Gupta, Koren, Regan, and Singer]{anil2020scalable}
Rohan Anil, Vineet Gupta, Tomer Koren, Kevin Regan, and Yoram Singer.
\newblock Scalable second order optimization for deep learning.
\newblock \emph{arXiv preprint arXiv:2002.09018}, 2020.

\bibitem[Anonymous(2024)]{anonymous2024improving}
Anonymous.
\newblock Improving adaptive moment optimization via preconditioner diagonalization.
\newblock In \emph{Submitted to The Thirteenth International Conference on Learning Representations}, 2024.
\newblock URL \url{https://openreview.net/forum?id=NdNuKMEv9y}.
\newblock under review.

\bibitem[Bergstra and Bengio(2012)]{bergstra2012random}
James Bergstra and Yoshua Bengio.
\newblock Random search for hyper-parameter optimization.
\newblock \emph{Journal of machine learning research}, 13\penalty0 (2), 2012.

\bibitem[Bernstein et~al.(2018)Bernstein, Wang, Azizzadenesheli, and Anandkumar]{bernstein2018signsgd}
Jeremy Bernstein, Yu-Xiang Wang, Kamyar Azizzadenesheli, and Animashree Anandkumar.
\newblock signsgd: Compressed optimisation for non-convex problems.
\newblock In \emph{International Conference on Machine Learning}, pages 560--569. PMLR, 2018.

\bibitem[Chen et~al.(2024{\natexlab{a}})Chen, Feng, Li, Lai, Yue, Yuan, and Wang]{chen2024fira}
Xi~Chen, Kaituo Feng, Changsheng Li, Xunhao Lai, Xiangyu Yue, Ye~Yuan, and Guoren Wang.
\newblock Fira: Can we achieve full-rank training of llms under low-rank constraint?
\newblock \emph{arXiv preprint arXiv:2410.01623}, 2024{\natexlab{a}}.

\bibitem[Chen et~al.(2024{\natexlab{b}})Chen, Liang, Huang, Real, Wang, Pham, Dong, Luong, Hsieh, Lu, et~al.]{chen2024symbolic}
Xiangning Chen, Chen Liang, Da~Huang, Esteban Real, Kaiyuan Wang, Hieu Pham, Xuanyi Dong, Thang Luong, Cho-Jui Hsieh, Yifeng Lu, et~al.
\newblock Symbolic discovery of optimization algorithms.
\newblock \emph{Advances in neural information processing systems}, 36, 2024{\natexlab{b}}.

\bibitem[Cho et~al.(2019)Cho, Choi, Park, Shin, and Choo]{cho2019image}
Wonwoong Cho, Sungha Choi, David~Keetae Park, Inkyu Shin, and Jaegul Choo.
\newblock Image-to-image translation via group-wise deep whitening-and-coloring transformation.
\newblock In \emph{Proceedings of the IEEE/CVF conference on computer vision and pattern recognition}, pages 10639--10647, 2019.

\bibitem[Choi(2019)]{choi2019empirical}
D~Choi.
\newblock On empirical comparisons of optimizers for deep learning.
\newblock \emph{arXiv preprint arXiv:1910.05446}, 2019.

\bibitem[Choi et~al.(2021)Choi, Jung, Yun, Kim, Kim, and Choo]{choi2021robustnet}
Sungha Choi, Sanghun Jung, Huiwon Yun, Joanne~T Kim, Seungryong Kim, and Jaegul Choo.
\newblock Robustnet: Improving domain generalization in urban-scene segmentation via instance selective whitening.
\newblock In \emph{Proceedings of the IEEE/CVF conference on computer vision and pattern recognition}, pages 11580--11590, 2021.

\bibitem[Dubey et~al.(2024)Dubey, Jauhri, Pandey, Kadian, Al{-}Dahle, Letman, Mathur, Schelten, Yang, Fan, Goyal, Hartshorn, Yang, Mitra, Sravankumar, Korenev, Hinsvark, Rao, Zhang, Rodriguez, Gregerson, Spataru, Rozi{\`{e}}re, Biron, Tang, Chern, Caucheteux, Nayak, Bi, Marra, McConnell, Keller, Touret, Wu, Wong, Ferrer, Nikolaidis, Allonsius, Song, Pintz, Livshits, Esiobu, Choudhary, Mahajan, Garcia{-}Olano, Perino, Hupkes, Lakomkin, AlBadawy, Lobanova, Dinan, Smith, Radenovic, Zhang, Synnaeve, Lee, Anderson, Nail, Mialon, Pang, Cucurell, Nguyen, Korevaar, Xu, Touvron, Zarov, Ibarra, Kloumann, Misra, Evtimov, Copet, Lee, Geffert, Vranes, Park, Mahadeokar, Shah, van~der Linde, Billock, Hong, Lee, Fu, Chi, Huang, Liu, Wang, Yu, Bitton, Spisak, Park, Rocca, Johnstun, Saxe, Jia, Alwala, Upasani, Plawiak, Li, Heafield, and Stone]{llama3}
Abhimanyu Dubey, Abhinav Jauhri, Abhinav Pandey, Abhishek Kadian, Ahmad Al{-}Dahle, Aiesha Letman, Akhil Mathur, Alan Schelten, Amy Yang, Angela Fan, Anirudh Goyal, Anthony Hartshorn, Aobo Yang, Archi Mitra, Archie Sravankumar, Artem Korenev, Arthur Hinsvark, Arun Rao, Aston Zhang, Aur{\'{e}}lien Rodriguez, Austen Gregerson, Ava Spataru, Baptiste Rozi{\`{e}}re, Bethany Biron, Binh Tang, Bobbie Chern, Charlotte Caucheteux, Chaya Nayak, Chloe Bi, Chris Marra, Chris McConnell, Christian Keller, Christophe Touret, Chunyang Wu, Corinne Wong, Cristian~Canton Ferrer, Cyrus Nikolaidis, Damien Allonsius, Daniel Song, Danielle Pintz, Danny Livshits, David Esiobu, Dhruv Choudhary, Dhruv Mahajan, Diego Garcia{-}Olano, Diego Perino, Dieuwke Hupkes, Egor Lakomkin, Ehab AlBadawy, Elina Lobanova, Emily Dinan, Eric~Michael Smith, Filip Radenovic, Frank Zhang, Gabriel Synnaeve, Gabrielle Lee, Georgia~Lewis Anderson, Graeme Nail, Gr{\'{e}}goire Mialon, Guan Pang, Guillem Cucurell, Hailey Nguyen, Hannah Korevaar, Hu~Xu, Hugo
  Touvron, Iliyan Zarov, Imanol~Arrieta Ibarra, Isabel~M. Kloumann, Ishan Misra, Ivan Evtimov, Jade Copet, Jaewon Lee, Jan Geffert, Jana Vranes, Jason Park, Jay Mahadeokar, Jeet Shah, Jelmer van~der Linde, Jennifer Billock, Jenny Hong, Jenya Lee, Jeremy Fu, Jianfeng Chi, Jianyu Huang, Jiawen Liu, Jie Wang, Jiecao Yu, Joanna Bitton, Joe Spisak, Jongsoo Park, Joseph Rocca, Joshua Johnstun, Joshua Saxe, Junteng Jia, Kalyan~Vasuden Alwala, Kartikeya Upasani, Kate Plawiak, Ke~Li, Kenneth Heafield, and Kevin Stone.
\newblock The llama 3 herd of models.
\newblock \emph{CoRR}, abs/2407.21783, 2024.

\bibitem[Duchi et~al.(2011)Duchi, Hazan, and Singer]{duchi2011adaptive}
John Duchi, Elad Hazan, and Yoram Singer.
\newblock Adaptive subgradient methods for online learning and stochastic optimization.
\newblock \emph{Journal of machine learning research}, 12\penalty0 (7), 2011.

\bibitem[Duvvuri et~al.()Duvvuri, Devvrit, Anil, Hsieh, and Dhillon]{duvvuricombining}
Sai~Surya Duvvuri, Fnu Devvrit, Rohan Anil, Cho-Jui Hsieh, and Inderjit~S Dhillon.
\newblock Combining axes preconditioners through kronecker approximation for deep learning.
\newblock In \emph{The Twelfth International Conference on Learning Representations}.

\bibitem[Gao et~al.(2021)Gao, Liu, Huang, Wang, Wang, Xu, and Yu]{gao2021trace}
Kaixin Gao, Xiaolei Liu, Zhenghai Huang, Min Wang, Zidong Wang, Dachuan Xu, and Fan Yu.
\newblock A trace-restricted kronecker-factored approximation to natural gradient.
\newblock In \emph{Proceedings of the AAAI Conference on Artificial Intelligence}, volume~35, pages 7519--7527, 2021.

\bibitem[George et~al.(2018)George, Laurent, Bouthillier, Ballas, and Vincent]{george2018fast}
Thomas George, C{\'e}sar Laurent, Xavier Bouthillier, Nicolas Ballas, and Pascal Vincent.
\newblock Fast approximate natural gradient descent in a kronecker factored eigenbasis.
\newblock \emph{Advances in Neural Information Processing Systems}, 31, 2018.

\bibitem[Grosse and Martens(2016)]{grosse2016kronecker}
Roger Grosse and James Martens.
\newblock A kronecker-factored approximate fisher matrix for convolution layers.
\newblock In \emph{International Conference on Machine Learning}, pages 573--582. PMLR, 2016.

\bibitem[Gupta et~al.(2018)Gupta, Koren, and Singer]{gupta2018shampoo}
Vineet Gupta, Tomer Koren, and Yoram Singer.
\newblock Shampoo: Preconditioned stochastic tensor optimization.
\newblock In \emph{International Conference on Machine Learning}, pages 1842--1850. PMLR, 2018.

\bibitem[Hu et~al.(2021)Hu, Shen, Wallis, Allen-Zhu, Li, Wang, Wang, and Chen]{hu2021lora}
Edward~J Hu, Yelong Shen, Phillip Wallis, Zeyuan Allen-Zhu, Yuanzhi Li, Shean Wang, Lu~Wang, and Weizhu Chen.
\newblock Lora: Low-rank adaptation of large language models.
\newblock \emph{arXiv preprint arXiv:2106.09685}, 2021.

\bibitem[Huang et~al.(2019)Huang, Zhou, Zhu, Liu, and Shao]{huang2019iterative}
Lei Huang, Yi~Zhou, Fan Zhu, Li~Liu, and Ling Shao.
\newblock Iterative normalization: Beyond standardization towards efficient whitening.
\newblock In \emph{Proceedings of the IEEE/CVF conference on computer vision and pattern recognition}, pages 4874--4883, 2019.

\bibitem[Hwang(2024)]{hwang2024fadam}
Dongseong Hwang.
\newblock Fadam: Adam is a natural gradient optimizer using diagonal empirical fisher information.
\newblock \emph{arXiv preprint arXiv:2405.12807}, 2024.

\bibitem[Jordan et~al.(2024)Jordan, Jin, Boza, You, Cecista, Newhouse, and Bernstein]{jordan2024muon}
Keller Jordan, Yuchen Jin, Vlado Boza, Jiacheng You, Franz Cecista, Laker Newhouse, and Jeremy Bernstein.
\newblock Muon: An optimizer for hidden layers in neural networks, 2024.
\newblock URL \url{https://kellerjordan.github.io/posts/muon/}.

\bibitem[Kingma(2014)]{kingma2014adam}
Diederik~P Kingma.
\newblock Adam: A method for stochastic optimization.
\newblock \emph{arXiv preprint arXiv:1412.6980}, 2014.

\bibitem[Koroko et~al.(2022)Koroko, Anciaux-Sedrakian, Gharbia, Gar{\`e}s, Haddou, and Tran]{koroko2022efficient}
Abdoulaye Koroko, Ani Anciaux-Sedrakian, Ibtihel~Ben Gharbia, Val{\'e}rie Gar{\`e}s, Mounir Haddou, and Quang~Huy Tran.
\newblock Efficient approximations of the fisher matrix in neural networks using kronecker product singular value decomposition.
\newblock \emph{arXiv preprint arXiv:2201.10285}, 2022.

\bibitem[Korthikanti et~al.(2023)Korthikanti, Casper, Lym, McAfee, Andersch, Shoeybi, and Catanzaro]{korthikanti2023reducing}
Vijay~Anand Korthikanti, Jared Casper, Sangkug Lym, Lawrence McAfee, Michael Andersch, Mohammad Shoeybi, and Bryan Catanzaro.
\newblock Reducing activation recomputation in large transformer models.
\newblock \emph{Proceedings of Machine Learning and Systems}, 5:\penalty0 341--353, 2023.

\bibitem[Li et~al.(2018)Li, Xie, Wang, and Gao]{li2018towards}
Peihua Li, Jiangtao Xie, Qilong Wang, and Zilin Gao.
\newblock Towards faster training of global covariance pooling networks by iterative matrix square root normalization.
\newblock In \emph{Proceedings of the IEEE conference on computer vision and pattern recognition}, pages 947--955, 2018.

\bibitem[Li(2017)]{li2017preconditioned}
Xi-Lin Li.
\newblock Preconditioned stochastic gradient descent.
\newblock \emph{IEEE transactions on neural networks and learning systems}, 29\penalty0 (5):\penalty0 1454--1466, 2017.

\bibitem[Li(2018)]{li2018preconditioner}
Xi-Lin Li.
\newblock Preconditioner on matrix lie group for sgd.
\newblock \emph{arXiv preprint arXiv:1809.10232}, 2018.

\bibitem[Li et~al.(2017)Li, Fang, Yang, Wang, Lu, and Yang]{li2017universal}
Yijun Li, Chen Fang, Jimei Yang, Zhaowen Wang, Xin Lu, and Ming-Hsuan Yang.
\newblock Universal style transfer via feature transforms.
\newblock \emph{Advances in neural information processing systems}, 30, 2017.

\bibitem[Lialin et~al.(2023)Lialin, Muckatira, Shivagunde, and Rumshisky]{lialin2023relora}
Vladislav Lialin, Sherin Muckatira, Namrata Shivagunde, and Anna Rumshisky.
\newblock Relora: High-rank training through low-rank updates.
\newblock In \emph{The Twelfth International Conference on Learning Representations}, 2023.

\bibitem[Lin et~al.(2024)Lin, Dangel, Eschenhagen, Bae, Turner, and Makhzani]{lin2024can}
Wu~Lin, Felix Dangel, Runa Eschenhagen, Juhan Bae, Richard~E Turner, and Alireza Makhzani.
\newblock Can we remove the square-root in adaptive gradient methods? a second-order perspective.
\newblock \emph{arXiv preprint arXiv:2402.03496}, 2024.

\bibitem[Loshchilov and Hutter(2016)]{loshchilov2016sgdr}
Ilya Loshchilov and Frank Hutter.
\newblock Sgdr: Stochastic gradient descent with warm restarts.
\newblock \emph{arXiv preprint arXiv:1608.03983}, 2016.

\bibitem[Ma et~al.(2024)Ma, Gong, Scetbon, and Meeds]{ma2024swan}
Chao Ma, Wenbo Gong, Meyer Scetbon, and Edward Meeds.
\newblock Swan: Preprocessing sgd enables adam-level performance on llm training with significant memory reduction.
\newblock \emph{arXiv preprint arXiv:2412.13148}, 2024.

\bibitem[Martens(2020)]{martens2020new}
James Martens.
\newblock New insights and perspectives on the natural gradient method.
\newblock \emph{Journal of Machine Learning Research}, 21\penalty0 (146):\penalty0 1--76, 2020.

\bibitem[Martens and Grosse(2015)]{martens2015optimizing}
James Martens and Roger Grosse.
\newblock Optimizing neural networks with kronecker-factored approximate curvature.
\newblock In \emph{International conference on machine learning}, pages 2408--2417. PMLR, 2015.

\bibitem[Martens et~al.(2018)Martens, Ba, and Johnson]{martens2018kronecker}
James Martens, Jimmy Ba, and Matt Johnson.
\newblock Kronecker-factored curvature approximations for recurrent neural networks.
\newblock In \emph{International Conference on Learning Representations}, 2018.

\bibitem[Morwani et~al.(2024)Morwani, Shapira, Vyas, Malach, Kakade, and Janson]{morwani2024new}
Depen Morwani, Itai Shapira, Nikhil Vyas, Eran Malach, Sham Kakade, and Lucas Janson.
\newblock A new perspective on shampoo's preconditioner.
\newblock \emph{arXiv preprint arXiv:2406.17748}, 2024.

\bibitem[Pooladzandi and Li(2024)]{pooladzandi2024curvature}
Omead Pooladzandi and Xi-Lin Li.
\newblock Curvature-informed sgd via general purpose lie-group preconditioners.
\newblock \emph{arXiv preprint arXiv:2402.04553}, 2024.

\bibitem[Raffel et~al.(2020)Raffel, Shazeer, Roberts, Lee, Narang, Matena, Zhou, Li, and Liu]{raffel2020C4}
Colin Raffel, Noam Shazeer, Adam Roberts, Katherine Lee, Sharan Narang, Michael Matena, Yanqi Zhou, Wei Li, and Peter~J Liu.
\newblock Exploring the limits of transfer learning with a unified text-to-text transformer.
\newblock \emph{Journal of machine learning research}, 21\penalty0 (140):\penalty0 1--67, 2020.

\bibitem[Shazeer and Stern(2018)]{shazeer2018adafactor}
Noam Shazeer and Mitchell Stern.
\newblock Adafactor: Adaptive learning rates with sublinear memory cost.
\newblock In \emph{International Conference on Machine Learning}, pages 4596--4604. PMLR, 2018.

\bibitem[Si et~al.(2024)Si, Wang, Yang, Xu, Li, Dai, Qiao, Yang, and Shen]{si2024flora}
Chongjie Si, Xuehui Wang, Xue Yang, Zhengqin Xu, Qingyun Li, Jifeng Dai, Yu~Qiao, Xiaokang Yang, and Wei Shen.
\newblock Flora: Low-rank core space for n-dimension.
\newblock \emph{arXiv preprint arXiv:2405.14739}, 2024.

\bibitem[Sun and Spall(2021)]{sun2021connection}
Shiqing Sun and James~C Spall.
\newblock Connection of diagonal hessian estimates to natural gradients in stochastic optimization.
\newblock In \emph{2021 55th Annual Conference on Information Sciences and Systems (CISS)}, pages 1--6. IEEE, 2021.

\bibitem[Tian et~al.(2020)Tian, Yu, Chen, and Ganguli]{tian2020understanding}
Yuandong Tian, Lantao Yu, Xinlei Chen, and Surya Ganguli.
\newblock Understanding self-supervised learning with dual deep networks.
\newblock \emph{arXiv preprint arXiv:2010.00578}, 2020.

\bibitem[Touvron et~al.(2023)Touvron, Lavril, Izacard, Martinet, Lachaux, Lacroix, Rozi{\`e}re, Goyal, Hambro, Azhar, et~al.]{touvron2023llama}
Hugo Touvron, Thibaut Lavril, Gautier Izacard, Xavier Martinet, Marie-Anne Lachaux, Timoth{\'e}e Lacroix, Baptiste Rozi{\`e}re, Naman Goyal, Eric Hambro, Faisal Azhar, et~al.
\newblock Llama: Open and efficient foundation language models.
\newblock \emph{arXiv preprint arXiv:2302.13971}, 2023.

\bibitem[Vyas et~al.(2024)Vyas, Morwani, Zhao, Shapira, Brandfonbrener, Janson, and Kakade]{vyas2024soap}
Nikhil Vyas, Depen Morwani, Rosie Zhao, Itai Shapira, David Brandfonbrener, Lucas Janson, and Sham Kakade.
\newblock Soap: Improving and stabilizing shampoo using adam.
\newblock \emph{arXiv preprint arXiv:2409.11321}, 2024.

\bibitem[Xu et~al.(2024)Xu, Xiang, Cai, and Wen]{xu2024adamlearningratescaling}
Minghao Xu, Lichuan Xiang, Xu~Cai, and Hongkai Wen.
\newblock No more adam: Learning rate scaling at initialization is all you need, 2024.
\newblock URL \url{https://arxiv.org/abs/2412.11768}.

\bibitem[Yang and Laaksonen(2008)]{yang2008principal}
Zhirong Yang and Jorma Laaksonen.
\newblock Principal whitened gradient for information geometry.
\newblock \emph{Neural Networks}, 21\penalty0 (2-3):\penalty0 232--240, 2008.

\bibitem[You et~al.(2017)You, Gitman, and Ginsburg]{you2017lars}
Yang You, Igor Gitman, and Boris Ginsburg.
\newblock Large batch training of convolutional networks.
\newblock \emph{arXiv preprint arXiv:1708.03888}, 2017.

\bibitem[You et~al.(2019)You, Li, Reddi, Hseu, Kumar, Bhojanapalli, Song, Demmel, Keutzer, and Hsieh]{you2019lamb}
Yang You, Jing Li, Sashank Reddi, Jonathan Hseu, Sanjiv Kumar, Srinadh Bhojanapalli, Xiaodan Song, James Demmel, Kurt Keutzer, and Cho-Jui Hsieh.
\newblock Large batch optimization for deep learning: Training bert in 76 minutes.
\newblock \emph{arXiv preprint arXiv:1904.00962}, 2019.

\bibitem[Zhang et~al.(2024)Zhang, Chen, Li, Ding, Wu, Ye, Luo, and Sun]{zhang2024adam}
Yushun Zhang, Congliang Chen, Ziniu Li, Tian Ding, Chenwei Wu, Yinyu Ye, Zhi-Quan Luo, and Ruoyu Sun.
\newblock Adam-mini: Use fewer learning rates to gain more.
\newblock \emph{arXiv preprint arXiv:2406.16793}, 2024.

\bibitem[Zhao et~al.(2024{\natexlab{a}})Zhao, Zhang, Chen, Wang, Anandkumar, and Tian]{zhao2024galore}
Jiawei Zhao, Zhenyu Zhang, Beidi Chen, Zhangyang Wang, Anima Anandkumar, and Yuandong Tian.
\newblock Galore: Memory-efficient llm training by gradient low-rank projection.
\newblock \emph{arXiv preprint arXiv:2403.03507}, 2024{\natexlab{a}}.

\bibitem[Zhao et~al.(2024{\natexlab{b}})Zhao, Li, Gu, Zheng, K{\"o}lker, Wang, and Yuan]{zhao2024adapprox}
Pengxiang Zhao, Ping Li, Yingjie Gu, Yi~Zheng, Stephan~Ludger K{\"o}lker, Zhefeng Wang, and Xiaoming Yuan.
\newblock Adapprox: Adaptive approximation in adam optimization via randomized low-rank matrices.
\newblock \emph{arXiv preprint arXiv:2403.14958}, 2024{\natexlab{b}}.

\bibitem[Zhao et~al.(2024{\natexlab{c}})Zhao, Morwani, Brandfonbrener, Vyas, and Kakade]{zhao2024deconstructing}
Rosie Zhao, Depen Morwani, David Brandfonbrener, Nikhil Vyas, and Sham Kakade.
\newblock Deconstructing what makes a good optimizer for language models.
\newblock \emph{arXiv preprint arXiv:2407.07972}, 2024{\natexlab{c}}.

\bibitem[Zhu et~al.(2024)Zhu, Zhang, Cong, Liu, Park, Chandra, Long, Pan, Wang, and Lee]{zhu2024apollo}
Hanqing Zhu, Zhenyu Zhang, Wenyan Cong, Xi~Liu, Sem Park, Vikas Chandra, Bo~Long, David~Z Pan, Zhangyang Wang, and Jinwon Lee.
\newblock Apollo: Sgd-like memory, adamw-level performance.
\newblock \emph{arXiv preprint arXiv:2412.05270}, 2024.

\end{thebibliography}
\bibliographystyle{plainnat}

\newpage

\appendix
\onecolumn

\printglossaries

\section{Examples of diagonal operations}
\label{app: example diagonal operations}
In this section, we will give some detailed examples on each of the diagonal operations we introduced in \cref{sec: preliminary}. 
\paragraph{Example of $\diag(\cdot)$}
This will extract the diagonals of a input matrix into a vector:
\begin{align*}
    \mM = \left[\begin{array}{ccc}
       a_{11}  & a_{12}&a_{13}  \\
        a_{21} & a_{22} & a_{23}\\
        a_{31}&a_{32} & a_{33}
    \end{array}\right],\;\;\; \diag(\mM) = [a_{11}, a_{22}, a_{33}]^T
\end{align*}

\paragraph{Example of $\diagb(\cdot)$}
This simply stack the input matrices sequence into a larger block diagonal matrix:
\begin{align*}
    \diagb(\mM_1,\mM_2,\mM_3) = \left[\begin{array}{ccc}
        \mM_1 & \bm{0} & \bm{0}  \\
        \bm{0} & \mM_2 & \bm{0}\\
        \bm{0} & \bm{0} & \mM_3
    \end{array}\right]
\end{align*}

\paragraph{Example of $\diagv(\cdot)$}
This will stack the vector element into a pure diagonal matrix:
\begin{align*}
    \diagv([a_{11},a_{22}, a_{33}]^T) = \left[\begin{array}{ccc}
        a_{11} & 0 & 0  \\
         0&a_{22}&0\\
         0&0&a_{33}
    \end{array}\right]
\end{align*}

\paragraph{Example of $\diagm(\cdot)$}
This will stack the elements in input matrix to form a larger pure diagonal matrix in a column-wise manner.
\begin{align*}
    \diagm\left(\left[\begin{array}{cc}
        a_{11} & a_{12}  \\
        a_{21} &  a_{22}
    \end{array}\right]\right) = \left[\begin{array}{cccc}
        a_{11} & 0 & 0 &0  \\
        0 & a_{21} & 0 & 0\\
        0&0&a_{12} & 0\\
        0&0&0&a_{22}
    \end{array}\right]
\end{align*}

\section{Background}
\label{app: background}
In this section, we will provide a more comprehensive backgrounds.

\subsection{Fisher information}
\label{subapp: background fisher information}
Fisher information can also be viewed as the "sharpness" of the likelihood around the true parameters from the maximum likelihood view point. 
Formally, under the context of LLMs $f_\theta(\cdot)$, we consider a dataset $\{\vx_i\}_{i=1}^N$, where $N$ is total batched sentences, and $\vx_i$ is the token sequence of $i^{\text{th}}$ sentence. One can define sentence-level auto-regressive loss function $\mathcal{L} = \sum_{j=1}c(x_{i,j+1}, f_{\theta}(\vx_{i,:j})$, where $j$ is the token index and $c$ is the user-defined loss metric. The corresponding likelihood can be defined as $p_\theta(\vx_i)\propto \exp(-\sum_{j=1}c(x_{i,j+1},f_{\theta}(\vx_{i,:j})))$. 
The standard empirical \gls{fim} is defined as 
\begin{align*}
    \mF = \sum_{i=1}^N\nabla_\theta \log p_\theta(\vx_i)\nabla_\theta^T\log p_\theta(\vx_i)
\end{align*}
In practice, we often use the mini-batched gradient $\vecg = \frac{1}{N_B}\sum_{i=1}^{N_B} \log p_\theta(\vx_i)$ with batch size $N_B$ for empirical \gls{fim} during training. More discussion can be found in \cite{lin2024can}.
Throughout this paper, we adopt the notation $\E[\vecg\vecg^T]$ as $\mF$. 

One standard application of \Gls{fim} is efficient optimization. \Gls{ngd} leverage the inverse of \gls{fim} to smooth the local information geometry, leading to the steepest descent in the probability space with KL divergence metric. This typically leads to faster convergences compared to its parameter-space counterpart; and more stable optimization \citep{martens2020new}. The update of \gls{ngd} with step size $\lambda$ is
\begin{align*}
    \theta\leftarrow \theta-\lambda \mF^{-1}\nabla_{\theta}\mathcal{L}.
\end{align*}
However, square-root inverse is sometimes more favorable than inverse, and has been shown that it provides a better approximation to the geodesic flow, compared with the default natural gradient update \cite{yang2008principal}. Empirically, it demonstrates stronger performance and desired properties when using non-constant learning rate \citep{lin2024can, loshchilov2016sgdr, bergstra2012random,choi2019empirical}. 
The corresponding update is to simply replace $\mF^{-1}$ with $\mF^{-\frac{1}{2}}$:
\begin{align}
    \theta\leftarrow \theta-\lambda \mF^{-\frac{1}{2}}\nabla_{\theta}\mathcal{L}.
\end{align}
In this paper, we will use the square-root inverse view point, but our analysis is agnostic to it and can be easily extended to the direct inverse. However, matrix multiplication involving \gls{fim} and its inverse are computationally expensive for large models since $\mF\in\R^{mn\times mn}$ for vectorized parameters $\theta\in\R^{mn}$. Next, we will briefly introduce Kronecker product and block diagonals, which are two main classes of structural assumptions used in this paper. Their nice properties can significantly reduce the assocaited computation burden.

\subsection{Kronecker product and block diagonals}
\label{subapp: kronecker product and block diagonals}
Kronecker product, denoted by $\otimes$, is to combine two arbitrary matrices into a larger matrix with block-wise patterns. It has emerged as a powerful tool for approximating higher-order information like Hessian \citep{grosse2016kronecker} or \gls{fim}. 
The main advantages are their nice properties regarding matrix operations, leading to efficient practical procedures when dealing with large matrix. We will mainly use two properties:
\begin{align}
    (\mA\otimes \mB)^{-\frac{1}{2}} &= \mA^{-\frac{1}{2}}\otimes \mB^{-\frac{1}{2}} \label{eq: Kronecker product root inverse}\\
    (\mA\otimes \mB)\vect(\mC) &= \vect(\mB\mC\mA^T)\label{eq: Kronecker product vector matrix conversion}
\end{align}
where the first one holds when $\mA$,$\mB$ are square-root invertible. The second one is particular useful to reduce the computation associated with matrix-vector multiplication in \cref{eq: square root ngd}. For $\mA\otimes\mB \in \R^{mn\times mn}$, it reduces the computation from $O(m^2n^2)$ to $O(mn^2+m^2n)$ with $\mA\in\Rnn$, $\mB\in\Rmm$, and $\mC\in\Rmn$. 

For block diagonal matrix, one can also easily compute its square-root inverse:
\begin{align}
    \diagb(\mM_1,\ldots,\mM_n)^{-\frac{1}{2}} = \diagb(\mM_1^{-\frac{1}{2}},\ldots,\mM_n^{-\frac{1}{2}})
    \label{eq: block diagonal square root inverse}
\end{align}
where each $\mM_i$ is square-root invertible. When each $\mM_i$ is also a diagonal matrix with positive values, we have the following:
\begin{equation}
    \diagb(\mM_1,\ldots,\mM_n)^{-\frac{1}{2}}\vect(\mC) = \vect\left(\frac{\mC}{\sqrt{\devect(\vv)}}\right)
    \label{eq: block diagonal elementwise division}
\end{equation}
where $\vv$ is the vector containing the diagonals of $\diagb(\mM_1,\ldots,\mM_n)$, transforming matrix vector product into element-wise operation.  

\subsection{Operators for gradient: normalization and whitening}
\label{subapp: normalization and whitening}
Recently, there are some optimizers \citep{ma2024swan, jordan2024muon, you2017lars, you2019lamb} that apply operators to pre-process the gradient and use it in standard SGD. Empirical evidence has verified their effectiveness in training LLMs. In particular, there are two well-known operators: normalization and whitening, where the above optimizers relies on one or both of them. In particular,
\begin{align}
    \normalize(\mG) =& \frac{\mG}{\bm{1}\sqrt{\vs^T}}= \mG\mS^{-\frac{1}{2}}\\
    \whiten(\mG) =& (\mG\mG^T)^{-\frac{1}{2}}\mG.
\end{align}
where $\vs\in\R^n$, $\vs_i= \sum_{j=1}^m \mG^2_{ij}$ with $\mG\in\Rmn$, and $\mS=\diagv(\vs)$. Namely, vector $\vs$ contains the squared column norm of $\mG$, and $\mS$ is a scaling matrix to normalize the columns. Normalizing the rows can also be written in a similar format. 
$\whiten$ operator essentially orthogonalizes $\mG$, that compute the closest orthogonal matrix to $\mG$ under Frobenius norm. In practice, the whitening operator can be iteratively solved using Newton-Schulz algorithm without explicitly computing the square-root inverse. 

\subsection{Shampoo Optimizer}
\label{subapp: Shampoo optimizer}
Shampoo \citep{gupta2018shampoo} was originally proposed as the second-order optimization technique over the tensor space. Under the context of transformers, typically matrix parameters are considered. Its core design principle also aims to approximate the \gls{fim} with structural assumptions $\Ft_t=\mR_{n,t}^{\frac{1}{2}}\otimes \mL_{m,t}^{\frac{1}{2}}$, with $\mR_{n,t} = \mR_{n,t-1}+ \mG_t^T\mG_t$ and $\mL_{m,t} = \mL_{m,t-1}+ \mG_t\mG_t^T$. The above update rule can be seen as the moving average estimate of $\E[\mG_t^T\mG_t]$ and $\E[\mG_t\mG_t^T]$.
However, the Shampoo paper \citep{gupta2018shampoo} does not explicitly show why these design choices of $\mR_n$, $\mL_m$ approximate the \gls{fim}. In fact, they only show $\mR_n^{\frac{1}{2}}\otimes \mL_m^{\frac{1}{2}}$ forms an upper bound of \gls{fim}\footnote{Lemma 8 in \citep{gupta2018shampoo}. Note that our paper assumes $\vect(\cdot)$ is stacking columns of matrix whereas \cite{gupta2018shampoo} assumes stacking the rows, explaining the reverse order of presentation}. This is not helpful in understanding whether this approximates the \gls{fim} or not. Another very recent follow-up work \citep{morwani2024new} provides an explanation of the shampoo preconditioner in terms of approximating \gls{fim}. They show the square of Shampoo's preconditioner is equivalent to a single step of power iteration of computing optimal Kronecker product approximation to \gls{fim}. It indirectly establishes the connections to \gls{fim} approximation since the approximation is expressed as an iterative algorithm. To the best of our knowledge, our result is the first to directly establish the connection of Shampoo to \gls{fim} approximation as minimizing a upper bound of the loss with the Frobenius norm (\cref{thm: optimal shampoo}). 

Nevertheless, the original Shampoo algorithm is summarized in \cref{alg: Shampoo optimizer}. 

\begin{algorithm}
    \caption{Shampoo Optimizer}
    \label{alg: Shampoo optimizer}
    \begin{algorithmic}
        \STATE {\bfseries Input:} $\mL_m=\epsilon\mI_m$, $\mR_n=\epsilon\mI_n$, learning rate $\lambda$, optimization step $T$, loss function $\mathcal{L}$.
        \FOR{$t=1,\ldots, T$}
            \STATE $\mG_t=\nabla_{\mW_t}\mathcal{L}$
            \STATE $\mL_{m,t}=\mL_{m,t-1}+\mG_t\mG_t^T$
            \STATE $\mR_{n,t}=\mR_{n,t-1}+\mG_t^T\mG_t$
            \STATE $\mW_{t}=\mW_{t-1}+\lambda \mL_{m,t}^{-\frac{1}{4}}\mG_t\mR_{n,t}^{-\frac{1}{4}}$
        \ENDFOR
    \end{algorithmic}
\end{algorithm}

\subsection{SOAP/AdaDiag++}
\label{subapp: SOAP}
SOAP/AdaDiag++ \citep{vyas2024soap, anonymous2024improving} is a recently proposed adaptive optimizer aiming to improve the practical convergence and stability of Shampoo. Their main intuition behind (from the view point of \cite{vyas2024soap}) is that they show Shampoo is equivalent to performing Adafactor \citep{shazeer2018adafactor} under the Shampoo's eigen-space. Namely, the eigen-matrix of $\mL_{m,t}$ and $\mR_{n,t}$ in \cref{alg: Shampoo optimizer}. Since Adafactor is an approximation to Adam, they propose to use Adam instead of Adafactor in Shampoo's eigen-space, to further improve the performance. They propose the following update rule: 
\begin{align}
    &\vm_{t} = \beta_1\vm_{t-1}+(1-\beta_1)\mG_t\;\;\; (\text{first moment})\nonumber\\
    &\mL_{m,t} = \beta_3\mL_{m,t-1}+(1-\beta_3)\mG_t\mG_t^T\;\;\;(\text{Shampoo's }\mL_{m,t})\nonumber\\
    & \mR_{n,t} = \beta_3\mR_{n,t-1} + (1-\beta_3) \mG_t^T\mG_t\;\;\;(\text{Shampoo's }\mR_{n,t})\nonumber\\
    & \mU_{L,t} = \eig(\mL_{m,t})\nonumber \;\;\;(\text{Shampoo's left eigen-space})\\
    & \mU_{R,t} = \eig(\mR_{n,t})\nonumber \;\;\; (\text{Shampoo's right eigen-space})\\
    & \vv_t = \beta_2\vv_{t-1}+(1-\beta_2)(\mU_{L,t}^T\mG_t\mU_{R,t})\elesquare\;\;\;(\text{second moment})\nonumber\\
    &\Delta = \mU_{L,t}\frac{\mU_{L,t}^T\vm_{t}\mU_{R,t}}{\sqrt{\vv_t}}\mU_{R,t}^T
\label{eq: practical soap updates}    
\end{align}
These update rules exactly describes the procedure to applying Adam updates in the "rotated" space defined by $\mU_{L,t}$ and $\mU_{R,t}$. Due to the computational burden associated with $\eig$, SOAP proposed to only update $\mU_{L,t}$, $\mU_{R,t}$ at certain intervals. This leads to the following algorithm:
\begin{algorithm}
    \caption{SOAP optimizer}
    \label{alg: SOAP optimizer}
    \begin{algorithmic}
        \STATE {\bfseries Input:} learning rate $\lambda$, update interval $K$, $\beta_1$, $\beta_2$, $\beta_3$, optimization step $T$.
        \STATE $\vm_0=0$; $\vv_0=0$ \COMMENT{Initialize two moments}
        \STATE $\mL_{m,0}=0$; $\mR_{n,0}=0$ \COMMENT{Initialize two \gls{ema} states for $\mG\mG^T$, $\mG^T\mG$}
        \FOR{$t=1,\ldots, T$}
            \STATE $\mG_t=\nabla_{\mW_t}\mathcal{L}$
            \STATE $\vm_t = \beta_1\vm_{t-1}+(1-\beta_1)\mG_t$

            \STATE $\mL_{m,t}=\beta_3\mL_{m,t-1}+(1-\beta_3)\mG_t\mG_t^T$ \COMMENT{Accumulation for $\mG\mG^T$}
            \STATE $\mR_{n,t}=\beta_3\mR_{n,t-1}+(1-\beta_3)\mG_t^T\mG_t$ \COMMENT{Accumulation for $\mG^T\mG$}
            \IF{$t==1$ or $(t\mod K)==0$}
                \STATE $\mU_{R,t}=\eig(\mR_{n,t})$ \COMMENT{Get right eigen-space $\mU_R$}
                \STATE $\mU_{L,t}=\eig(\mL_{m,t})$ \COMMENT{Get left eigen-space $\mU_L$}
            \ELSE
                \STATE $\mU_{R,t}=\mU_{R,t-1}$
                \STATE $\mU_{L,t}= \mU_{L,t-1}$
            \ENDIF
            \STATE $\widetilde{\vm}_t = \mU_{L,t}^T\vm_t\mU_{R,t}$ \COMMENT{Get rotated 1st moment}
            \STATE $\vv_t = \beta_2\vv_{t-1}+(1-\beta_2)(\mU_{L,t}^T\mG_t\mU_{R,t})\elesquare$ \COMMENT{Compute second moments}
            \STATE $\mW_{t+1}=\mW_{t} -\lambda \mU_{L,t}\frac{\widetilde{\vm}_t}{\sqrt{\vv_t}}\mU_{R,t}^T$
        \ENDFOR
    \end{algorithmic}
\end{algorithm}

\subsection{AdaDiag and one-side SOAP}
\label{subapp: AdaDiag}
AdaDiag++ \citep{anonymous2024improving}, a concurrent work to SOAP, independently develops the equivalent update rules as SOAP. The only difference is that they disable the \gls{ema} tracking for $\mL_{m,t}$ and $\mR_{n,t}$. The resulting optimizer is both computational and memory expensive due to the storage of $\mU_{R}$, $\mU_L$ and two eigenvalue decompositions. To address this issue, they both propose a one-side version called AdaDiag and one-side SOAP by only considering either left or right eigen-space. The resulting update rule is exactly the same as our proposed \gls{alicec} (i.e.~\cref{eq: practical alicec equations}). 

However, they propose this design choice purely based on intuition to reduce computation and memory consumption, and do not explicitly reveal the connections to their two-sided version. Thus, it lacks the understanding on why two-sided version obtains empirically better performance. 
Based on \cref{subsec: alicec} and \cref{subsec: new opt combination},  we show that although one-sided version has similar updates as the two-sided twins, they are different optimizers with distinct underlying structural assumptions. In fact, the structures of SOAP/AdaDiag++ strictly generalizes their one-sided version, explaining the better empirical performance. 
The resulting algorithm is the following:
\begin{algorithm}
    \caption{\gls{alicec}/AdaDiag/one-side SOAP optimizer}
    \label{alg: alicec optimizer}
    \begin{algorithmic}
        \STATE {\bfseries Input:} learning rate $\lambda$, update interval $K$, $\beta_1$, $\beta_2$, $\beta_3$, optimization step $T$.
        \STATE $\vm_0 = 0$; $\vv_0=0$ \COMMENT{Initialize two moments}
        \STATE $\mQ_0=0$ \COMMENT{Initialize the \gls{ema} state for $\mG\mG^T$}
        \FOR{$t=1,\ldots,T$}
            \STATE $\mG_t=\nabla_{\mW_t}\mathcal{L}$
            \STATE $\mQ_t = \beta_3 \mQ_{t-1} + (1-\beta_3)\mG_t\mG_t^T$ \COMMENT{Accumulate the $\mG\mG^T$}
            \STATE $\vm_t = \beta_1\vm_{t-1}+(1-\beta_1)\mG_t$ \COMMENT{Accumulate the first moment}
            \IF{$t==1$ or $(t \mod K)==0$}
                \STATE $\mU_{f,t} = \eig(\mQ_t)$ \COMMENT{Obtain the $\mU$}
            \ELSE
                \STATE $\mU_{f,t} = \mU_{f,t-1}$
            \ENDIF
            \STATE $\tilde{\vm}_t = \mU_{f,t}^T\vm_t$ \COMMENT{Rotate the first moment}
            \STATE $\vv_t = \beta_2\vv_{t-1}+(1-\beta_2)(\mU_{f,t}^T\mG_t)\elesquare$ \COMMENT{Accumulate the second moments}
            \STATE $\mW_{t+1}=\mW_t -\lambda \mU_{f,t}\frac{\tilde{\vm}_t}{\sqrt{\vv_t}}$\COMMENT{Update in the original space}
        \ENDFOR
    \end{algorithmic}
\end{algorithm}

\subsection{SWAN}
\label{subapp: SWAN optimizer}
Recently, there is a newly proposed adaptive optimizer that completely removes the needs of storing internal states, called SWAN. It relies on two processing operators applied to raw current gradient: \emph{GradNorm} and \emph{GradWhitening}, as a replacement of first and second moments. For a current gradient $\mG$,
\begin{align}
    &\mathtt{GradNorm}(\mG) = \frac{ \mG - \bar{\vecg} \mathbf{1}_{n}^\top }{ \vs \mathbf{1}_{n}^\top  } \\
    & \mathtt{GradWhitening}(\mG) = (\mG\mG^T)^{-\frac{1}{2}}\mG
    \label{eq: gradnorm and whitening}
\end{align}
where $ \bar{\vecg} = \frac{1}{n} \sum_{i=1}^{n} \mG_{:, i} $ is the mean across rows; $ \vs = \sqrt{ \frac{1}{n} \sum_{i=1}^{n} (\mG_{:, i} - \bar{\vecg})^2} $ is the standard deviation across rows; $\bar{\vecg}$ and $\vs$ are ${m}$-dimensional column vectors; and $ \mathbf{1}_{n} $ is a ${n}$-dimensional column vector of ones. 
Then, SWAN performs the following to generate the update:
\begin{align}
    \tilde{\mG} = \mathtt{GradNorm}(\mG)\nonumber\\
    \Delta = \mathtt{GradWhitening}(\tilde{\mG})
    \label{eq: SWAN update}
\end{align}
We can see that the proposed $\mathtt{GradNorm}$ is equivalent to normalization up  to a scaling $\sqrt{n}$ when the mean $ \bar{\vecg}$ is $0$.
SWAN derives these two steps from investigating the LLM dynamics. In practice, SWAN proposes to compute the $(\mG\mG^T)^{-\frac{1}{2}}$ using Newton-Schulz iterations.

\subsection{Newton Schulz iteration}
\label{subapp: Newton schulz iteration}
In many machine learning applications, like successive whitening and coloring transform \citep{li2017universal,cho2019image,choi2021robustnet}, one often encountered the computation of square root inverse of some \gls{spd} matrix. One standard approach is to compute the \gls{evd} and take the square root inverse of the eigenvalue matrix. However, \gls{evd} is computationally expensive. Another alternative approach is use Newton-Schulz iteration (NS), an iterative updates of two matrix. Specifically, 
\begin{align*}
    \mY_0 &= \frac{\mA}{\Vert\mA\Vert_F}\;\;\;\ \mZ_0=\mI\\
    \mY_{t+1} &= \frac{1}{2}\mY_t(3\mI-\mZ_t\mY_t)\\
    \mZ_{t+1} &= \frac{1}{2}(3\mI-\mZ_t\mY_t)\mZ_t
\end{align*}
with convergence $\mY_t \rightarrow \frac{\mA^{\frac{1}{2}}}{\sqrt{\Vert\mA\Vert_F}}$ and $\mZ_t\rightarrow \mA^{-\frac{1}{2}}\sqrt{\Vert\mA\Vert_F}$. Typically, NS converges very fast with only 5 steps \citep{li2018towards,huang2019iterative}. 

\subsection{Muon}
\label{subapp: Muon}
Muon \citep{jordan2024muon}, is recently proposed to speed-up the training of LLMs, that relies on the whitening operator similar to SWAN. The core of the Muon is to orthogonalize the the momentum. The proposed update rule is 
\begin{align}
    \vm_t =& \beta_1 \vm_{t-1}+(1-\beta_1)\mG_t\nonumber\\
    \Delta =& \mathtt{GradWhitening}(\vm_t).
    \label{eq: Muon update}
\end{align}
Similarly, the $\mathtt{GradWhitening}$ step is computed using Newton-Schulz iteration. The main difference between Muon and SWAN is that Muon still requires the storage of first moments as state, whereas SWAN relies on the $\mathtt{GradNorm}$ operator applied to the raw gradient.  

\subsection{Lars}
\label{subapp: Lars}
SWAN and Muon both involve the whitening operator with/without normalization, respectively. On the other hand, Lars \citep{you2017lars} is an operator that only relies on layer-wise normalization. For each layer, it simply compute the first moments, followed by a normalization operation. The update rule for each layer is 
\begin{align}
    \vm_t = \beta_1\vm_{t-1} + (1-\beta_1)\mG_t\nonumber\\
    \Delta = \phi(\Vert\theta\Vert)\frac{\vm_t}{\Vert\vm_t\Vert}
    \label{eq: Lars update}
\end{align}
where $\phi$ is a scaling function with input of the parameter $\theta$ norm. One major difference of this layer-wise normalization to SWAN is that it is applied on the layer-wise level, whereas SWAN applies row/column-wise normalization.

\subsection{Low rank optimizers}
\label{subapp: Low rank optimizers}
The primary goal of low-rank optimizer is to reduce the memory consumed by the states of adaptive optimizers. The popularity of low-rank based method for large models starts from the well-known LoRA \citep{hu2021lora}, where each weight is inserted with an low-rank adapter $\mW+\mA\mB$ with $\mA \in \Rmr$ and $\mB\in \Rrn$ during the finetuning stage. This formulation directly modifies the model architecture. \cite{si2024flora} explicitly show that LoRA is secretly a gradient compressor, which translates the modification to model architecture into a low-rank optimizer with randomly sampled matrix. At the same time, GaLore \citep{zhao2024galore} popularizes the use of low-rank optimizers, which demonstrates on-par performance compared to full-rank Adam training. GaLore is proposed based on the analysis of reversible networks \citep{tian2020understanding}. However, in practice, transformer may not satisfy the reversibility condition. Thus, GaLore does not provide a clear understanding on why it works on LLMs. 

\Cref{alg: GaLore optimizer} summarizes the procedures. In practice, GaLore can be viewed as the composition of subspace search algorithm (i.e.~SVD) with standard Adam optimizer. The original GaLore does not provide an explanation on the choice of Adam. On the other hand, our analysis reveals that the GaLore is an approximate low-rank extension to a different optimizer, \gls{alicec}/AdaDiag/one-side SOAP, that is generalizes Adam (see \cref{subsec: alicec}). 

\begin{algorithm}
    \caption{GaLore Optimizer}
    \label{alg: GaLore optimizer}
    \begin{algorithmic}
        \STATE \bfseries{Input:} learning rate $\lambda$, decay rates $\beta_1$, $\beta_2$, rank $r$, update interval $k$, scale $\alpha$
        \FOR{$t=1,\ldots, T$}
        \STATE $\mG_t=\nabla_{\mW_t}\mathcal{L}$
        \IF{$t==1$ or $t\mod k==0$}
            \STATE $\mU_t=\svd(\mG_t, r)$
        \ELSE
            \STATE $\mU_t=\mU_{t-1}$
        \ENDIF
        \STATE $\bm{\sigma}_t=\mU_t^T\mG_t$
        \STATE $\Delta = Adam(\bm{\sigma}_t, \beta_1,\beta_2)$
        \STATE $\mW_t = \mW_{t-1} + \lambda\alpha\Delta$
        \ENDFOR
    \end{algorithmic}
\end{algorithm}

Since the states of Adam optimizer are based on the projected gradient $\bm{\sigma}_t$, the overall memory consumption of GaLore is $mn+2nr+mr$. 

\subsection{Apollo}
\label{subapp: Apollo}
Concurrent to this work, there is a recently proposed optimizer, called Apollo \citep{zhu2024apollo}, that only scales the raw gradient and obtains SGD-like memory with Adam-level performance. They propose to scale the columns or rows similar to normalization in SWAN, but the scaling factor is estimated following the procedure proposed by Fira \citep{chen2024fira}. The core idea of Apollo is to obtain $\Delta$ from GaLore algorithm (\cref{alg: GaLore optimizer}), followed by computing the column/row-wise norm of $\Delta$. This norm will be used as the scaling factor for the raw gradient $\mG$. Apollo has many variants. In particular, we choose Apollo-mini and Apollo-svd, where the former uses rank-1 random projection for scaling estimation, and the latter relies on the use of top $r$ singular vectors as the projection, same as GaLore.
Apollo-mini only maintains the rank-1 states, leading to significant memory savings. The memory consumption is $mn+2n+2$ for parameter $\mW\in\Rmn$. And Apollo-svd consumes the same memory as GaLore. \cref{alg: Apollo optimizer} summarizes the procedures.

\begin{algorithm}
    \caption{Apollo Optimizer}
    \label{alg: Apollo optimizer}
    \begin{algorithmic}
        \STATE \bfseries{Input:} learning rate $\lambda$, decay rates $\beta_1$, $\beta_2$, rank $r$, update interval $k$, scale $\alpha$
        \FOR{$t=1,\ldots, T$}
        \STATE $\mG_t=\nabla_{\mW_t}\mathcal{L}$
        \IF{$t==1$ or $t\mod k==0$}
            \STATE $\mU_t\sim \mathcal{N}(0,\frac{1}{r})$
            \STATE seed $\leftarrow$ an independent new random seed
        \ELSE
            \STATE $\mU_t=\mU_{t-1}$
        \ENDIF
        \STATE $\bm{\sigma}_t=\mU_t^T\mG_t$
        \STATE $\Delta_t = Adam(\bm{\sigma}_t, \beta_1,\beta_2)$
        \STATE $\mS_t \leftarrow \diagv(s_0, \ldots, s_m)\;\{s_i=\frac{\Vert\Delta_{t,:,i}\Vert_2}{\Vert\bm{\sigma_{t,:,i}}\Vert_2}\}$ 
        \STATE $\mW_t = \mW_{t-1} + \lambda\alpha\mG_t\mS_t$
        \ENDFOR
    \end{algorithmic}
\end{algorithm}
Note that when rank $r$ is set to $1$, they propose to use global scaling $\frac{\Vert\Delta_t\Vert_2}{\Vert\bm{\sigma_t}\Vert_2}$ instead of row/column-wise scaling.

\subsection{Subspace iteration}
\label{subapp: background subspace iteration}
The subspace iteration method—also known as the block power method is a classical iterative technique for computing the dominant eigenvalues and corresponding eigenvectors of a matrix. It generalizes the basic power method from operating on a single vector to operating on a subspace, typically spanned by a initial matrix. When the initial matrix is closed to the targeting eigen-matrix, the convergence is fast. Empirically, we found that only 1 step of iteration is enough to give a satisfactory performance. \Cref{alg: subspace iteration} summarizes the subspace iteration algorithm to for finding the top $r$ eigenvectors of a matrix $\mA$. We can see that the computation is bottlenecked by the matrix multiplication $\mA\mU_{t-1}$ which is $O(m^2r)$ if only performing 1 step. 

\begin{algorithm}
    \caption{Subspace iteration}
    \label{alg: subspace iteration}
    \begin{algorithmic}
        \STATE \bfseries{Input:} symmetric matrix $\mA\in\Rmm$, iteration step $T$, initial matrix $\mM\in\Rmr$
        \STATE $\mU_0=\mM$
        \FOR{$t=1,\ldots, T$}
        \STATE $\mH_t = \mA\mU_{t-1}$
        \STATE $\mU_t=\text{QR decomposition}(\mH_t)$
        \ENDFOR
        \STATE $\mV = \mU_T^T\mA\mU_T$
        \STATE $\mU = \eig(\mV)$
        \STATE {\bfseries Return} $\mU$
    \end{algorithmic}
\end{algorithm}

\section{Derivation of update formula}
\label{app: derivation of update formula}
In this section, we will explicitly show how to connect the solution from minimizing reconstruction loss of \gls{fim} (\cref{eq: UFE equation}) to corresponding update rule. 

\subsection{Shampoo's update formula}
\label{subapp: Shampoo update formula}
The key update formula of Shampoo is 
\begin{align*}
    \mW_t = \mW_{t-1} + \lambda \mL_{m,t}^{-\frac{1}{4}}\mG_t\mR_{n,t}^{-\frac{1}{4}}
\end{align*}
\begin{proof}
    From \cref{thm: optimal shampoo}, we simply apply the properties of Kronecker product to square-root version of natural gradient descent:
    \begin{align*}
        &\devect\left(\Ft^{-\frac{1}{2}}\vecg\right)\\
        =&\devect\left((\mR_{n}^{\frac{1}{2}}\otimes \mL_{m}^{\frac{1}{2}})^{-\frac{1}{2}}\vecg\right)\\
        =& \devect\left(\vect\left(\mL_{m}^{-\frac{1}{4}}\mG\mR_{n}^{-\frac{1}{4}}\right)\right)\\
        =& \mL_{m}^{-\frac{1}{4}}\mG\mR_{n}^{-\frac{1}{4}}
    \end{align*}
\end{proof}

\subsection{Generalization to whitening and normalization}
\label{subapp: update of generalization of whitening}
The square-root \gls{ngd} update with $\Ft$ in \cref{eq: generalization whitening} in \cref{coro: generalization to whitening and normalization} is 
\begin{equation}
    \devect\left(\Ft^{-\frac{1}{2}}\vecg\right) = \sqrt{n}\E[\mG\mG^T]^{-\frac{1}{2}}\mG
\end{equation}
\begin{proof}
    From the solution in \cref{eq: generalization whitening}, we can simply apply the properties of Kronecker product as in the derivation of Shampoo's update:
    \begin{align*}
        &\devect\left(\Ft^{-\frac{1}{2}}\vecg\right)\\
        =&\devect\left((\mI_n\otimes \mM)^{-\frac{1}{2}}\vecg\right)\\
        =&\devect\left(\vect\left(\sqrt{n}\mM^{-\frac{1}{2}}\mG\right)\right)\\
        =& \sqrt{n}\E[\mG\mG^T]^{-\frac{1}{2}}\mG
    \end{align*}
\end{proof}
Similarly, the square-root \gls{ngd} update with $\Ft=\mS\otimes \mI_m$ is 
\begin{equation}
    \devect\left(\Ft^{-\frac{1}{2}}\vecg\right) = \sqrt{m}\mG\mS^{-\frac{1}{2}}
\end{equation}
\begin{proof}
    This is trivial by applying the property of Kronecker product:
    \begin{align*}
        &\devect((\mS\otimes\mI_m)^{-\frac{1}{2}}\vecg)\\
    =& \devect(\vect(\mG\mS^{-\frac{1}{2}}))\\
    =&\mG\mS^{-\frac{1}{2}}
    \end{align*}
\end{proof}
\subsection{Update formula for \gls{alicec}}
\label{subapp: update of generlized adam}
\begin{proof}
    From the \cref{thm: alicec 1 step refinement}, we can apply the properties of block diagonal and Kronecker product with a full-rank $\mU$:
    \begin{align*}
        &\devect(\Ft^{-\frac{1}{2}}\vecg)\\
        =& \devect\left(\diagb\left(\mM_1^{-\frac{1}{}2},\ldots, \mM_n^{-\frac{1}{2}}\right)\vecg\right)\\
        =&\devect\left(\diagb\left(\mU\mD_1^{-\frac{1}{2}}\mU^T,\ldots,\mU\mD_{n}^{-\frac{1}{2}\mU^T}\right)\vecg\right)\\
        =&\devect\left((\mI_n\otimes \mU)\diagb(\sqrt{\mD_1},\ldots,\sqrt{\mD_n})(\mI\otimes \mU^T)\vecg\right)\\
        =&\devect\left((\mI_n\otimes \mU)\diagb(\sqrt{\mD_1},\ldots,\sqrt{\mD_n})\vect\left(\mU^T\mG\right)\right)\\
        =&\devect\left((\mI_n\otimes \mU)\vect\left(\frac{\mU^T\mG}{\sqrt{\E[(\mU^T\mG)\elesquare]}}\right)\right)\\
        =&\devect\left(\vect\left(\mU\frac{\mU^T\mG}{\sqrt{\E[(\mU^T\mG)\elesquare]}}\right)\right)\\
        =& \mU\frac{\mU^T\mG}{\sqrt{\E[(\mU^T\mG)\elesquare]}}
    \end{align*}
\end{proof}

\subsection{Update formula for SOAP}
\label{subapp: update formula for SOAP}
Based on the \cref{thm: optimal asham}, we can derive the update formula of the corresponding square-root \gls{ngd} following the same procedure as \gls{alicec}:
\begin{align*}
    \devect\left(\Ft^{-\frac{1}{2}}\vecg\right) = \mU_L\frac{\mU_L^T\mG\mU_R}{\sqrt{\E[(\mU_L^T\mG\mU_R)\elesquare]}}\mU_R^T.
\end{align*}
\begin{proof}
    \begin{align*}
        &\devect\left(\Ft^{-\frac{1}{2}}\vecg\right)\\
        =&\devect\left((\mU_R\otimes\mU_L)\diagm(\E[(\mU_L^T\mG\mU_R)\elesquare])^{-\frac{1}{2}}(\mU_R\otimes\mU_L)^T\vecg\right)\\
        =&\devect\left((\mU_R\otimes\mU_L)\diagm(\E[(\mU_L^T\mG\mU_R)\elesquare])^{-\frac{1}{2}}\vect\left(\mU_L^T\mG\mU_R\right)\right)\\
        =&\devect\left((\mU_R\otimes\mU_L)\vect\left(\frac{\mU_L^T\mG\mU_R}{\sqrt{\E[(\mU_L^T\mG\mU_R)\elesquare]}}\right)\right)\\
        =&\devect\left(\vect\left(\mU_L\frac{\mU_L^T\mG\mU_R}{\sqrt{\E[(\mU_L^T\mG\mU_R)\elesquare]}}\mU_R^T\right)\right)\\
        =&\mU_L\frac{\mU_L^T\mG\mU_R}{\sqrt{\E[(\mU_L^T\mG\mU_R)\elesquare]}}\mU_R^T
    \end{align*}
\end{proof}
Therefore, one can design the optimizer based on this update formula and exactly recovers the SOAP's procedure (\cref{eq: practical soap updates} in \cref{subapp: SOAP}).

\section{Theory and proof}
\label{app: theory and proof}
To prove the results, we need to first introduce some useful lemmas and inequalities.

\begin{lemma}
    Assume $\Ft$ is a block diagonal matrix with $n$ squared block matrix $\mM_i\in\Rmm$, then 
    \begin{equation}
        \min_{\Ft} \Fnorm{\Ft-\mF} = \min_{\{\mM_i\}_{i=1}^n} \sum_{i=1}^n \Fnorm{\mM_i} - 2\tr(\mM_i^T\E[\vg_i\vg_i^T]) + C
    \end{equation}
where $\vg_i$ is the $i^{\text{th}}$ column of gradient $\mG$, $C$ is a constant that is idenpendent of $\Ft$, and $\mF$ is the \gls{fim}. 
\label{lemma: block diagonal simplification}
\end{lemma}
\begin{proof}
    This is straightforward by expanding the \gls{f-norm}. 
    \begin{align*}
        &\Fnorm{\Ft-\mF}\\
        =& \tr\left((\Ft-\mF)^T(\Ft-\mF)\right)\\
        =&\Fnorm{\Ft} - 2\tr\left(\Ft^T\mF\right) + C\\
        =& \sum_{l=1}^{mn} \Ft_{l,:}^T\Ft_{:,l} - \Ft_{l,:}^T \mF_{:,l} + C\\
        =& \sum_{i=1}^n \Fnorm{\mM_i}-2\tr\left(\mM_i^T\E[\vg_i\vg_i^T]\right) + C
    \end{align*}
where $\Ft_{l,:}\in \mathbb{R}^{mn}$ indicates the $l^{\text{th}}$ row vector of $\Ft$ and $\Ft_{:,l}$ is the $l^{\text{th}}$ column vector. The last equation is obtained by the fact that $\Ft$ is a block diagonal matrix. So only the values of $\mF$ at the position of non-zero values $\Ft$ contributes to the trace, which is exactly the outer product: $\E[\vg_i\vg_i^T]$. 
\end{proof}

\begin{lemma}[Powers-Stormer inequality]
    For positive semi-definite operator $\mA$, $\mB$, we have the following inequality
    \begin{equation}
        \tr((\mA-\mB)^T(\mA-\mB)) \leq \Vert\mA^2 -\mB^2\Vert_1
        \label{eq: powers stormer inequality}
    \end{equation}
    \label{lemma: powers stormer inequality}
    where $\Vert\cdot\Vert_1$ is the trace norm.
\end{lemma}

\subsection{Proof of \cref{prop: adam solution}}
\label{subapp: proof of adam}
\begin{proof}
    From \cref{lemma: block diagonal simplification}, we have
    \begin{align*}
        &\Fnorm{\Ft-\mF}\\
        =& \sum_{i=1}^n \Fnorm{\mM_i} - 2\tr\left(\mM_i^T\E[\vg_i\vg_i^T]\right)\\
        =& \sum_{i=1}^n\sum_{j=1}^m M_{i,jj}^2 - 2M_{i,jj}\E[g_{i,j}^2]
    \end{align*}
    By taking the derivative w.r.t $M_{i,jj}$, we have
    \begin{align*}
        M_{i,jj} = \E[g_{i,j}^2]
    \end{align*}
    Thus, we have $\Ft = \diag(\E[\vecg^2])$.
\end{proof}
\subsection{Proof of \cref{thm: optimal shampoo}}
\label{subapp: proof of shampoo optimiality}

To prove this theorem, we need to leverage the \cref{coro: generalization to whitening and normalization} for generalized whitening (\cref{eq: generalization whitening}) in \cref{subsec: sve}. This is proved in \cref{subapp: proof normalization}. But in the following, we will provide an alternative proof for completeness.

\begin{proof}
    From \cref{lemma: block diagonal simplification}, we have
    \begin{align*}
        &\Fnorm{\Ft-\mF} \\
        =& \sum_{i=1}^n \Fnorm{\mM} - 2\tr(\mM^T\E[\vg_i\vg_i^T]) + C\\
        =& n\Fnorm{\mM} - 2\tr(\mM^T\E[\sum_{i=1}^n \vg_i\vg_i^T]) + C\\
        =& n\Fnorm{\mM} - 2\tr(\mM^T\E[\mG\mG^T]) + C \\
    \end{align*}
    To minimize this, we take the derivative w.r.t. $\mM$, we have
    \begin{align*}
        2n\mM - 2\E[\mG\mG^T] = 0 \Rightarrow \mM = \frac{1}{n} \E[\mG\mG^T]
    \end{align*}
\end{proof}
Next, we prove another proposition that is "symmetric" to the whitening results in \cref{coro: generalization to whitening and normalization}.
\begin{proposition}
    Assume $\family = \{\mR_n \otimes \mI_m\}$, where $\mR_n\in \Rnn$ is \gls{spd} matrix, then \cref{eq: UFE equation} can be analytically solved with the optimal solution as 
    \begin{equation}
        \mR_n^* = \frac{1}{m} \E[\mG^T\mG]
        \label{eq: optimal shampoo right}
    \end{equation}
    \label{prop: optimal shampoo right}
\end{proposition}
\begin{proof}
    Since $\mR_n\otimes \mI_m$ does not have a nice block diagonal structure like the previous proposition, we need to analyze it a bit more. First, we have
    \begin{align*}
        &\Fnorm{\mR_n \otimes \mI_m - \mF}\\
        =& \Fnorm{\mR_n\otimes \mI_m} - 2\tr\left(\underbrace{(\mR_n\otimes \mI_m)^T\E[\vecg\vecg^T]}_{\mZ}\right) + C\\
    \end{align*}
Since we only care about the diagonal of $\mZ$, therefore, we only inspect the block diagonal of $\mZ$ with each block $\mZ_i$ of size $\Rmm$, and $i=1,\ldots, n$. By basic algebra, we have
\begin{align*}
    \mZ_i = \sum_{k=1}^n R_{ik} \vg_k\vg_i^T
\end{align*}
where $\vg_k$ is the $k^{\text{th}}$ column of $\mG$. Therefore, we can simplify the trace of $\mZ$ as 
\begin{align*}
    \tr(\mZ) &= \sum_{i=1}^n\tr(\mZ_i)\\
    =& \tr(\sum_{i=1}^n\sum_{k=1}^n R_{ij}\vg_k\vg_i^T)\\
    =& \sum_{i=1}^n\sum_{k=1}^n\sum_{j=1}^m R_{ik}[\mG]_{ji}[\mG^T]_{kj}
\end{align*}
where $[\mG]_{ji}$ is the element of $\mG$ at $j^{\text{th}}$ row and $i^{\text{th}}$ column. 

Now, let's perform the same analysis of the following quantity
\begin{align*}
    \tr\left((\mI_m \otimes \mR_n)\E[\vecgt\vecgt^T]\right)
\end{align*}
where $\vecgt$ is the vectorized transposed gradient $\mG^T$. Namely, it now stacks the rows of $\mG$ instead of columns of $\mG$ like $\vecg$. This object is simple to treat due to its block diagonal structure, by algebric manipulation, we have
\begin{align*}
    \tr\left((\mI_m\otimes \mR_n)\E[\vecgt\vecgt^T]\right) &= \underbrace{\sum_{k=1}^m}_{\text{over blocks}}\tr(R_{ij}\underbrace{[\mG^T]_k}_{\text{kth column of }\mG^T}[\mG^T]_k^T)\\
    =&\sum_{k=1}^m \sum_{i=1}^n \sum_{j=1}^n R_{ij} [\mG^T]_{jk}[\mG]_{ki}
\end{align*}
Now, let's change the variable $i=i$, $j=k$ and $k=j$, the above becomes
\begin{align}
    &\tr\left((\mI_m\otimes \mR_n)\E[\vecgt\vecgt^T]\right) \nonumber\\
    =& \sum_{j=1}^m \sum_{i=1}^n \sum_{k=1}^n R_{ik} [\mG^T]_{kj}[\mG]_{ji} \nonumber\\
    =& \tr(\mZ) \label{eq: proof NI=IN}
\end{align}
We should also note that
\begin{align*}
    &\Fnorm{\mR_n\otimes \mI_m}\\
    =& \tr\left((\mR_n\otimes \mI_m)^T(\mR_n\otimes \mI_m)\right)\\
    =& \tr\left((\mR_n^T\mR_n)\otimes \mI_m\right)\\
    =& \tr(\mR_n^T\mR_n)\tr(\mI_m)\\
    =&\tr\left((\mI_m\otimes \mR_n)^T(\mI_m\otimes \mR_n)\right)\\
    =& \Fnorm{(\mI_m\otimes \mR_n)}
\end{align*}
Therefore, by using the above equation and \cref{eq: proof NI=IN}, the original minimization problem is translated to 
\begin{align*}
    \argmin_{\mR_n} \Fnorm{\mR_n\otimes \mI_m -\mF} = \argmin_{\mR_n}\Fnorm{\mI_m \otimes \mR_n -\E[\vecgt\vecgt^T]}
\end{align*}
Thus, we can leverage \cref{coro: generalization to whitening and normalization} to obtain the optimal solution
\begin{align*}
    \mR_n^* = \frac{1}{m} \E[\mG^T\mG]
\end{align*}
\end{proof}

With the above two propositions, we can start to prove \cref{thm: optimal shampoo}.
\begin{proof}
    First, we note that
    \begin{align*}
        &\Fnorm{\Rnr\otimes \Lmr - \mF}\\
        =& \Fnorm{\underbrace{(\mR_n\otimes \mI_m)^{\frac{1}{2}}}_{\mA}\underbrace{(\mI_n\otimes \mL_m)^{\frac{1}{2}}}_{\mB}-\underbrace{\E[\vecg\vecg^T]^{\frac{1}{2}}}_{\mC}\E[\vecg\vecg^T]^{\frac{1}{2}}}\\
        =& \Fnorm{\mA\mB - \mC\mC}
    \end{align*}
    Next, we will upper bound this quantity.
    First, we have
    \begin{align*}
        \mA\mB-\mC\mC = \mA(\mB-\mC) + (\mA-\mC)\mC
    \end{align*}
    By triangular inequality, we have
    \begin{align*}
        &\Vert\mA\mB-\mC\mC\Vert_F\\
        &\leq \Vert\mA(\mB-\mC)\Vert_F+\Vert(\mA-\mC)\mC\Vert_F \\
        &\leq \Vert\mA\Vert_F\Vert\mB-\mC\Vert_F+ \Vert\mA-\mC\Vert_F\Vert\mC\Vert_F\\
        &\leq (\Vert\mA-\mC\Vert_F+\Vert\mC\Vert_F)\Vert\mB-\mC\Vert_F+ \Vert\mA-\mC\Vert_F\Vert\mC\Vert_F\\
        &= \Vert\mA-\mC\Vert_F\Vert\mB-\mC\Vert_F+\Vert\mC\Vert_F(\Vert\mB-\mC\Vert_F+ \Vert\mA-\mC\Vert_F)
    \end{align*}
    Now, the squared norm can be upper bounded by 
    \begin{align}
        \Fnorm{\mA\mB-\mC\mC} \leq& 3\left(\Fnorm{\mA-\mC}\Fnorm{\mB-\mC}+\Fnorm{\mC}\Fnorm{\mA-\mC}+\Fnorm{\mC}\Fnorm{\mB-\mC}\right)\nonumber\\
        \leq&3\left(mn\Vert\mA^2-\mC^2\Vert_F\Vert\mB^2-\mC^2\Vert_F+\sqrt{mn}\Fnorm{\mC}\Vert\mA^2-\mC^2\Vert_F+\sqrt{mn}\Fnorm{\mC}\Vert\mB^2-\mC^2\Vert_F\right)
        \label{eq: proof upper bound shampoo}
    \end{align}
The first inequality is obtained by the fact that for any three matrix $\mP$, $\mQ$ and $\mH$, we have
\begin{align*}
    \Fnorm{\mP+\mQ+\mH}\leq& \left(\Vert\mP\Vert_F+\Vert\mQ\Vert_F+\Vert\mH\Vert_F\right)^2\\
    =& \Fnorm{\mP}+\Fnorm{\mQ}+\Fnorm{\mH} + 2\Vert\mP\Vert_F\Vert\mQ\Vert_F + 2\Vert\mP\Vert_F\Vert\mH\Vert_F+2\Vert\mQ\Vert_F\Vert\mH\Vert_F\\
    \leq& 3\left(\Fnorm{\mP}+\Fnorm{\mQ}+\Fnorm{\mH}\right)
\end{align*}
The second inequality is obtained by directly applying Powers-Stormer's inequality and Holder's inequality. For completeness, we will show how to upper-bound $\Fnorm{\mA-\mC}$, the rest can be bounded in the same way. 
From \cref{lemma: powers stormer inequality} and both $\mA$, $\mC$ are \gls{spd} matrix, we have
\begin{align*}
    \Fnorm{\mA-\mC}\leq \Vert\mA^2-\mC^2\Vert_1
\end{align*}
Then, we can select $p=q=2$ for Holder's inequaity and obtain
\begin{align*}
    \Vert\mA^2-\mC^2\Vert_1\leq \sqrt{mn}\Vert\mA^2-\mC^2\Vert_F
\end{align*}
where $\sqrt{mn}$ comes from the $\Vert\mI_{mn}\Vert_F$ in Holder's inequality. By substitute it back, we obtain the upper bound.

We can see that minimizing the upper bound \cref{eq: proof upper bound shampoo} is equivalent to minimize each $\Vert\mA^2-\mC^2\Vert_F$, $\Vert\mB^2-\mC^2\Vert_F$ individually, and 
\begin{align*}
    \Vert\mA^2-\mC^2\Vert_F &= \Vert\mR_n\otimes \mI_m - \mF\Vert_F\\
    \Vert\mB^2-\mC^2\Vert_F &= \Vert\mI_n\otimes \mL_m-\mF\Vert_F
\end{align*}
Thus, from \cref{coro: generalization to whitening and normalization} and \cref{prop: optimal shampoo right}, we prove the theorem. 
\end{proof}

\subsection{Proof of \cref{coro: generalization to whitening and normalization} and \cref{prop: two sided scaling}}
\label{subapp: proof normalization}
Instead of proving the \cref{coro: generalization to whitening and normalization}, we propose a generalization to those gradient operations, where \cref{coro: generalization to whitening and normalization} is a special case.

\paragraph{Structure assumption} We consider $\family=\{\mS\otimes \mM\}$ with identical \gls{spd} $\mM\in\Rmm$ and positive diagonal $\mS\in\Rnn$. 
The following theorem proves that the optimal solution can be solved by a fixed-point iteration. 
\begin{theorem}
    Assuming $\family=\{\mS\otimes \mM\}$ with positive diagonal $\mS\in\Rnn$ and \gls{spd} $\mM\in\Rmm$, and $\E_{GG'}[(\mG^T\mG')\elesquare]$ contains positive values, solving \cref{eq: UFE equation} admits a fixed point procedure:
    \begin{align}
        \diag(\mS) = \frac{\diag(\E[\mG^T\mM\mG])}{\Fnorm{\mM}},\;\;\;
        \mM=\frac{\E[\mG\mS\mG^T]}{\Fnorm{\mS}}.
        \label{eq: optimal S and M}
    \end{align}
    The solution $\diag(\mS^*)$ converges to the principal eigenvector of $\E[(\mG^T\mG')\elesquare]$ up to a scaling with unique $\mS^*\otimes \mM^*$. 
    \label{thm: generalization to normal and whiten}
\end{theorem}

To prove \cref{thm: generalization to normal and whiten}, we first introduce some classic results. 
\begin{theorem}[Perron-Frobenius theorem]
For a matrix $\mA\in\Rnn$ with positive entries, the principal eigenvalue $r$ is positive, called Perron-Frobenius eigenvalue. The corresponding eigenvector $\vv$ of $\mA$ is called Perron vector and only contains positive components: $\mA\vv=r\vv$ with $v_i>0$. In addition, there are no other positive eigenvectors of $\mA$. 
\label{thm: Perron-Frobenius theorem}
\end{theorem}

\begin{definition}[Hilbert projective metric]
    For any given vectors $\vv$, $\vw$ in $C\slash \{0\}$ where $C$ is a closed convex pointed non-negative cone $C$, i.e.~$C\cap (-C)=\{0\}$, the Hilbert projective metric is defined as 
    \begin{align*}
        d_H(\vv,\vw) = \log \left(\max_i \frac{v_i}{w_i}\right) - \log \left(\min_i \frac{v_i}{w_i}\right)
    \end{align*}
\end{definition}
This is a pseudo metric since it has a scaling invariance property: $d_H(\vv,\alpha\vm)=d_H(\vv,\vm)$ for $\alpha>0$. This means $d_H(\vv,\vm)=0$ does not mean $\vv=\vm$ but $\vv=\alpha\vm$ with some positive scaling $\alpha$. However, this is a metric on the space of rays inside the cone. 

\begin{theorem}[Birkhoff-Hopf theorem]
Let $\mP\in\Rnn$ be a positive matrix and let
\begin{align*}
    \kappa(\mP) = \inf\left\{\alpha\geq 0: d_H(\mP\vx,\mP\vy)\leq \alpha d_H(\vx,\vy), \forall \vx,\vy \in C_+, \vx \sim \vy \right\}
\end{align*}
where $C_+$ is the cone that each element is non-negative and $\sim$ is the induced equivalence relation. Namely, if $\vx\sim\vy$, there exists $\alpha,\beta>0$ such that $\alpha\vx<\vy<\beta\vx$, and $\vx<\vy$ means $y-x\in C_+$. Then, it holds
\begin{align*}
    \kappa(\mP) = \tanh{\frac{1}{4}\Delta(\mP)}\;\;\;\text{with}\;\Delta(\mP) = \max_{i,j,k,l}\frac{P_{ij}P_{kl}}{P_{il}P_{kj}}
\end{align*}
\label{thm: Birkhoff-Hopf}
\end{theorem}
This theorem suggests that when $\mP$ is a positive matrix, the corresponding linear mapping is contractive since $\tanh(\cdot)\leq 1$ under Hilbert projective metric. 

Now, let's prove the \cref{thm: generalization to normal and whiten}.
\begin{proof}
    First, we can simplify the \cref{eq: UFE equation} using \cref{lemma: block diagonal simplification}:
    \begin{align*}
        &\Fnorm{\mS\otimes\mM - \mF}\\
        =& \sum_{i=1}^n S_i^2\Fnorm{\mF} - 2\tr(S_i\mM\E[\vg_i\vg_i^T]) +C
    \end{align*}
    Then, we simply take its derivative w.r.t. $s_i$, and obtain
    \begin{align*}
        &2S_i\Fnorm{\mM} = 2\tr(\mM\E[\vg_i\vg_i^T])\\
        \Longrightarrow& S_i = \frac{\tr(\mM\E[\vg_i\vg_i^T])}{\Fnorm{\mM}}\\
        \Longrightarrow& \diag(\mS) = \frac{\diag\left(\E[\mG^T\mM\mG]\right)}{\Fnorm{\mM}}
    \end{align*}
    Similarly, we have
    \begin{align*}
        \mM =& \frac{\sum_{i=1}^nS_i\E[\vg_i\vg_i^T]}{\Fnorm{\mS}}\\
        =& \frac{\E[\mG\mS\mG^T]}{\Fnorm{\mS}}
    \end{align*}
These define an iterative procedure. Next, we will show it converges.
Let's substitute $\mM$ into $\mS$, and obtain
\begin{align*}
    \mS =& \diag\left(\E_{G}\left[\mG^T\E_{G'}[\mG'\mS\mG^{'T}]\mG\right]\right)\alpha(\mS)\\
    =& \diag\left(\E_{GG'}\left[\underbrace{\mG\mG^{'T}}_{\mH}\mS\mG^{'T}\mG\right]\right)
\end{align*}
where $\alpha(\mS)$ is the scaling term. In the following, we use $\E$ as $\E_{GG'}$.
Since we can show
\begin{align*}
    S_i = \E\left[\sum_{j}^nS_j\right],
\end{align*}
we can write $\mS$ in its vector format:
\begin{align*}
    \vs = \underbrace{\E\left[\mH\elesquare\right]}_{\mP}\vs.
\end{align*}
From the assumption, we know $\mP$ contains only positive values, let's define a quotient space for positive vectors $\vs$ and $\vq$ under the equivalence relation $\vs\sim \vs'$ if $\vs =\alpha \vs'$ for some positive scaling $\alpha$. Namely, we define a space of rays inside the positive cone. Therefore, the Hilbert projective metric becomes a real metric inside the quotient space. 

From the \cref{thm: Birkhoff-Hopf}, we know the linear mapping associated with $\mP$ is contractive. Therefore, we can follow the proof of Banach fixed point theorem on the previously defined quotient space with Hilbert projective metric to show the convergence of this fixed point iteration on $\vs$.  

Now, we show the solution $\vs^*$ is always positive. Since it is converging, therefor, the solution satisfies 
\begin{align*}
    \vs^* = \alpha(\vs^*) \mP\vs^*
\end{align*}
This is equivalent to finding the eigenvectors of $\mP$. By leveraging Perron-Frobenius theorem (\cref{thm: Perron-Frobenius theorem}), we know $\vs^*$ is the principal eigenvector of $\mP$, and only contain positive values. It is also easy to verify that this fixed point converges upto a positive scaling factor (this is expected since the contractive mapping holds true for the quotient space with Hilbert metric, that is invariant to scaling.)

Although $\vs^*$ is not unique, but $\mS\otimes \mM$ is, since for arbitrary positive scaling $\beta$
\begin{align*}
    &\vs^{'*}=\beta\vs^* \Longrightarrow \mM^{'*} = \frac{1}{\beta}\frac{\E[\mG\mS^*\mG^T]}{\Vert\vs\Vert_2^2}\\
    \Longrightarrow& \mS^{'*}\otimes\mM^{'*} = \mS^*\otimes \mM^*
\end{align*}
\end{proof}

Therefore, \cref{coro: generalization to whitening and normalization} is a direct consequence by substituting $\Ft = \mI_n\otimes\mM$ and $\Ft = \mS\otimes \mI_m$ into \cref{eq: optimal D for compensation}.

Next, we prove \cref{prop: two sided scaling}.
\begin{proof}
    From the \cref{thm: generalization to normal and whiten}, the iterative procedure for $\mQ$ can be simply obtained by taking the diagonals of $\mM$:
    \begin{align*}
        \mQ = \frac{\diag\left(\E\left[\mG\mS\mG^T\right]\right)}{\Fnorm{\mS}}.
    \end{align*}
    Following the same proof strategy of \cref{thm: generalization to normal and whiten}, we substitute $\mQ$ into the update of $\mS$ and re-write it into the vector format. First, let's rewrite the update of $\mS$
    \begin{align*}
        &S_i \propto \E[\sum_{j=1}G_{ji}^2 Q_j]\\
        \Longrightarrow& \vs = \frac{\mP^T \vq}{\Vert\vq\Vert_2^2}
    \end{align*}
    where $\mP=\E[\mG\elesquare]$. Similarly, $\vq = \frac{\mP\vs}{\Vert\vs\Vert_2^2}$.
    Thus,
    \begin{align*}
        \vs = \alpha(\vs) \mP^T\mP \vs
    \end{align*}
    From the assumption $\mP$ contains only positive values, we can follow the exact same argument made in \cref{thm: generalization to normal and whiten} to show the convergence of this fixed point update and the positivity of the final solution $\vs^*$. Precisely, $\vs^*$ and $\vq^*$ are the right and left principal singular vectors of $\mP$, respectively, and $\mS^*\otimes \mQ^*$ are unique. 
\end{proof}
\subsection{Proofs of \gls{alicec}}
\label{subapp: proofs of alicec}
\subsubsection{Proof of \cref{thm: alicec 1 step refinement}}

\begin{proof}
    For simplicity, we omit the subscript $f$ in $\mU_{f}$. If we assume all $\mD_i$ are equal and only contain positive values, then each block $\mU\mD_i\mU^T$ are the same for all $i$, and it is \gls{spd} matrix. Then, to minimize the the loss \cref{eq: UFE equation}, we can directly leverage the whitening results in \cref{coro: generalization to whitening and normalization}, and obtain $\mM^*=\E[\mG\mG^T]$. Due to the structure of $\mU\mD\mU^T$, the optimal $\mU^*$ is exactly the eigen-matrix of $\mM^*$. 

    Next, we prove for any fixed $\mU$, we can find the corresponding optimal $\mD_i$. 
    From the block diagonal structure and \cref{lemma: block diagonal simplification}, we have
    \begin{align*}
        &\Fnorm{\Ft-\mF}\\
        =&\sum_{i=1}^n \Fnorm{\mU\mD_i\mU^T} - 2\tr\left(\mU^T\E[\vg_i\vg_i^T]\mU \mD_i\right) + C\\
        =& \sum_{i=1}^n \Fnorm{\mD_i} - 2\tr\left(\mU^T\underbrace{\E[\vg_i\vg_i^T]}_{\mH_i}\mU\mD_i\right) + C\\
        =&\sum_{i=1}^n\sum_{j=1}^m D_{i,jj}^2 - 2\sum_{i=1}^n\sum_{j=1}^m D_{i,jj}\vu_j^T\mH_i\vu_j + C
    \end{align*}
Taking the derivative w.r.t. $D_{i,jj}$, we can find the optimal $D_{i,jj}$ is 
\begin{align*}
    D^*_{i,jj} =& \vu_j^T\mH_i\vu_j\\
    =& \E[(\vu_j^T\vg_i)^2]
\end{align*}
Now, by simple algebra manipulation, we have
\begin{align*}
    \mD^*_i = \diagv(\E[(\mU^T\vg_i)^2])
\end{align*}
where $\diag_M$ is to expand the vector to a diagonal matrix. 
Finally, for $\tilde{\mD}$, we have
\begin{equation}
    \tilde{\mD} = \diagm(\E[(\mU^T\mG)\elesquare])
\end{equation}

The optimality of $\widetilde{\mD}$ can also be obtained by leveraging the Lemma 1 in \citep{george2018fast}, and set the eigenbasis as $\mI_n\otimes \mU$. 
\end{proof}
\subsection{Proof of \cref{prop: subspace switching}}
\label{subapp: proof subspace switching}
\begin{proof}
    Within the time block $i+1$ with low-rank mapping $\mU$, the gradient at each step can be decomposed as 
    \begin{align*}
        \mG_t = \underbrace{\mU\mU^T\mG_t}_{\widetilde{\mG_t}} + \underbrace{(\mG_t - \mU\mU^T\mG_t)}_{\mR_t}
    \end{align*}
    Therefore, the true state $\mQ^*_{(i+1)k}$ can be simplified as 
    \begin{align*}
        & \mQ_{(i+1)k} = \sum_{t=ik+1}^{(i+1)k}\mG_t\mG_t^T\\
        =& \sum_{t=ik+1}^{(i+1)k} (\widetilde{\mG}_t+\mR_t)(\widetilde{\mG}_t+\mR_t)^T\\
        =& \sum_{t=ik+1}^{(i+1)k} \widetilde{\mG}_t\widetilde{\mG}_t^T+ \underbrace{\widetilde{\mG}_t\mR_t^T}_{0}+ \underbrace{\mR_t\widetilde{\mG}_t^T}_{0} + \mR_t\mR_t^T
    \end{align*}
The third equality is obtained because we assume $\mG_t\mG_t^T$ shares the same eigen-basis as $\mQ^*_{ik}$. Namely, 
\begin{align*}
\mG_t\mG_t^T=&[\mU,\mU_c]\left[\begin{array}{cc}
    \mA_{t} & 0 \\
   0  & \bm{\Sigma}_t
\end{array}\right]\left[\begin{array}{c}
     \mU^T  \\
     \mU_c^T
\end{array}\right]\\
=& \mU\mA_t\mU^T + \mU_c\bm{\Sigma}_t\mU_c^T
\end{align*}
where $\mA_t$ and $\bm{\Sigma}_t$ are diagonal matrix. Then, we have
\begin{align*}
    &\widetilde{\mG}_t\mR_t^T\\
    =& \mU\mU^T\mG_t(\mU_c\mU_c^T\mG_t)^T\\
    =& \mU\mU^T(\mU\mA_t\mU^T+\mU_c\bm{\Sigma}_t\mU_c^T)\mU_c\mU_c^T\\
    =& \mU\mA_t\underbrace{\mU^T\mU_c}_{0}\mU_c^T + \mU\underbrace{\mU^T\mU_c}_{0}\bm{\Sigma}_t\mU_c^T\mU_c\mU_c^T\\
    =&\bm{0}
\end{align*}
In addition, we can also simplify
\begin{align*}
    &\mR_t\mR_t^T\\
    =& \mU_c\mU_c^T\mG_t\mG_t^T\mU_c\mU_c^T\\
    =& \mU_c\mU_c^T(\mU\mA_t\mU^T+\mU_c\bm{\Sigma}_t\mU_c^T)\mU_c\mU_c^T\\
    =&\mU_c\bm{\Sigma}_t\mU_c^T
\end{align*}
Therefore, 
\begin{align*}
    \mQ^*_{(i+1)k} = \sum_{t=ik+1}^{(i+1)k} \widetilde{\mG}_t\widetilde{\mG}_t^T + \mU_c\bm{\Sigma}_t\mU_c^T
\end{align*}
\end{proof}

\subsection{Proof of \cref{thm: optimal compensation}}
\label{subapp: proof optimal compensation}

\begin{proof}
For simplicity, we ignore the subscript $t$ for the following proof.
First, we let $\mO=\mS^{-2}$
    Then, the loss function can be written as 
    \begin{align*}
        &\Fnorm{\mO\otimes \mU_c\mU_c^T-\Ft_{c}}\\
        =& \sum_{i=1}^n \Fnorm{O_{ii} \mU_c\mU_c^T} - 2\tr((O_{ii}\mU_c\mU_c^T)^T(\mU_c\mM_i\mU_c^T))\\
        =&\sum_{i=1}^n \Fnorm{O_{ii} \mU_c\mU_c^T} - 2\tr(O_{ii}\mM_i)\\
        =& \sum_{i=1}^n\sum_{k=1}^{m-r}\sum_{j=1}^m O_{ii}^2U_{c,jk}^2 - 2\tr(O_{ii}\mM_i)
    \end{align*}
    where $\mM_i=\diag(\E[(\mU_c^T\vg_i)^2])$, $\vg_i$ is the $i^{\text{th}}$ column of $\mG$, and $U_{c, jk}$ is the element in $j^{\text{th}}$ row, $k^{\text{th}}$ column of $\mU_c$. 
    Then, we take the derivative w.r.t. $O_{ii}$, and we have
    \begin{align*}
        &2O_{ii} \sum_{k=1}^{m-r}\sum_{j=1}^m U_{c,jk}^2 = 2\sum_{k=1}^{m-r} \E[(\mU_c^T\vg_i)^2_{k}]\\
        \Longrightarrow& O_{ii} = \frac{\E[\sum_{k=1}^{m-r}(\mU_c^T\vg_i)_k^2]}{m-r}
    \end{align*}
    This form still requires the access to $\mU_c$. Next, let's simplify it. 
    First, let $\tilde{\mU} = [\mU, \mU_c]$ to be the complete basis, we can show
    \begin{align*}
        &\sum_{k=1}^m (\tilde{\mU}^T\vg_i)^2_k \\
        =&\tr((\tilde{\mU}^T\vg_i)^T(\tilde{\mU}^T\vg_i))\\
        =&\tr(\vg_i^T\underbrace{\tilde{\mU}\tilde{\mU}^T}_{\mI}\vg_i)\\
        =&\vg_i^T\vg_i\\
    \end{align*}
    Now, let's re-write the above in a different format:
    \begin{align*}
    &\sum_{k=1}^m (\tilde{\mU}^T\vg_i)^2_k \\
        =&\tr(\vg_i^T\tilde{\mU}\tilde{\mU}^T\vg_i)\\
        =&\tr(\tilde{\mU}^T\vg_i\vg_i^T\tilde{\mU})\\
        =&\tr\left(\left[\begin{array}{c}
             \mU^T  \\
             \mU_c^T
        \end{array}\right]\vg_i\vg_i^T[\mU,\mU_c]\right)\\
        =& \tr(\mU\mU^T(\vg_i\vg_i^T)+ \mU_c\mU_c^T(\vg_i\vg_i^T))\\
        =& \tr((\mU^T\vg_i)^T(\mU^T\vg_i)) + \tr((\mU_c^T\vg_i)^T(\mU_c^T\vg_i))\\
        =& \sum_{k=1}^r(\mU^T\vg_i)_k^2 + \sum_{k=1}^{m-r} (\mU_c^T\vg_i)^2_k
    \end{align*}
    Therefore, we have
    \begin{align*}
        \E[\sum_{k=1}^{m-r}(\mU_c^T\vg_i)_k^2] = \E[\vg_i^T\vg_i-\sum_{k=1}^r (\mU^T\vg_i)_k^2]
    \end{align*}
    So, we have 
    \begin{align*}
        \diag(\mO) = \frac{\E[\bm{1}_m^T\mG\elesquare - \bm{1}_r^T(\mU^T\mG)\elesquare]}{m-r}
    \end{align*}
    and 
    \begin{align*}
        \diag(\mD) = \frac{\sqrt{m-r}}{\sqrt{\E[\bm{1}_m^T\mG\elesquare - \bm{1}_r^T(\mU^T\mG)\elesquare]}} 
    \end{align*}

\end{proof}
\subsection{Proof of \cref{thm: optimal asham}}
\label{subapp: proof asham}

\begin{proof}
    The proof strategy is a straightforward combination of \cref{thm: optimal shampoo} and \cref{thm: alicec 1 step refinement}. First, when we assume $\tilde{\mD}$ has the Kronecker product structure, one can easily write 
    \begin{align*}
        &(\mU_R\otimes \mU_L)(\mS_R\otimes \mS_L)(\mU_R\otimes \mU_L)^T\\
        =& (\mU_R\otimes \mU_L)\left[(\mS_R\mU_R^T)\otimes (\mS_L\mU_L^T)\right]\\
        =& \underbrace{(\mU_R\mS_R\mU_R^T)}_{\mA}\otimes \underbrace{(\mU_L\mS_L\mU_L^T)}_{\mB}
    \end{align*}
    Therefore the loss (\cref{eq: UFE equation}) becomes
    \begin{align*}
        \Fnorm{\mA\mB-\mC\mC}
    \end{align*}
    where $\mC = \E[\vecg\vecg^T]^{\frac{1}{2}}$. 
    This is exactly the formulation used in \cref{thm: optimal shampoo} with $\mR_n^{\frac{1}{2}} = \mU_R\mS_R\mU_R^T$ and $\mL_m^{\frac{1}{2}} = \mU_L\mS_L\mU_L^T$. 

    Thus, by directly utilizing \cref{thm: optimal shampoo}, we can see the optimal solution
    \begin{align*}
        \mU_R\mS_R^2\mU_R^T &= \E[\mG^T\mG]\\
        \mU_L\mS_L^2\mU_L^T &= \E[\mG\mG^T]
    \end{align*}
    Due to the structural assumption of $\mU_R$, $\mU_L$, $\mS_R$, $\mS_L$, their corresponding optimal solution can directly obtained using eigenvalue decomposition. 

    Now, let's prove the optimal $\tilde{\mD}$ with any fixed $\mU_R$, $\mU_L$. This is also straightforward by applying the same technique as \cref{thm: alicec 1 step refinement}. 
    The loss can be written as 
    \begin{align*}
        &\Fnorm{\underbrace{(\mU_R\otimes \mU_L)}_{\Pi}\tilde{\mD}\underbrace{(\mU_R\otimes \mU_L)^T}_{\Pi^T}-\mF}\\
        =& \Fnorm{\Pi\tilde{\mD}\Pi^T} - 2\tr\left(\Pi^T\E[\vecg\vecg^T]\Pi\tilde{\mD}\right) + C
    \end{align*}
    Since it is easy to verify orthonormality of $\Pi$, i.e.$\Pi^T\Pi=\mI$, the above is simplified to
    \begin{align*}
        &\Fnorm{\tilde{\mD}}-2\tr\left(\Pi^T\E[\vecg\vecg^T]\Pi\tilde{\mD}\right)\\
        &=\sum_{i=1}^{mn} D_{ii}^2 -2 \sum_{i=1}^{mn} D_{ii}[\Pi]_{i}^T\E[\vecg\vecg^T][\Pi]_{i}
    \end{align*}
    where $[\Pi]_i$ is the $i^{\text{th}}$ column of matrix $\Pi$. Then, by taking the derivative, the optimal $D_{ii}$:
    \begin{align*}
        D^*_{ii} &= [\Pi]_i^T\E[\vecg\vecg^T][\Pi]_i\\
        =&\E[([\Pi]_i^T\vecg)^2]
    \end{align*}
    Therefore, 
    \begin{align*}
        \diag(\tilde{\mD}^*) &= \E[(\Pi^T\vecg)^2]\\
        =& \E[((\mU_R^T\otimes \mU_L^T)\vecg)^2]\\
        =&\vect((\E[(\mU_L^T\mG\mU_R)\elesquare]))\\
    \end{align*}
    \begin{align*}
        \Rightarrow \tilde{\mD}^* = \diag_M((\E[(\mU_L^T\mG\mU_R)\elesquare]))
    \end{align*}
    
\end{proof}
\section{Further discussion}
\label{app: further discussion}
\subsection{Connections between different structural assumptions}
\label{subapp: connections between structural assumptions}
Here, we will make explicit connections between different structures in terms of their generality.

\paragraph{\gls{alicec} generalizes Adam} Since the structure behind \gls{alicec} is $\diagb(\mU\mD_1\mU^T,\ldots,\mU\mD_n\mU^T)$, when constraining $\mU=\mI_m$, the resulting structural is pure diagonal matrix $\diagb(\mD_1,\ldots,\mD_n)$. This coincide with the pure diagonal structure behind Adam.

\paragraph{\gls{alicec} generalizes $\mS\otimes \mM$} \gls{alicec} not only extends Adam, but also generalizes the structure considered in \cref{thm: generalization to normal and whiten}. Consider setting $\mD_i = S_i\mD$, then we have
\begin{align*}
    \mU\mD_i\mU^T = S_i\underbrace{\mU\mD\mU^T}_{\mM}
\end{align*}
Therefore, $\diagb(\mU\mD_1\mU^T,\ldots, \mU\mD_n\mU^T) = \diagb(S_1\mM,\ldots,S_n\mM)$. Since the structures of normalization and whitening are special cases of \cref{thm: generalization to normal and whiten}, \gls{alicec} generalizes these two gradient operators as a consequence.  

\paragraph{SOAP generalizes \gls{alicec}} We can re-write the structural assumption of \gls{alicec} in the Kronecker product format:
\begin{align*}
    \diagb(\mU\mD_1\mU^T,\ldots, \mU\mD_n\mU^T) = (\mI\otimes \mU)\underbrace{\diagb(\mD_1,\ldots,\mD_n)}_{\tilde{\mD}}(\mI\otimes \mU^T).
\end{align*}
It is clear that this is a special case of SOAP structure by containing $\mU_R=\mI$. 

\paragraph{SOAP generalizes Shampoo} The structure of Shampoo consists of two \gls{spd} matrices $\mR_n$ and $\mL_m$. If we eigen-decompose those, and let $\mR_n=\mU_R\mD_R\mU_R^T$ and $\mL_m=\mU_L\mD_L\mU_L^T$, then the structure of Shampoo can be re-write as 
\begin{align*}
    &\mR_n^{\frac{1}{2}}\otimes \mL_m^{\frac{1}{2}}\\
    =&(\mU_R\sqrt{\mD_R}\mU_R^T)\otimes(\mU_L\sqrt{\mD_L}\mU_L^T)\\
    =& (\mU_R\otimes \mU_L)\underbrace{(\sqrt{\mD_R\otimes \mD_L})}_{\tilde{\mD}} (\mU_R^T\otimes \mU_L^T).
\end{align*}
This coincides with SOAP's structure when the positive diagonal $\tilde{\mD}$ can be decomposed based on Kronecker product. 

\subsection{Memory consumption comparison}
\label{subapp: more thorough memory consumption}
\begin{table}[h]
\centering
Apart from the \cref{tab: summary table} provided in the main paper, we also include the memory consumption of more optimizers. 
\resizebox{\textwidth}{!}{
\begin{tabular}{l|l|llll}
\hline
              & Adam & GaLore & Fira & \gls{alice} & \gls{alicez}\\ \hline
Total         &  $3mn$    &  $mn+2nr+mr$ & $mn+2nr+mr$  &  $mn+2nr+mr+r^2+n$       &   $mn+2nr+mr+n$        \\ \hline
Weight        &   $mn$   &  $mn$   &  $mn$       &  $mn$       &  $mn$      \\
First moment  &  $mn$    &  $nr$   &   $nr$      &     $nr$    &    $nr$     \\
Second moment &   $mn$   &  $nr$  &    $nr$     &   $nr$      &      $nr$    \\
$\mU$ & N/A     &  $mr$   &    $mr$     &   $mr$      &   $mr$      \\
$\mC_t$  &  N/A    &  N/A   &    N/A     &    $n$     &    $n$     \\
$\widetilde{\mQ}$  &  N/A    &  N/A   &   N/A      &    $r^2$    &    N/A     \\
\hline
\end{tabular}}
\caption{The memory consumption of low-rank optimizers. Here, we assume the weight has a shape $m\times n$ with $m<n$ and $r\ll m$. Note that memory consumption $n$ and $r^2$ is typically very small. For example, let's take the largest possible $n=30K$ at the output layer. For 1B LLaMA, we select $r=512$, leading to $r^2\approx 262K$, which is 8x larger than $n$. However, both are marginal compared to $mr = 5120\times512$, which is 10x larger than $r^2$, and 80x larger than $n$.}
\label{tab: memory consumption for low-rank}
\end{table}

\subsection{Comparison to previous \gls{fim} approximation}
\label{subapp: comparison to previous work}
The idea of efficiently approximating the \gls{fim} is not novel, and has been extensively studied in the previous work. KFAC \citep{martens2015optimizing} is a well-known second order optimization algorithm for neural networks based on structural approximation to \gls{fim}. In particular, they explicitly express the \gls{fim} in terms of the gradient w.r.t this layer's output and the input to this layer. Namely, they directly decompose the gradient $\mG$ used in our paper, whereas in our paper, we treat $\mG$ as a whole quantity. In addition, the original KFAC also makes one crucial approximation: $\E[\mA\otimes \mB]\approx \E[\mA]\otimes \E[\mB]$. They do not show theoretically whether this is a good approximation and under what metric, but argue in practice this gives good accuracy. Last but not least, the original KFAC considers the \gls{fim} for entire layers, and each block (i.e.~corresponding to each layer) under their setup is actually the entire \gls{fim} we aim to approximate. To be precise, the principle is to structurally approximate the diagonal block of KFAC without decomposing the gradient $\mG$ using KFAC. 

Another more related work is AdaGrad \citep{duchi2011adaptive}. If we only considers layer-wise gradient, the full AdaGrad matrix is the $\sum_{t=1}^T\vecg_t\vecg_t^T$. Although our principle is to approximate \gls{fim}, in practice, we use \gls{ema} as the approximation of $\E$. Therefore, our principle can also be viwed as a structural approximation to full \gls{ema} AdaGrad matrix. Under this view point, there are some previous work on structural approximations. Shampoo \citep{gupta2018shampoo}, was originally proposed as an approximation to full AdaGrad matrix under the online learning setup. However, they do not explicit show under what metric this approximation is based on, and whether it is optimal or not. Later, there is a follow-up work \citep{morwani2024new}, showing that Shampoo is a 1-step power iteration algorithm of optimal Kronecker product approximation to \gls{fim} under Frobenius norm \citep{koroko2022efficient}. In this paper, we further extend the idea of approximating \gls{fim}/AdaGrad matrix to more efficient structures, under Frobenius norm. 

\subsection{Solutions to general block diagonal structures}
\label{subapp: solution to general block diagonal}
Here, we present the solution of the most general block diagonal approximation to \gls{fim}, and discuss why this is not practical. 

\begin{proposition}[General block diagonal approximation]
    Assume $\Ft=\diag(\mM_1,\ldots, \mM_n)$ with \gls{spd} matrix $\mM_i\in\Rmm$, then minimizing \cref{eq: UFE equation} admits analytic solutions:
    \begin{align}
        \mM_i^* = \E[\vg_i\vg_i^T]
        \label{eq: optimal general block diagonal}
    \end{align}
where $\vg_i$ is the $i^{\text{th}}$ column of $\mG$.
\end{proposition}
\begin{proof}
    This is straightforward by taking the derivative w.r.t. $\mM_i$. By leveraging \cref{lemma: block diagonal simplification}, we have
    \begin{align*}
        &\Fnorm{\Ft-\mF}\\
        =&\sum_{i=1}^n \Fnorm{\mM_i} - 2\tr(\mM_i^T\E[\vg_i\vg_i^T])
    \end{align*}
    Then, taking derivative w.r.t. $\mM_i$, we get 
    \begin{align*}
        \mM_i^* = \E[\vg_i\vg_i^T]
    \end{align*}
\end{proof}
From the above, although it admits analytic solutions, its corresponding practical procedure requires the \gls{ema} to estimate $\E[\vg_i\vg_i^T]\in\Rmm$ for all $i=1,\ldots, n$, leading to expensive memory cost $nm^2$. Another issue is that it does not allow efficient computation of the inverse. From the property of block diagonal matrix, to compute $\Ft^{-\frac{1}{2}}\vecg$, one needs to invert each $\mM_i$, incurring $O(m^3)$ computational cost. In total, the computational cost of inverting matrix $\Ft$ is $O(nm^3)$. Due to both memory and computational constraints, this structural assumption does not lead to practical algorithms.

\subsection{Connections to existing optimizers}
\label{subapp: connection to existing optimizers}
\paragraph{Lars and Lamb}
Lars and Lamb \citep{you2017lars, you2019lamb} relies on the normalization operation of the gradients. The main difference is that Lars proposed to normalize the raw gradient, whereas Lamb normalizes the processed gradient from Adam optimizer. They also involves a scaling parameter that is a function of the weight. Compared to the normalization discussed in \cref{subsec: sve}, the main difference is that we use channel-wise normalization with unit column or row norm, whereas Lars/Lamb uses matrix-wise normalization where the norm of the matrix is regularized to be 1. However, if one vectorizes the matrix weight into a vector, and stacks those vectors into a larger matrix. Then, the normalization step of Lamb and Lars can be viewed as a 1-sample approximation to \gls{fim}\footnote{This \gls{fim} is now the full Fisher across layers, compared to the 1-layer \gls{fim} considered in the paper} under the structure considered in \cref{subsec: sve}. 

\paragraph{Muon}
Muon \citep{jordan2024muon} performs the whitening operation on the momentum. \Cref{subapp: Muon} gives a brief introduction. In fact, Muon can be viewed as a special case of \cref{eq: generalization whitening} in \cref{coro: generalization to whitening and normalization} by considering the following approximation:
\begin{equation}
    \E[\mG\mG^T] \approx \E[\mG]\E[\mG^T].
\end{equation}
Thus, the resulting operation becomes the whitening of the $\E[\mG]$, which is estimated by the momentum in practice. There exists one difference: Muon omits the $\frac{1}{n}$ scalar, which serves as a layer-wise effective learning rate. 

\paragraph{SWAN}
SWAN composes the gradient normalization and whitening operations as the replacement of Adam's first and second moments. Each individual operations can be viewed as a special case of \cref{coro: generalization to whitening and normalization}. However, \cref{subsec: sve} does not provide an explanation for composing these two operators. Namely, the whitening operator is estimated using normalized gradients, rather than the raw gradient. We will leave the investigation of operator composition for future work. 

\paragraph{Adapprox}
Adapprox \citep{zhao2024adapprox} uses low-rank approximation for Adam's second moments through randomized \gls{svd} to boost the memory efficiency. However, its performance will be similar to Adam. In fact, \Cref{prop: two sided scaling} of \gls{ssgd} proves that the converged $\vs$ and $\vq$ represents the right and left singular vectors, coinciding with the rank-1 reconstruction. However, the proposed \gls{ssgd} is different from rank-1 Adapprox in three main aspects: (1) Adapprox proposes to use low-rank \gls{ema} tracking on $\mG\elesquare$, whereas we use \gls{ema} directly on scaling vectors $\vs$, $\vq$. The resulting vectors are no longer singular vectors; (2) we do not have separate scaling apart from norm-growth limiter and user-defined scale, whereas Adapprox uses the scaling from Adafactor; (3) \gls{ssgd} do not use any first moment. Specifically, (1) is an important difference, since for any low-rank reconstruction with rank $r>1$, it is not guaranteed to be positive, causing numerical issue during square-root inverse scaling. Adapprox adds a manually defined offset to ensure positivity. On the other hand, \gls{ssgd} is guaranteed to have positive $\vs$, $\vq$ at each step from Perron-Frobenius theorem. Therefore, \gls{ssgd} is numerically stable. 

\paragraph{Adafactor}
Adafactor \citep{shazeer2018adafactor} is another optimizer that uses rank-1 approximation to Adam's second moment. However, their scaling is derived by minimizing a different norm, compared to Frobenius norm in this paper. Adapprox \cite{zhao2024adapprox} has demonstrated the advantages of using Frobenius norm compared to Adafactor with a slightly better performance on LLM training. 

\paragraph{AdaDiag and one-sided SOAP}
We acknowledge that there exists two concurrent work: AdaDiag \citep{anonymous2024improving} and one-sided SOAP \citep{vyas2024soap}, that are mathematically equivalent to \gls{alicec}. They are derived based on distinct view points. AdaDiag is based on the intuition to transform the gradient covariance matrix so that it is diagonalizable, where this diagonal matrix is estimated using Adam's second moments. One-sided SOAP, a memory-efficient version of SOAP, is proposed to based on intuitions that only one-sided eigenspace is enough for LLM training. On the other hand, \gls{alicec} is derived on the basis of the structured \gls{fim} view point. providing a deeper connections to optimizers considered in this paper.

\paragraph{Apollo}
There is one concurrent work, Apollo \citep{zhu2024apollo}, that is also based on scaling the raw stochastic gradients for memory-efficient LLM training. It proposed to scale columns \textbf{or} rows of the $\mG$ through a scaling matrix estimated by similar procedure as Fira \cite{chen2024fira}. The main idea is that column or row norm after scaling matches the column or row norm of the gradient from GaLore update. Thus, they require a GaLore procedure to compute this update at each step. Our proposed \gls{ssgd} scales both columns and rows at the same time. The main differences are: (1) the scaling estimation in \gls{ssgd} is different from Apollo; (2) the scaling scheme of \gls{ssgd} is inspired by the generalization of normalization, providing theoretical support. They enjoy a similar memory consumption (i.e.~$mn+2n+2$ of Apollo compared to $mn+m+n+1$ of \gls{ssgd}). 

\paragraph{Fira}
Fira \cite{chen2024fira} was proposed as an improvement towards GaLore, which modifies the low-rank update to full-rank one by adding a compensation term. Their idea is similar to the compensation used in \gls{alice}. The main differences is that our proposed compensation is inspired by approximating \gls{fim} and has the theoretical foundation on its optimality. Also, the compensation strategy is different to Fira. We have conduced the ablation study in \cref{sec: experiments} to show the advantage of our approach. 

\paragraph{GaLore}
Comparing the \gls{alice} procedure to GaLore (\cref{alg: GaLore optimizer}), we can clearly see that GaLore is a special case of \gls{alice} by disabling tracking, switching and compensation. From the connection of \gls{alice} to \gls{alicec}, GaLore is a simple low-rank extension of \gls{alicec}, a more general optimizer than Adam. Therefore, the \gls{fim} view point reveals that GaLore is inherently a different optimizer compared to Adam, despite its similarity to Adam's update. This also provides an explanation on why GaLore can sometimes outperforms Adam under certain scenarios.

\subsection{Discussion of low-rank extension framework}
\label{subapp: discussion low-rank}
First, we derive how to decompose the full-rank update of \gls{alicec} into low-rank update and its residuals in the main text. We assume $\tilde{\mU}=[\mU,\mU_c]$, where $\mU$, $\mU_c$ are defined in \cref{subsec: compensation}.
\begin{align*}
    \devect(\Ft^{-\frac{1}{2}}\vecg) =& \tilde{\mU}\frac{\tilde{\mU}^T\mG}{\sqrt{\E[(\tilde{\mU}^T\mG)\elesquare]}}\\
    =& [\mU,\mU_c] \frac{\left[\begin{array}{c}
         \mU^T  \\
          \mU_c^T
    \end{array}\right]\mG}{\sqrt{\E\left[\left(\left[\begin{array}{c}
         \mU^T  \\
          \mU_c^T
    \end{array}\right]\mG\right)\elesquare\right]}}\\
    =& [\mU,\mU_c]\frac{\left[\begin{array}{c}
         \mU^T\mG  \\
         \mU_c^T\mG 
    \end{array}\right]}{\sqrt{\E\left[\left(\left[\begin{array}{c}
         \mU^T\mG  \\
         \mU_c^T\mG
    \end{array}\right]\right)\elesquare\right]}}\\
    =&[\mU,\mU_c]\left[\begin{array}{c}
           \frac{\mU^T\mG}{\sqrt{\E[(\mU^T\mG)\elesquare]}}\\
         \frac{\mU_c^T\mG}{\sqrt{\E[(\mU_c^T\mG)\elesquare]}}
    \end{array}\right]\\
    =& \mU\frac{\mU^T\mG}{\sqrt{\E[(\mU^T\mG)\elesquare]}} + \mU_c\frac{\mU_c^T\mG}{\sqrt{\E[(\mU_c^T\mG)\elesquare]}}
\end{align*}

\paragraph{Moore-Penrose inverse} In \cref{eq: alice update decomposition}, we define $\Ft^{-\frac{1}{2}}$ using the pseudo-inverse since each block in $\Ft_c$ is low-rank and not invertible. To be precise, for any block $\mU_c\mD_{ci}\mU_c^T$, its pseudo-inverse can be easily verified as $\mU_c\mD_{ci}^{-1}\mU_c^T$ by checking the definition. Similar to typical inverse, for any block diagonal $\Ft_c$, the pseudo-inverse is equivalent to applying pseudo-inverse of each blocks. The square-root pseudo inverse can be defined through a similar way. 

Similarly, for the proposed compensation, we can easily verify its vectorized format (ignoring the subscript $t$ for simplicity)
\begin{align*}
    \vect(\mU_c\mU_c^T\mG\mS) =& (\mS\otimes \mU_c\mU_c^T)\vecg\\
    =& (\mS^{-2}\otimes \mU_c\mU_c^T)^{-\frac{1}{2}}\vecg
\end{align*}
where the $^{-\frac{1}{2}}$ is the pseudo-inverse as the above. The second inequality can be easily verified by the following facts:
(1) the square root pseudo inverse of $\mU_c\mU_c^T$ is itself; (2) the square root pseudo-inverse of block-diagonal matrix is the pseudo-inverse of each individual block. 
\section{Experiment details}
\label{app: experiment details}
In this section, we will include the detailed setup, hyperparameters and additional experiment results. 

\subsection{Implementation details of baselines}
\label{subapp: baseline implementation}
\paragraph{GaLore} We leveraged the official GaLore package released by the author (\url{https://github.com/jiaweizzhao/GaLore}). 
\paragraph{Fira} Since the main difference compared to GaLore is the additional compensation term, we follow the implementation of the official Fira repo (\url{https://github.com/xichen-fy/Fira}) and add the compensation to GaLore. 

\subsection{Experiment setup for pretraining LLaMA}
\label{subapp: pretrain experiment setup}
For all model parameters, optimizer states, we use BF16 format.
We use context length of $256$ and batch size of $128$ with $4$ gradient accumulations, apart from $60$M and $1.3$B (batch size $256$ with $2$ gradient accumulations). \gls{ssgd} and \gls{alice} are applied to all linear modules of attention and MLPs, others are optimized with Adam.
We use the first $10\%$ of the total training steps as the warm-up, followed by a cosine learning rate decay to $10\%$ of the original learning rate. For hyperparameters, we perform a grid search on our proposed \gls{ssgd}, \gls{alice} as well as GaLore, Fira and full-rank Adam. For other methods, we directly cite their results in \citep{zhu2024apollo}. For Adam, we use $0.9$ and $0.999$ for $\beta_1$ and $\beta_2$, and enable the bias correction without weight decay. \Cref{tab:galore fira hyperparameter} summarizes the hyperparameters used for both GaLore and Fira. \Cref{tab:hyperparameters for adam} summarizes the hyperparameters for Adam optimizer. \Cref{tab:hyperparameters for ssgd} and \cref{tab: alice hyperparameters} summarizes the hyperparameters used for \gls{ssgd} and \gls{alice}, respectively. In addition, for all methods with Fira limiter, we use threshold $\gamma=1.01$ as suggested in \cite{chen2024fira}. For \gls{ssgd}, we use Adam to train the last layer of LLaMA, following the same setup as Apollo for fair comparison. In summary, \gls{alice}, GaLore and Fira do not use Adam to train the last layer to fully test the capabilities of low-rank methods, whereas Apollo, \gls{ssgd} and Adam utilize Adam for last layer. 
The total training steps for $60$M, $130$M, $350$M and $1.3$B are $10$K, $20$K, $60$K and $100$K, respectively. These correspond to $1.1$B, $2.6$B, $7.8$B and $13.1$B training tokens, summarized in \cref{tab: model architecture}. All experiments are conduced on NVIDIA A100 GPUs.

\begin{table}[ht]
\centering
\begin{minipage}{0.45\textwidth}
    \centering
    \caption{The hyperparameters for GaLore and Fira}
    \label{tab:galore fira hyperparameter}
    \resizebox{\textwidth}{!}{
    \begin{tabular}{l|llll}
    \hline
         & learning rate & update scale & rank & update interval \\ \hline
    60M  & 0.02          & 0.3          & 128  & 200          \\
    130M & 0.02          & 0.3          & 256  & 200          \\
    350M & 0.02          & 0.3          & 256  & 200          \\
    1.3B & 0.01          & 0.25         & 512  & 200          \\ \hline
    \end{tabular}}
\end{minipage}
\hfill
\begin{minipage}{0.45\textwidth}
    \centering
    \caption{The hyperparameters used for Adam optimizer.}
    \label{tab:hyperparameters for adam}
    \resizebox{\textwidth}{!}{
    \begin{tabular}{l|llll}
    \hline
         & learning rate & $\beta_1$ & $\beta_2$ & correct bias \\ \hline
    60M  & 0.001         & 0.9       & 0.999     & True         \\
    130M & 0.001         & 0.9       & 0.999     & True         \\
    350M & 0.001         & 0.9       & 0.999     & True         \\
    1.3B & $7\times 10^{-4}$ & 0.9   & 0.999     & True         \\ \hline
    \end{tabular}}
\end{minipage}
\end{table}

\begin{table}[ht]
\begin{minipage}{0.3\textwidth}
    \centering
    \caption{The hyperparameters for \gls{ssgd}.}
    \label{tab:hyperparameters for ssgd}
    \resizebox{\textwidth}{!}{
    \begin{tabular}{l|lll}
    \hline
         & learning rate & $\beta$ & scale $\alpha$  \\ \hline
    60M  & 0.02         & 0.9       & 0.05             \\
    130M & 0.02       & 0.9       & 0.05              \\
    350M & 0.02         & 0.9       & 0.05             \\
    1.3B & 0.02 & 0.9   & 0.02           \\ \hline
    \end{tabular}}
\end{minipage}\hfill
\begin{minipage}{0.6\textwidth}
    \centering
\caption{Model architectures and training steps}
\label{tab: model architecture}
\resizebox{\textwidth}{!}{
\begin{tabular}{l|llllll}
\hline
     & Hidden & Intermediate & Heads & Layers & Steps & Data amount \\ \hline
60M  & 512    & 1376         & 8     & 8      & 10K   & 1.3B        \\
130M & 768    & 2048         & 12    & 12     & 20K   & 2.6B        \\
350M & 1024   & 2736         & 16    & 24     & 60K   & 7.8B        \\
1.3B & 4096   & 5461         & 24    & 32     & 100K  & 13.1B       \\ \hline
\end{tabular}}
\end{minipage}
\end{table}

\begin{table}[]
\centering
\caption{The hyperparmeters for \gls{alice} optimizer}
\label{tab: alice hyperparameters}
\resizebox{\textwidth}{!}{
\begin{tabular}{l|lllllllll}
\hline
     & learning rate & scale $\alpha$ & compensation scale $\alpha_c$ & $\beta_1$ & $\beta_2$ & $\beta_3$ & update interval $K$ & rank $r$ & leading basis number $l$ \\ \hline
60M  & 0.02          & 0.3            & 0.4                           & 0.9      & 0.9      & 0.999    & 200                 & 128      & 40                       \\
130M & 0.02          & 0.3            & 0.4                           & 0.9      & 0.9      & 0.999    & 200                 & 256      & 40                       \\
350M & 0.02          & 0.3            & 0.4                           & 0.9      & 0.9      & 0.999    & 200                 & 256      & 40                       \\
1.3B & 0.02          & 0.25           & 0.2                          & 0.9      & 0.9      & 0.999    & 200                 & 512      & 160                      \\ \hline
\end{tabular}}
\end{table}

\subsection{Setup of 1B v.s. 7B}
\label{subapp: setup 1B vs 7B}
We mainly follow the setup as \citet{zhu2024apollo}. For 1B model, we use 16 per-device batch size with 8xA100s and 4 gradient accumulations, which is 512 batch size in total, same as 7B model. For \gls{alice}, we use learning rate $0.02$, scale $\alpha=0.3$ and compensation scale $\alpha_c=0.2$ with rank $r=512$, leading basis number $l=160$. We also use full-rank Adam to train the last layer for better performance. For \gls{ssgd}, we follow the same setup of 1.3B as reported in \cref{tab:hyperparameters for ssgd}.
For memory estimation, we assume the use of BF16 format. 
For 8-bit optimizers, we assume weights are stored in BF16, but optimizer states use FP8. GaLore uses $r=1024$ for 7B model. 

\subsection{Memory estimation}
Following the setup of \cite{zhao2024galore}, we provide the estimated GPU memory for each optimizers due to the difficulty of directly measuring their practical memory consumption without considering the activations, internal storage for gradient, etc. We assume BF16 format, which consuming 2 Bytes per element. For example, 1273M parameters are optimized by \gls{alice} and 65M parameters are optimized by Adam for 1B model with rank $512$. Therefore, \gls{alice} part will consume $3.86$GB and Adam will consume $0.37$GB, summing to $4.42$GB for \gls{alice} optimizer.

We also report the actual memory footprint with BF16 format. We use token batch size of $25$, following the same setup as \cite{zhao2024galore}. "-layerwise" represents layer-wise training where we only store the gradient of the current back-propagated layer. 

\subsection{Throughput estimation}
\label{subapp: exp details throughput}
We report both the actual throughput and effective throughput. The effective throughput of a target optimizer compared to a reference optimizer is defined as the ratio of training tokens consumed from the target optimizer w.r.t total time consumption of reference optimizer when reaching the same evaluation loss of the reference optimizer at the end of training. Compared to the standard throughput, this considers the speed-up effect. 

\subsection{Additional pretrain results}
\label{subapp: additional pretrain results}
\Cref{fig: additional pretrain curve} presents additional training curves with 60M, 130M and 350M models sizes. For all cases, the proposed \gls{alice} and \gls{ssgd} outperforms the baselines with noticable margin, and achieves clear speed-ups compared to full-rank Adam.
\begin{figure}
\subfigure[]{
    \centering
    \includegraphics[scale=0.24]{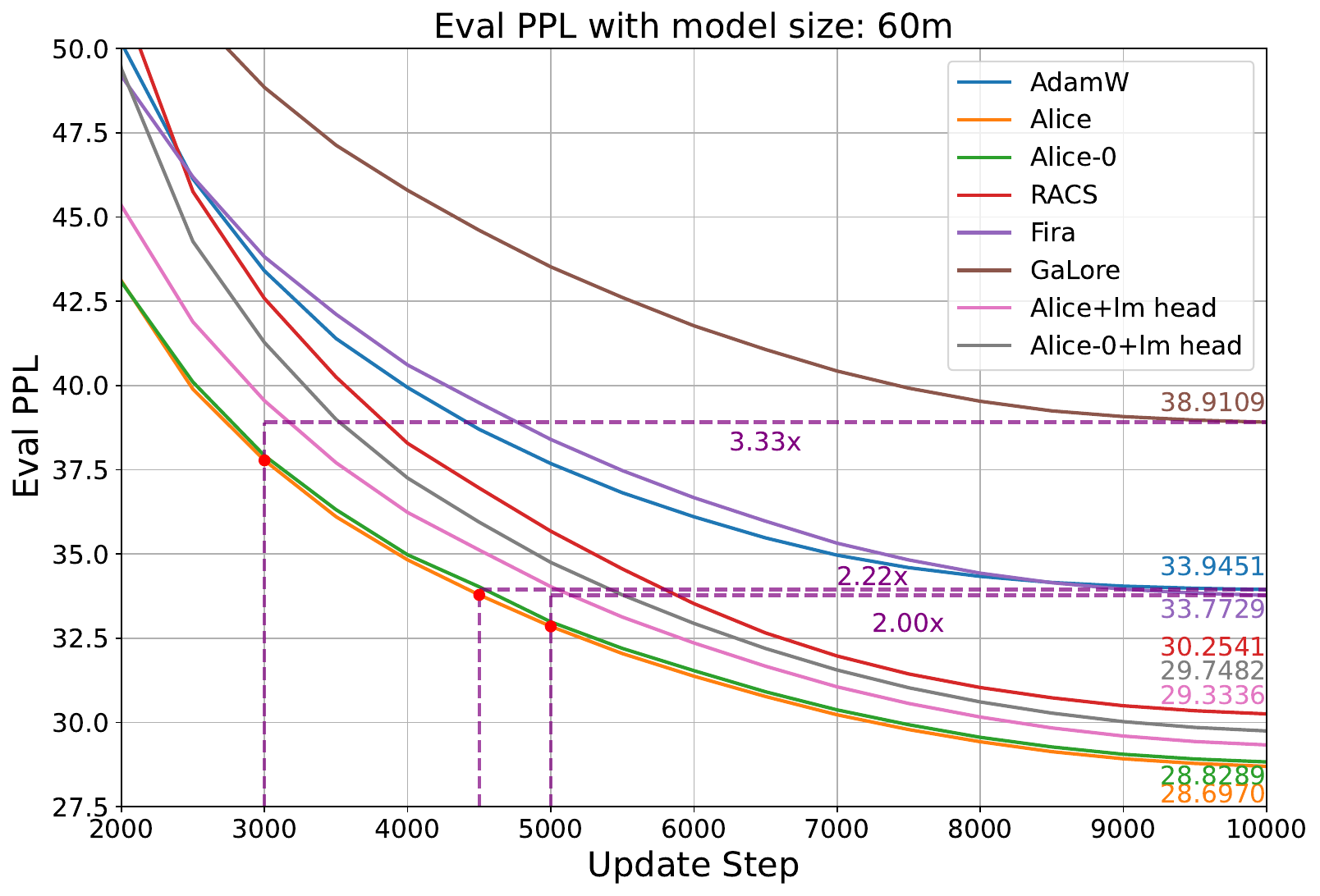}
    \label{fig:60M llama pretrain curve}
}\hfill
\subfigure[]{
    \centering
    \includegraphics[scale=0.24]{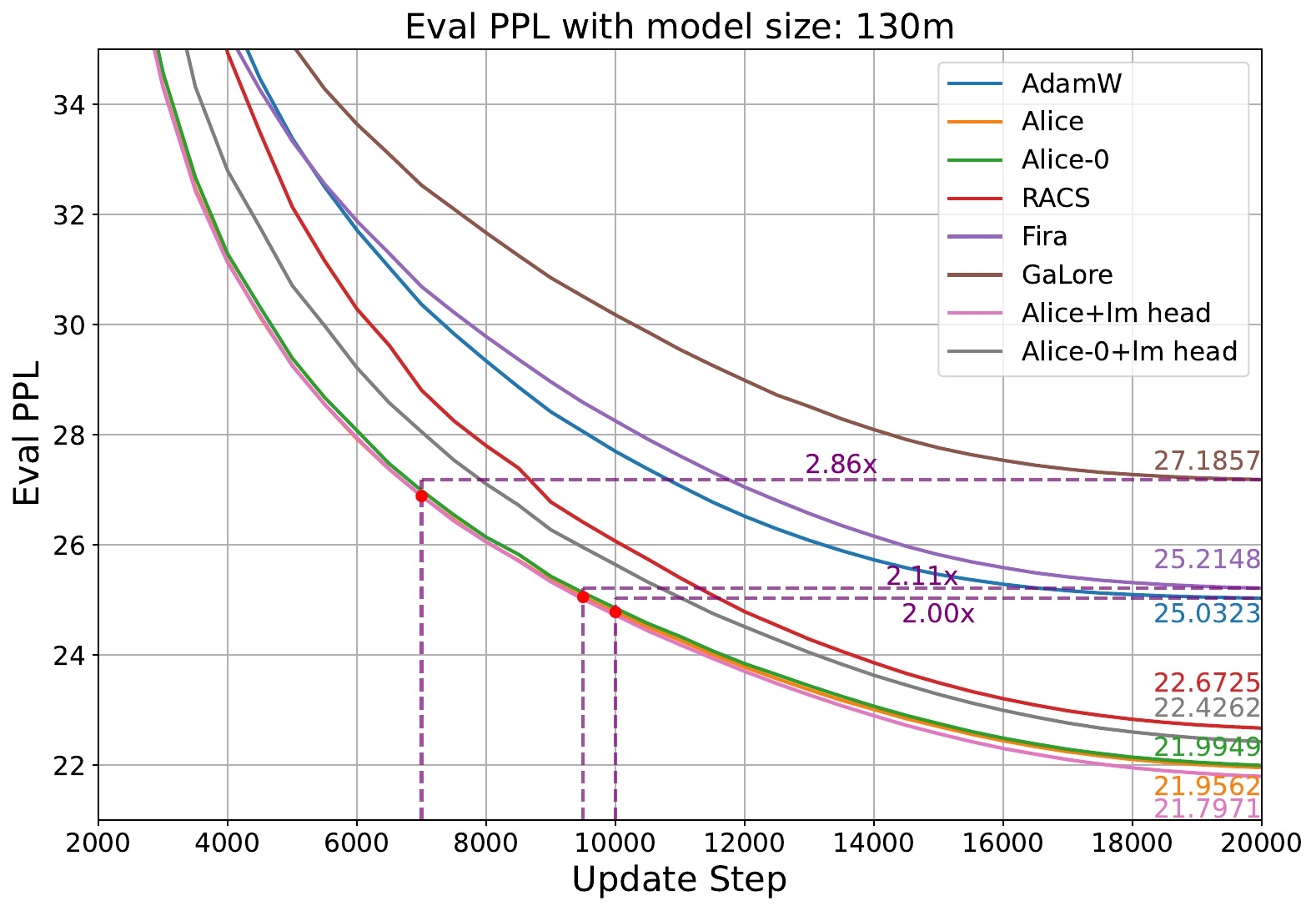}
    \label{fig:130M llama pretrain curve}
}\\
\centering
\subfigure[]{
    \centering
    \includegraphics[scale=0.24]{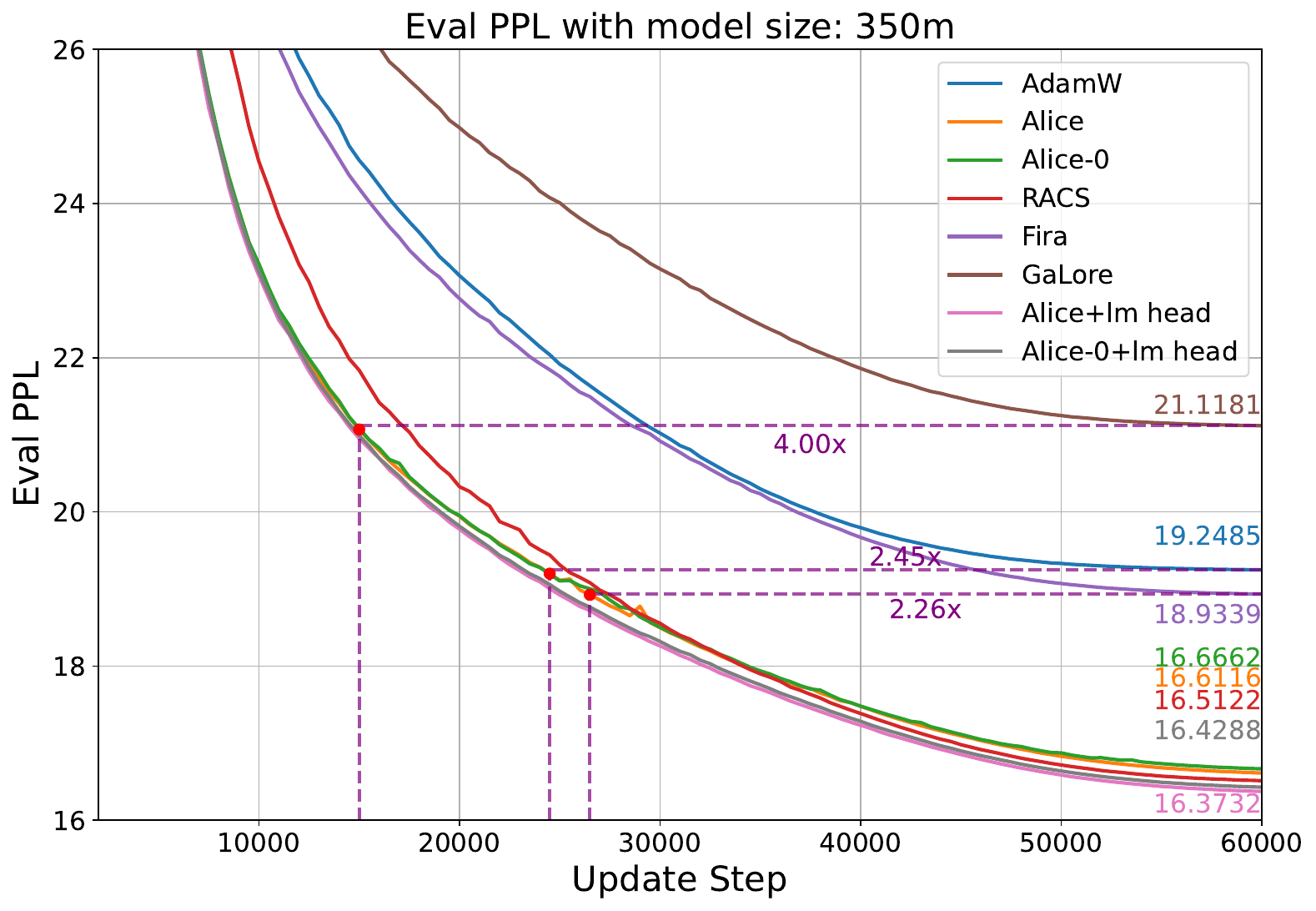}
    \label{fig:350M llama pretrain curve}
}
\caption{Additional LLaMA C4 pretrain performance curve. (a), (b) and (c) represents the 60M, 130M and 350M, respectively. "+lm head" represents that the last layer of LLaMA is trained by full-rank Adam.}
    \label{fig: additional pretrain curve}
\end{figure}

\begin{figure}
\centering
\subfigure[Absolute throughput]{
    \centering
    \includegraphics[scale=0.5]{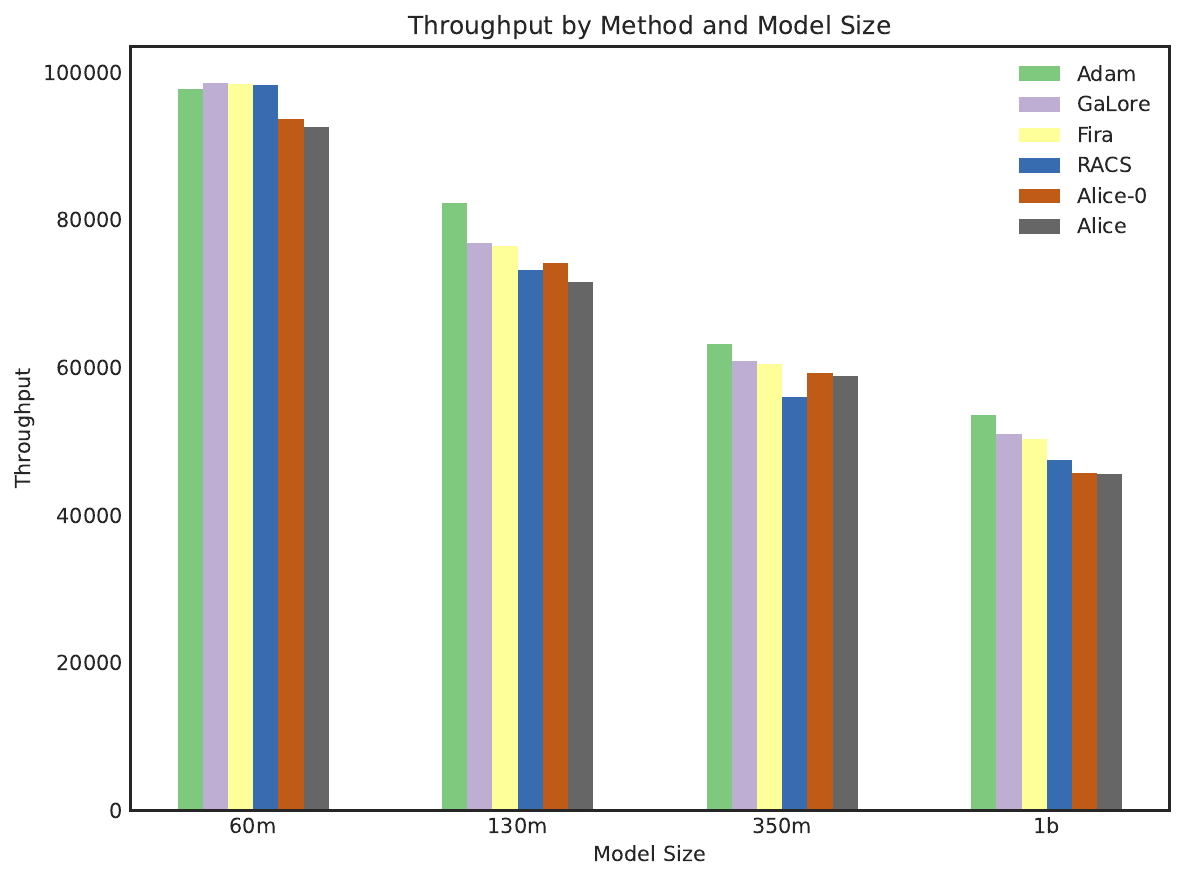}
    \label{fig:actual throughput}
}\\
\subfigure[Effective throughput]{
    \centering
    \includegraphics[scale=0.5]{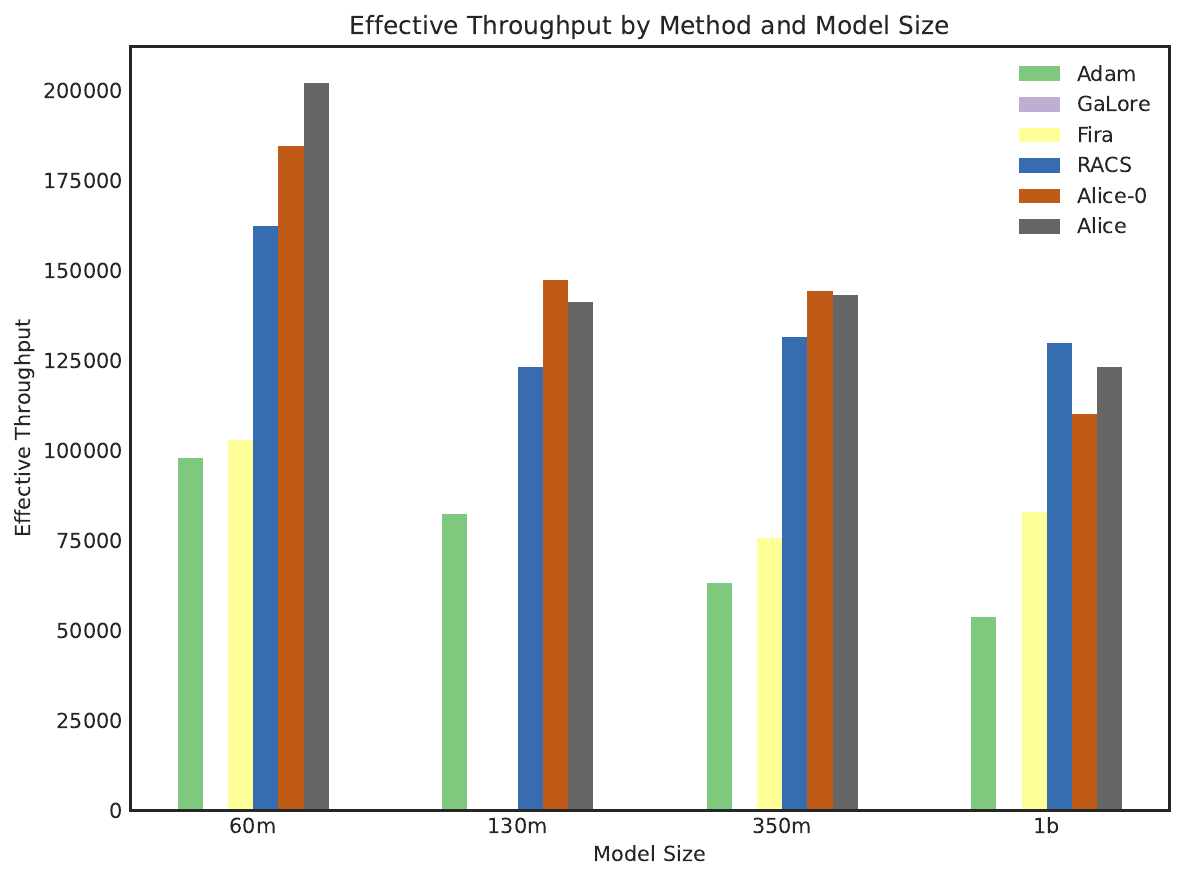}
    \label{fig:effective throughput}}
\caption{Throughput of various methods. (a) this reports the absolute throughput, representing the number of training token processed per second. (b) the effective throughput using Adam as the reference optimizer. This represents the absolute throughput adjusted by the speed-up factor. The effective throughput of GaLore and Fira is $0$ for some model sizes since they under-perform the Adam. }
    \label{fig: throughput plot}
\end{figure}

\begin{figure}
    \centering
    \includegraphics[scale=0.5]{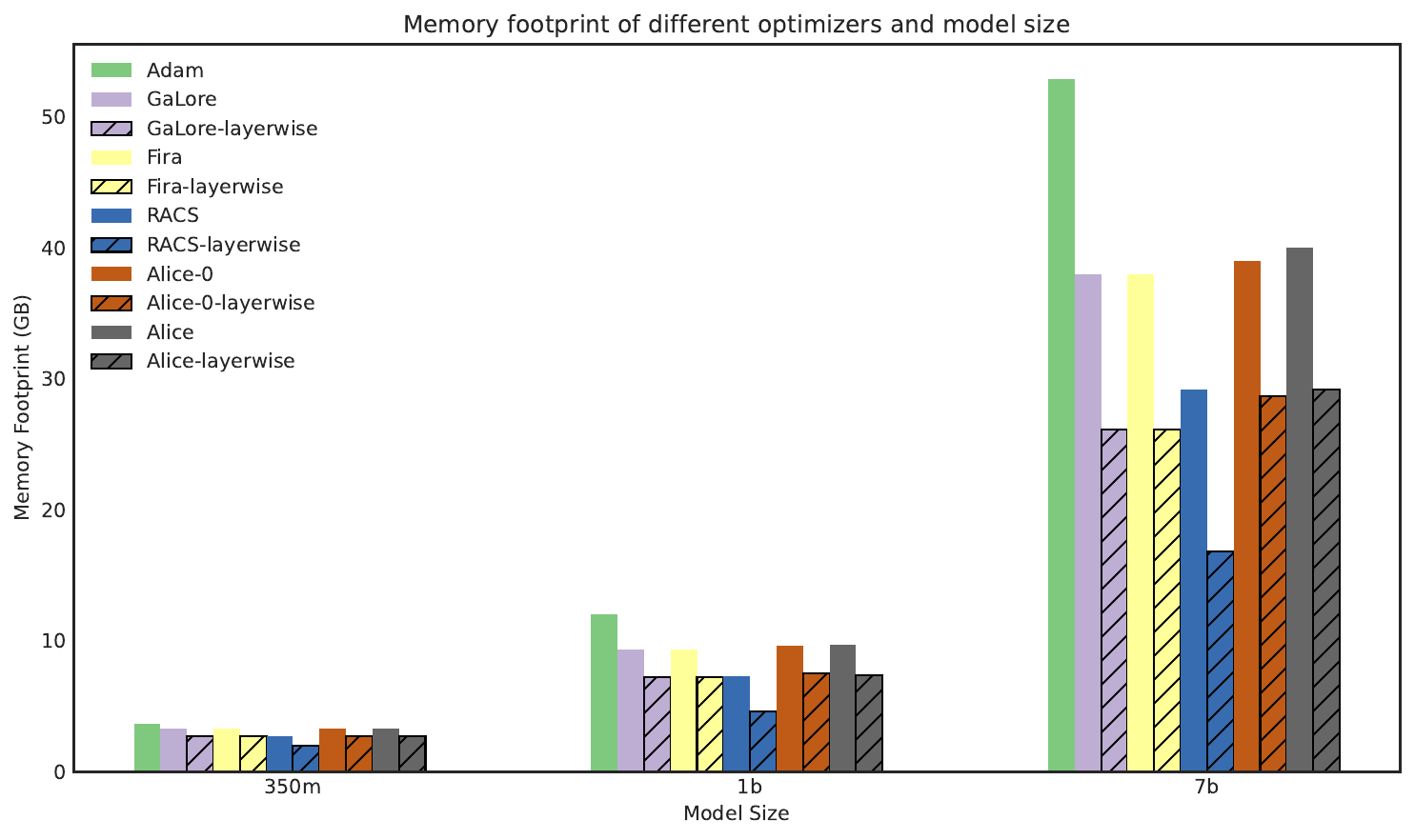}
    \caption{The memory footprint of various optimizers. We use token batch size of $256$ following the same setup as \cite{zhao2024galore} under BF16 format. The suffix "layerwise" represents the memory consumption when enabling layerwise training so that only the gradient of the current layer is stored. }
    \label{fig:actual memory footprint}
\end{figure}

\subsection{Results for ablation studies}
\label{subapp: ablation results}
For all ablations, we consider 130M LLaMA model. Apart from the ablation on the last layer, we assume the last layer is trained by the candidate optimizer, rather than the full-rank Adam for thorough evaluation. 

\paragraph{Setup: ablation of tracking} We disable the compensation step, and verify the effect of low-rank tracking under our proposed switch strategy or purely replying on \gls{evd} (i.e.~no switching). The other hyperparameters follow the pre-training setup as \cref{subapp: pretrain experiment setup}. 

\paragraph{Setup: switch strategy}
For this setup, we enable low-rank tracking. Apart from Gaussian, we use $40$ as the leading basis number. The Gaussian distribution is a standard isotropic Gaussian with zero mean and unit variance. To make sure the norm of sampled Gaussian vectors are consistent with the real basis, we normalize those sampled vectors to have a unit $l_2$ norm. The same operation is also applied to Gaussian-mix. 

\paragraph{Setup: compensation strategy}
We enable the tracking and switching, but with different compensation terms. For Fira, we directly leverage the compensation proposed in \cite{chen2024fira}. Fira+ modifies original Fira by the following two steps: (1) rescale the Fira compensation to have the same $l_2$ norm as the \gls{alice} low-rank update (i.e.~first term in \cref{eq: alice update decomposition}); (2) multiply this compensation by a scale parameter like $\alpha_c$ in \cref{alg: alice optimizer}. We found this empirical trick boosts the performance of Fira.

\paragraph{Effects of last layer}
One crucial setup difference during evaluation for low-rank methods is whether the last layer is trained by full-rank Adam or not. Most previous work train the last layer, arguably one of the most important layer \citep{zhao2024deconstructing}, using full-rank Adam. This effectively reduces the performance gap compared to full-rank method, and does not reveal their true capabilities. We investigate the effect of the last layer to \gls{alice} compared to GaLore and Fira. From the \cref{tab: pretrain performance}, we can see that for all model sizes, GaLore and Fira are greatly affected by this effect. Training last layer with full-rank Adam will boost their performance significantly. On the other hand, \gls{alice} is less impacted with marginally worse performance. For example, \cref{fig: effect of last layer} shows the training curve comparison with 130M model. When the rank is sufficiently large (e.g.~rank 128 is sufficient large for 60M model), \cref{fig:60M llama pretrain curve} shows that using Adam to train the last layer even decreases the performance. 
These serve as evidences that \gls{alice} is a better optimizer than GaLore and Fira, and has the potential to surpass full-rank Adam with large enough rank.  

\paragraph{Effect of \gls{ema} in \gls{ssgd}}
The only internal states inside \gls{ssgd} are the \gls{ema} tracking states of two vectors $\vs$ and $\vq$. We investigate the importance of \gls{ema} scheme. From \Cref{fig: effect of ema}, \gls{ssgd} without \gls{ema} performs much worse than \gls{ssgd}, suggesting the \gls{ema} is necessary for satisfactory performance of \gls{ssgd}.

\begin{figure}
\subfigure[Effect of tracking]{
    \centering
    \includegraphics[scale=0.24]{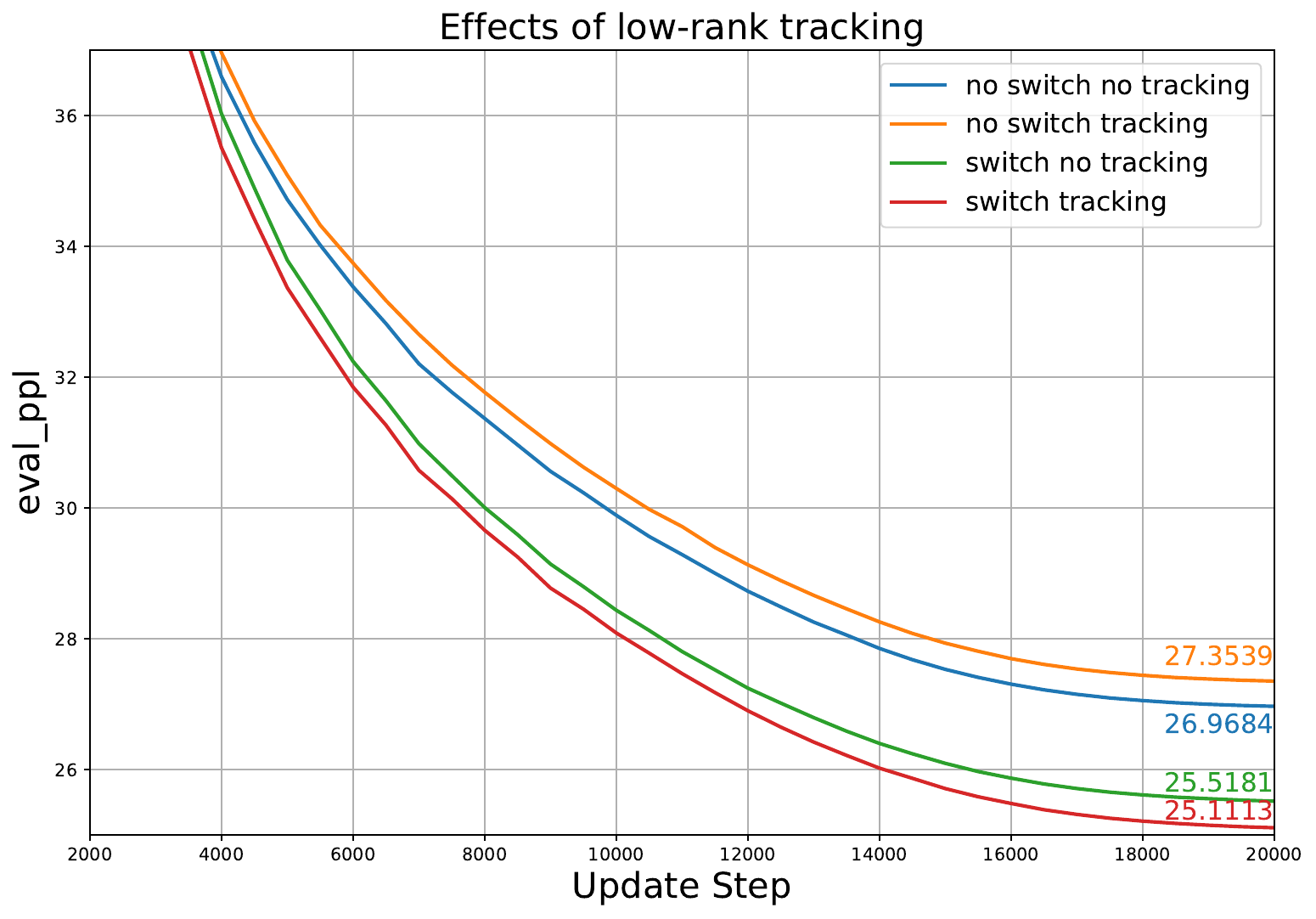}
    \label{fig: effect of tracking}}\hfill
\subfigure[Effect of switching]{
    \centering
    \includegraphics[scale=0.24]{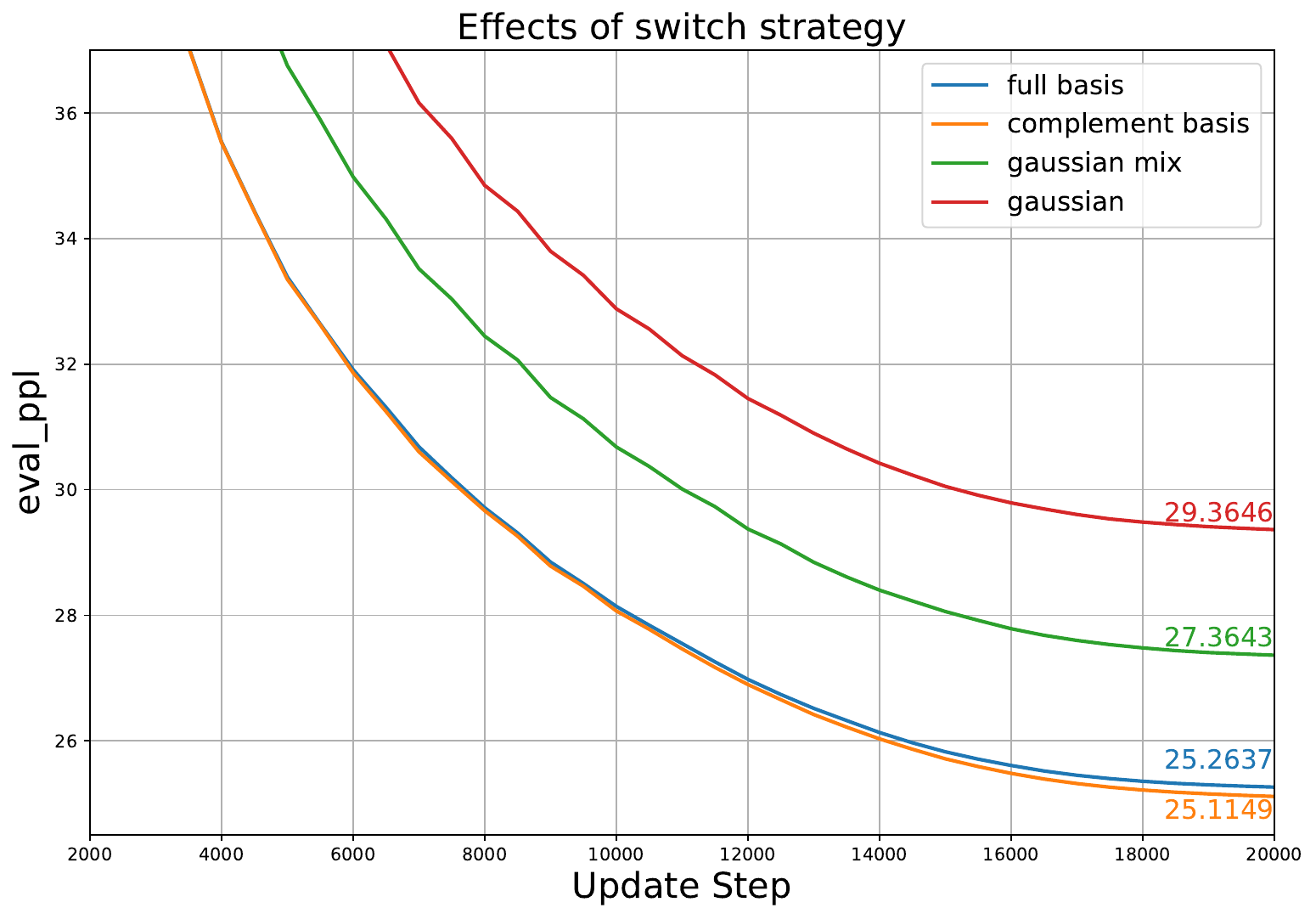}
    \label{fig: effect of switch strategy}}\\
\subfigure[Effect of compensation]{
    \centering
    \includegraphics[scale=0.24]{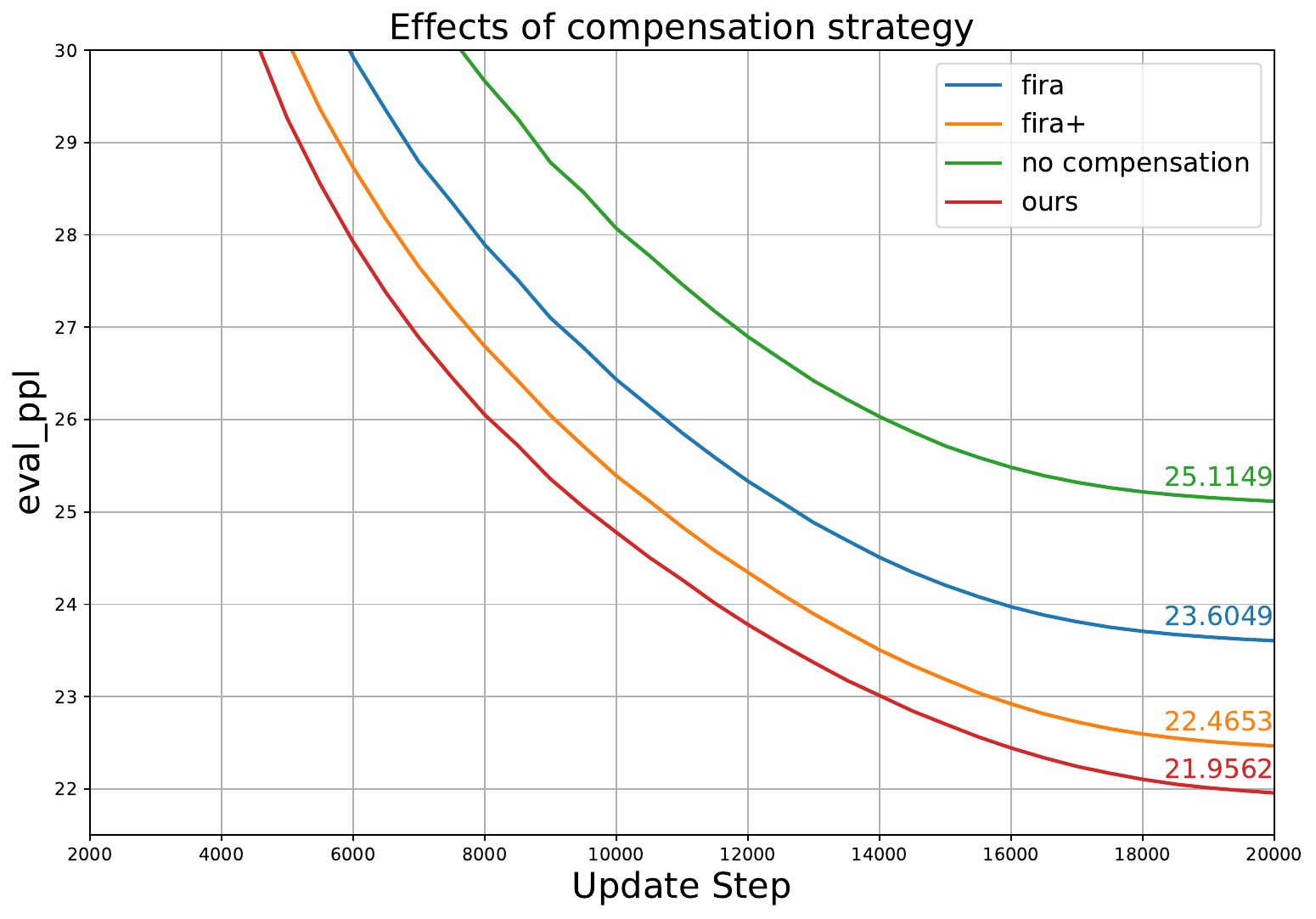}
    \label{fig: effect of compensation}}\hfill
\subfigure[Effect of last layer. "+lm head" represents the last layer is trained by full-rank Adam.]{
    \centering
    \includegraphics[scale=0.24]{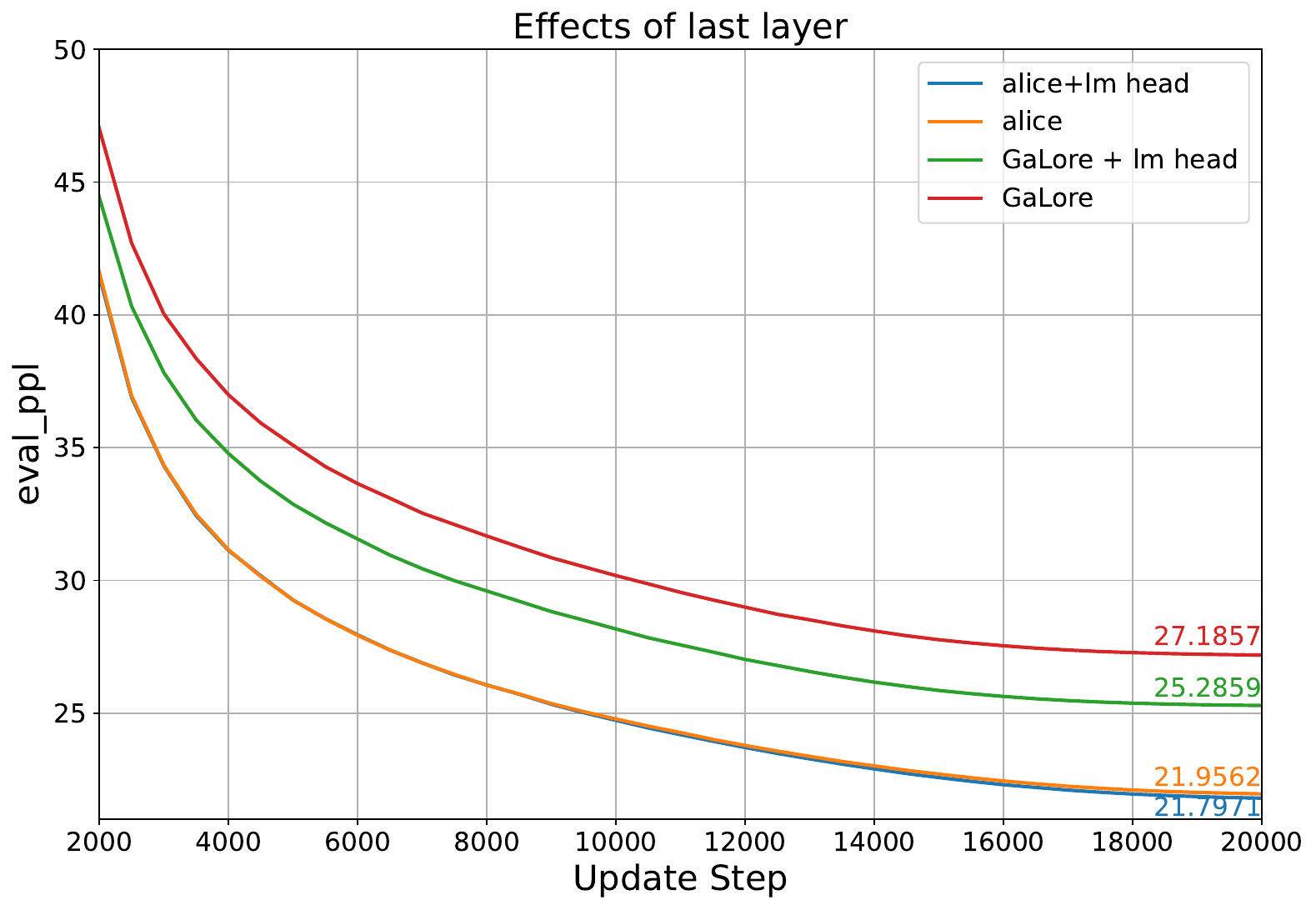}
    \label{fig: effect of last layer}}\\
\centering
\subfigure[Effect of \gls{ema} in \gls{ssgd}.]{
    \centering
    \includegraphics[scale=0.24]{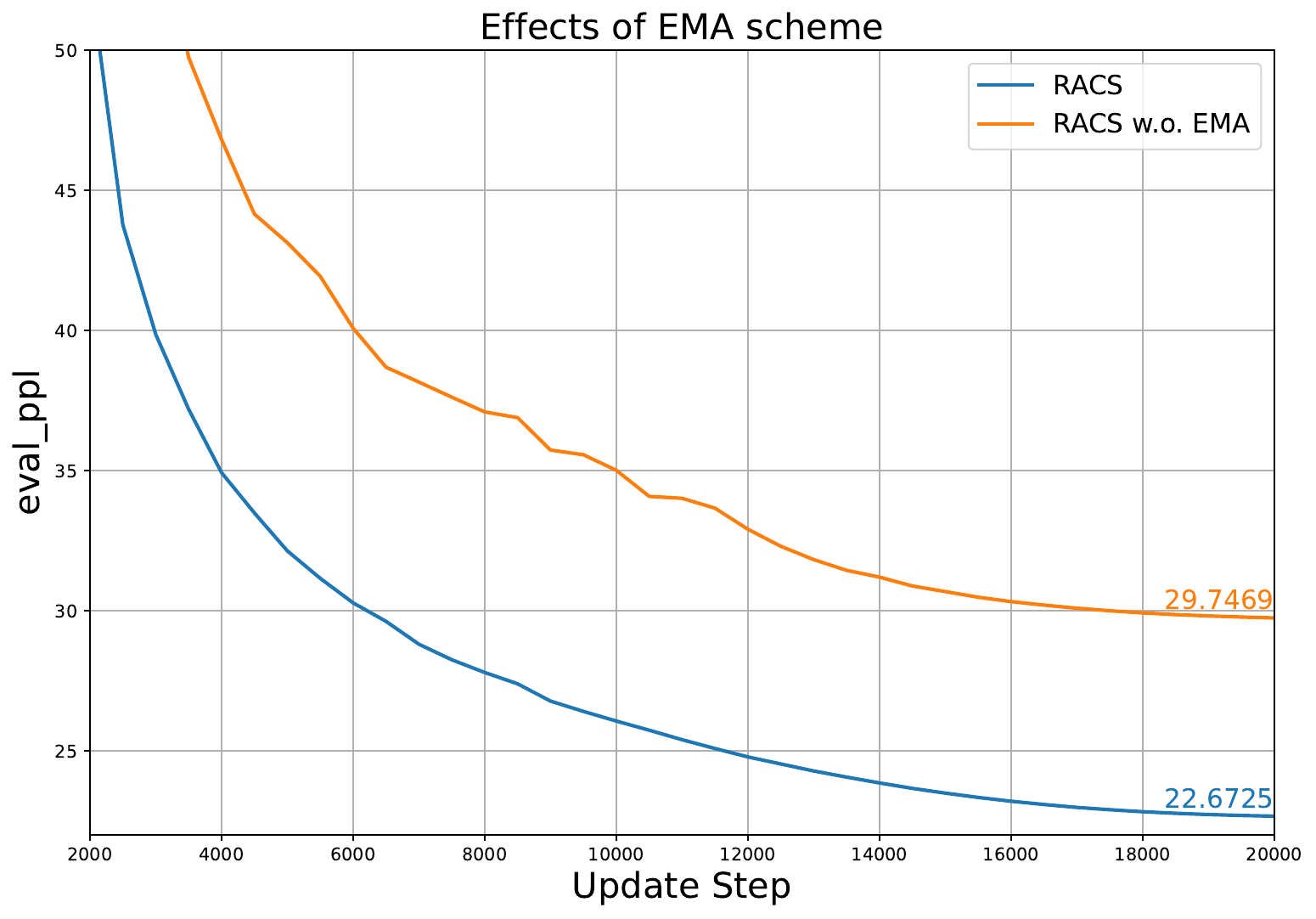}
    \label{fig: effect of ema}}
\caption{The pre-training curve to verify the effectiveness of the design choice. We consider 130M model size.}
\end{figure}

\begin{figure}
\subfigure[Index 1]{
    \centering
    \includegraphics[scale=0.45]{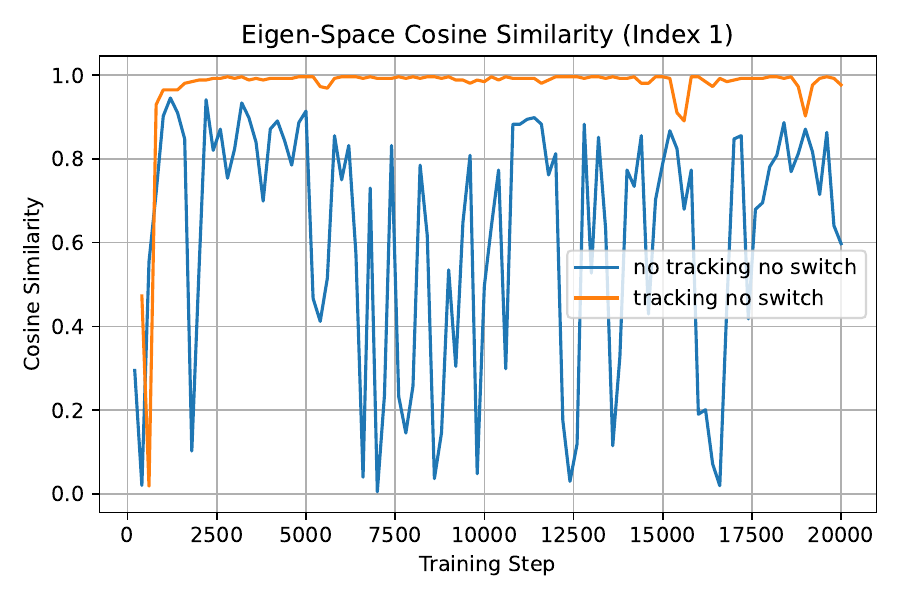}
    \label{fig:cossim index 1}}\hfill
\subfigure[Index 2]{
\centering
    \includegraphics[scale=0.45]{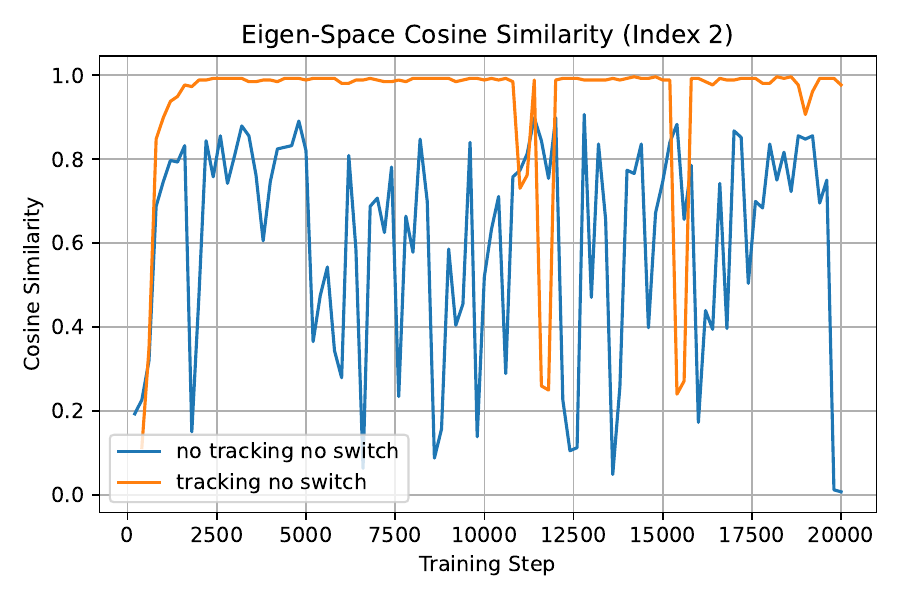}
    \label{fig:cossim index 2}
}\\
\subfigure[Index 4]{
\centering
    \includegraphics[scale=0.45]{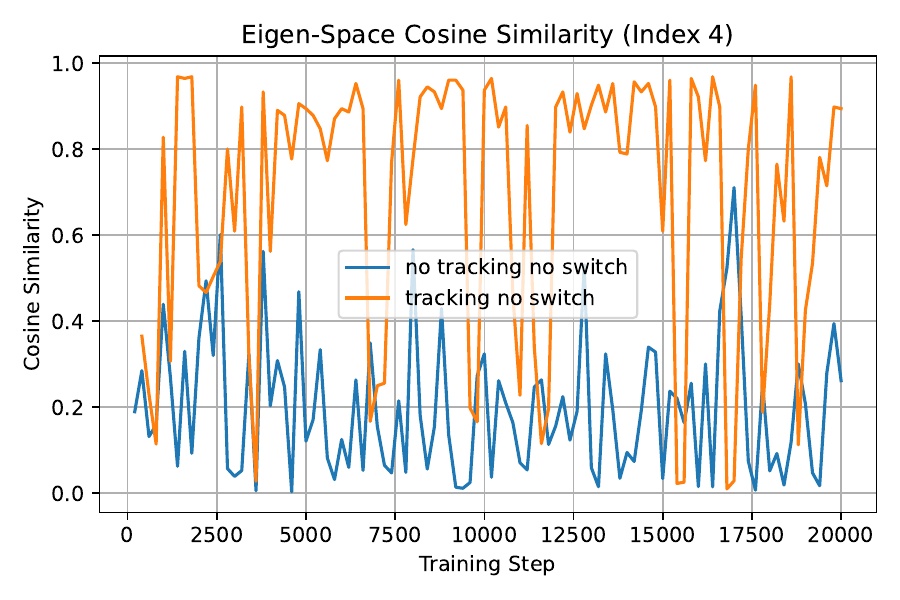}
    \label{fig:cossim index 4}
}\hfill
\subfigure[Index 8]{
\centering
    \includegraphics[scale=0.45]{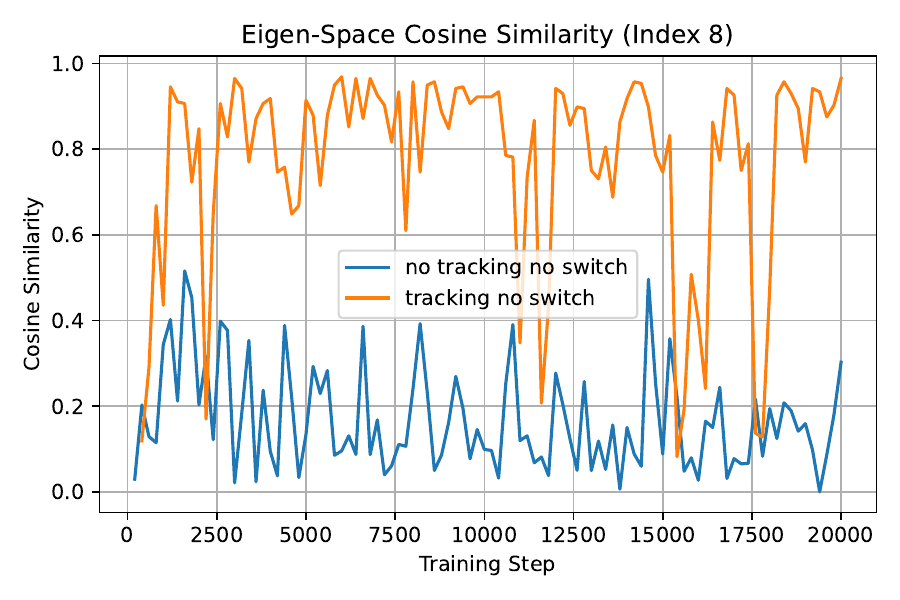}
    \label{fig:cossim index 8}
}
    \caption{The cosine similarity between eigenvectors per 200 steps. Since the update interval of subspace is $200$ steps, this essentially compare the similarity before and after updating the projection $\mU$ with a certain index. We always arrange the eigenvectors in a descending order based on the eigenvalues. }
    \label{fig: eigenspace cosine similiary}
\end{figure}

\end{document}